\documentclass{article}
\pdfoutput=1

\usepackage{ijcai19}

\usepackage{times}
\usepackage{soul}
\usepackage{url}
\usepackage{hyperref}
\usepackage[utf8]{inputenc}
\usepackage[small]{caption}
\usepackage{graphicx}
\usepackage{amsmath}
\usepackage{booktabs}
\usepackage{algorithm}
\usepackage{algorithmic}
\usepackage{natbib}

\usepackage{nicefrac}
\usepackage{subfig}
\usepackage{breakcites}

\usepackage{amsmath,amsthm,amssymb}
\usepackage{xifthen}
\usepackage{multirow}

\allowdisplaybreaks

\makeatletter
\newcommand\footnoteref[1]{\protected@xdef\@thefnmark{\ref{#1}}\@footnotemark}
\makeatother

\usepackage{array}
\newcolumntype{L}[1]{>{\raggedright\let\newline\\\arraybackslash\hspace{0pt}}m{#1}}
\newcolumntype{C}[1]{>{\centering\let\newline\\\arraybackslash\hspace{0pt}}m{#1}}
\newcolumntype{R}[1]{>{\raggedleft\let\newline\\\arraybackslash\hspace{0pt}}m{#1}}

\def\tRhelp#1#2\relax{L_{\csname dom#1\endcsname#2}} 

\def\eRhelp#1#2\relax{\hat{L}_{\csname set#1\endcsname#2}}
\newcommand{\loss}[1]{l\ifthenelse{\isempty{#1}{}}{}{\left(#1\right)}}
\newcommand{\negspace}[1]{\phantom{\hspace*{-#1}}}

\DeclareMathOperator*{\argmin}{arg\,min}
\DeclareMathOperator*{\argmax}{arg\,max}
\DeclareMathOperator*{\expect}{\mathbb{E}}
\DeclareMathOperator*{\prob}{\mathbb{P}}

\newcommand{\mydefv}[1]{\expandafter\newcommand\csname v#1\endcsname{\mathbf{#1}}}
\newcommand{\mydefallv}[1]{\ifx#1\mydefallv\else\mydefv{#1}\expandafter\mydefallv\fi}
\mydefallv abcekmosuvwxyz\mydefallv 

\newcommand{\mydefvsym}[1]{\expandafter\newcommand\csname v#1\endcsname{\boldsymbol{\csname #1\endcsname}}}
\newcommand{\mydefallvsym}[1]{\ifx#1\mydefallvsym\else\mydefvsym{#1}\expandafter\mydefallvsym\fi}
\mydefallvsym {sigma}{alpha}{gamma}{mu}\mydefallvsym 

\newcommand{\mydefm}[1]{\expandafter\newcommand\csname m#1\endcsname{\mathbf{#1}}}
\newcommand{\mydefallm}[1]{\ifx#1\mydefallm\else\mydefm{#1}\expandafter\mydefallm\fi}
\mydefallm ABCIKLMOSTVWXZ\mydefallm 

\newcommand{\mydefmsym}[1]{\expandafter\newcommand\csname m#1\endcsname{\boldsymbol{\csname #1\endcsname}}}
\newcommand{\mydefallmsym}[1]{\ifx#1\mydefallmsym\else\mydefmsym{#1}\expandafter\mydefallmsym\fi}
\mydefallmsym {Sigma}{Gamma}{Phi}{gamma}\mydefallmsym 

\newcommand{\mydefalg}[1]{\expandafter\newcommand\csname alg#1\endcsname{\mathcal{#1}}}
\newcommand{\mydefallalg}[1]{\ifx#1\mydefallalg\else\mydefalg{#1}\expandafter\mydefallalg\fi}
\mydefallalg A\mydefallalg 

\newcommand{\mydefdom}[1]{\expandafter\newcommand\csname dom#1\endcsname{\mathcal{#1}}}
\newcommand{\mydefalldom}[1]{\ifx#1\mydefalldom\else\mydefdom{#1}\expandafter\mydefalldom\fi}
\mydefalldom HSTV\mydefalldom 

\newcommand{\mydefset}[1]{\expandafter\newcommand\csname set#1\endcsname{{#1}}}
\newcommand{\mydefallset}[1]{\ifx#1\mydefallset\else\mydefset{#1}\expandafter\mydefallset\fi}
\mydefallset BHPSTVY\mydefallset 

\newcommand{\mydefdistr}[1]{\expandafter\newcommand\csname distr#1\endcsname{\mathcal{D}_{\csname dom#1\endcsname}}}
\newcommand{\mydefalldistr}[1]{\ifx#1\mydefalldistr\else\mydefdistr{#1}\expandafter\mydefalldistr\fi}
\mydefalldistr DHSTV\mydefalldistr 

\newcommand{\mydefedistr}[1]{\expandafter\newcommand\csname edistr#1\endcsname{\mathcal{#1}}}
\newcommand{\mydefalledistr}[1]{\ifx#1\mydefalledistr\else\mydefedistr{#1}\expandafter\mydefalledistr\fi}
\mydefalledistr W\mydefalledistr 

\newcommand{\mydefspace}[1]{\expandafter\newcommand\csname space#1\endcsname{\mathcal{#1}}}
\newcommand{\mydefallspace}[1]{\ifx#1\mydefallspace\else\mydefspace{#1}\expandafter\mydefallspace\fi}
\mydefallspace DFGHKLMPRTUVXYZ\mydefallspace 

\newcommand{\mydeff}[1]{\expandafter\newcommand\csname f#1\endcsname[2][]{#1##1\ifthenelse{\equal{##2}{}}{}{\!\left(##2\right)}}}
\newcommand{\mydefallf}[1]{\ifx#1\mydefallf\else\mydeff{#1}\expandafter\mydefallf\fi}
\mydefallf cdfghkzCFGHMRT{PV}\mydefallf 

\newcommand{\mydeftilde}[1]{\expandafter\newcommand\csname tilde#1\endcsname{\tilde{#1}}}
\newcommand{\mydefalltilde}[1]{\ifx#1\mydefalltilde\else\mydeftilde{#1}\expandafter\mydefalltilde\fi}
\mydefalltilde fh\mydefalltilde 
\newcommand{\mydeffsym}[1]{\expandafter\newcommand\csname f#1\endcsname[2][]{\csname #1\endcsname##1\ifthenelse{\equal{##2}{}}{}{\!\left(##2\right)}}}
\newcommand{\mydefallfsym}[1]{\ifx#1\mydefallfsym\else\mydeffsym{#1}\expandafter\mydefallfsym\fi}
\mydefallfsym {phi}{epsilon}{eta}{theta}{nu}{tildef}{tildeh}\mydefallfsym 

\newcommand{\mydefnset}[1]{\expandafter\newcommand\csname nset#1\endcsname{\mathbb{#1}}}
\newcommand{\mydefallnset}[1]{\ifx#1\mydefallnset\else\mydefnset{#1}\expandafter\mydefallnset\fi}
\mydefallnset CNRSZ\mydefallnset 

\newcommand{\normTwo}[1]{\left\|#1\right\|_2}

\newcommand{\normOne}[1]{\left\|#1\right\|_1}

\newcommand{\abs}[1]{\left|#1\right|}

\newcommand{\dotprod}[2]{\left\langle#1,#2\right\rangle}

\newcommand{\bigO}[1]{\mathcal{O}\left( #1 \right)}

\newcommand{\indic}[1]{\mathbb{I}_{ #1}}

\newcommand{\ceil}[1]{{\left\lceil #1 \right\rceil}}

\newtheorem{myth}{Theorem}
\newtheorem*{myth*}{Theorem}

\newtheorem*{mycor*}{Corollary}
\newtheorem{mylem}{Lemma}
\newtheorem*{mylem*}{Lemma}

\newtheorem{myex}{Example}
\newtheorem*{myex*}{Example}

\newenvironment{reth}[1]
    {
\begingroup
\def\themyth{#1}
\begin{myth}
    }
    { 
\end{myth}
\addtocounter{myth}{-1}
\endgroup
    }
    
\newenvironment{reex}[1]
    {
\begingroup
\def\themyex{#1}
\begin{myex}
    }
    { 
\end{myex}
\addtocounter{myex}{-1}
\endgroup
    }

\newcommand{\abstain}{\vartheta}

\makeatletter
\def\th@plain{%
  \thm@notefont{}
  \itshape 
}
\def\th@definition{%
  \thm@notefont{}
  \normalfont 
}
\makeatother

\title{Boosting for Comparison-Based Learning}

\author{
Micha\"el Perrot$^1$
\and
Ulrike von Luxburg$^{1,2}$
\affiliations
$^1$Max-Planck-Institute for Intelligent Systems, T\"ubingen, Germany\\
$^2$University of T\"ubingen, Department of Computer Science, Tübingen, Germany
\emails
michael.perrot@tuebingen.mpg.de,
ulrike.luxburg@uni-tuebingen.de
}

\begin{document}

\maketitle

\begin{abstract}
We consider the problem of classification in a comparison-based setting: given a set of objects, we only have access to triplet comparisons of the form \emph{object $x_i$ is closer to object $x_j$ than to object $x_k$}.
In this paper we introduce TripletBoost, a new method that can learn a classifier just from such triplet comparisons. 
The main idea is to aggregate the triplets information into weak classifiers, which can subsequently be boosted to a strong classifier.
Our method has two main advantages: (i) it is applicable to data from any metric space, and (ii) it can deal with large scale problems using only passively obtained and noisy triplets. 
We derive theoretical generalization guarantees and a lower bound on the number of necessary triplets, and we empirically show that our method is both competitive with state of the art approaches and resistant to noise.
\end{abstract}

\section{Introduction}\label{sec:intro}

In the past few years the problem of comparison-based learning has attracted growing interest in the machine learning community 
\citep{agarwal2007generalized,jamieson2011low,tamuz2011adaptively,tschopp2011randomized,van2012stochastic,heikinheimo2013crowd,amid2015multiview,kleindessner2015dimensionality,jain2016finite,haghiri2017comparison,kazemi2018comparison}.
The motivation is to relax the assumption that an explicit representation of the objects or a distance metric between pairs of examples are available.
Instead one only has access to a set of ordinal distance comparisons that can take several forms depending on the problem at hand.
In this paper we focus on triplet comparisons of the form \emph{object $x_i$ is closer to object $x_j$ than to object $x_k$}, that is on relations of the form $\fd{x_i,x_j} < \fd{x_i,x_k}$ where $\fd{}$ is an unknown metric\footnote{Note that this kind of ordinal information is sometimes just used as side information \citep{bellet2015metric,kane2017active}, but, as the references in the main text, we focus on the setting where ordinal comparisons are the sole information available.}.

We address the problem of classification with noisy triplets that have been obtained in a passive manner: the examples lie in an unknown metric space, not necessarily Euclidean, and we are only given a small set of triplet comparisons --- there is no way in which we could actively ask for more triplets.
Furthermore we assume that the answers to the triplet comparisons can be noisy.
To deal with this problem one can try to first recover an explicit representation of the examples, a task that can be solved by ordinal embedding approaches \citep{agarwal2007generalized,van2012stochastic,terada2014local,jain2016finite}, and then apply standard machine learning approaches.
However, such embedding methods assume that the examples lie in a Euclidean space and do not scale well with the number of examples: typically they are too slow for datasets with more than $10^3$ examples. 
As an alternative, it would be desirable to have a classification algorithm that can work with triplets directly, without taking a detour via ordinal embedding. 
To the best of our knowledge, for the case of passively obtained triplets, this problem has not yet been solved in the literature. 

Another interesting question in this context is 
that of the minimal number of triplets required to successfully learn a classifier. 
It is known that to exactly recover an ordinal embedding one needs of the order $\Omega(n^3)$ passively queried triplets in the worst case (essentially all of them), unless we make stronger assumptions on the underlying metric space \citep{jamieson2011low}. 
However, classification is a problem which seems simpler than ordinal embedding, and thus it might be possible to obtain better lower bounds.

In this paper we propose TripletBoost, a method for classification that is able to learn using only passively obtained triplets while not making any assumptions on the underlying metric space.
To the best of our knowledge this is the first approach that is able to solve this problem.
Our method is based on the idea that the triplets can be aggregated into simple \emph{triplet classifiers}, which behave like decision stumps and are well-suited for boosting approaches \citep{schapire2012boosting}.
From a theoretical point of view we prove that our approach learns a classifier with low training error, and we derive generalization guarantees that ensure that its error on new examples is bounded.
Furthermore we derive a new lower bound on the number of triplets that are necessary to ensure useful predictions. 
From an empirical point of view we demonstrate that our approach can be applied to datasets that are several order of magnitudes larger than the ones that can currently be handled by ordinal embedding methods.
Furthermore we show that our method is quite resistant to noise.

\section{The TripletBoost Algorithm}\label{sec:contrib}

In this paper we are interested in multi-class classification problems.
Let $\left( \spaceX,\fd{} \right)$ be an unknown and general metric space, typically not Euclidean.
Let $\spaceY$ be a finite label space. 
Let $\setS = \left\lbrace (x_i,y_i) \right\rbrace_{i=1}^n$ be a set of $n$ examples drawn i.i.d. from an unknown distribution $\distrS$ defined over $\spaceX \times \spaceY$.
Note that we use the notation $x_i$ as a convenient way to identify an object; it does not correspond to any explicit representation that could be used by an algorithm (such as coordinates in a vector space).  
Let $\setT = \left\lbrace (x_i,x_j,x_k) : (x_i,y_i),(x_j,y_j),(x_k,y_k) \in \setS, x_j \neq x_k \right\rbrace$ be a set of $m$ triplets.
Each ordered tuple $(x_i,x_j,x_k) \in \setT$ encodes the following relation between the three examples:
\begin{align}
\fd{x_i,x_j} < \fd{x_i,x_k}\text{.} \label{def:triplet}
\end{align}
Given the triplets in $T$ and the label information of all points, our goal is to learn a classifier.
We make two main assumptions about the data. 
First, we assume that the triplets are uniformly and independently sampled from the set of all possible triplets.
Second, we assume that the triplets in $\setT$ can be noisy (that is the inequality has been swapped, for example $(x_i,x_k,x_j) \in \setT$ while the true relation is $\fd{x_i,x_j} < \fd{x_i,x_k}$), but the noise is uniform and independent from one triplet to another.
In the following $\indic{a}$ denotes the indicator function returning $1$ when a property $a$ is verified and $0$ otherwise, $\edistrW_c$ is an empirical distribution over $\setS \times \spaceY$, and $w_{c,x_i,y}$ is the weight associated to object $x_i$ and label $y$.

\subsection{Weak Triplet Classifiers}
\label{sec:tripletclassifiers}

\begin{figure}
\centering
\includegraphics[width=.6\linewidth]{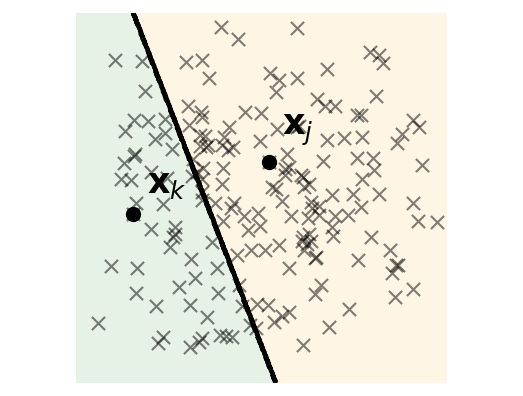}
\caption{Example of a triplet classifier. Given two reference points $x_j$ and $x_k$, the space is divided in two half-spaces: examples closer to $x_j$ and examples closer to $x_k$. Triplet information is enough to reveal which half-space a point $x_i$ is in. \label{fig:tripletClassifier}}
\end{figure}

Rather than considering triplets as individual pieces of information we propose to aggregate them into decision stumps that we call \emph{triplet classifiers}.
The underlying idea is to select two reference examples, $x_j$ and $x_k$, and to divide the space into two \emph{half-spaces}: examples that are closer to $x_j$ and examples that are closer to $x_k$.
This is illustrated in Figure~\ref{fig:tripletClassifier}. In principle, this can be achieved with triplets only. However, a major difficulty in our setting is that our triplets are passively obtained: for most training points $x_i$ we do not know whether they are closer to $x_j$ or $x_k$. In particular, it is impossible to evaluate the classification accuracy of such a simple classifier on the whole training set. 
To deal with this problem we propose to use an abstention scheme where a triplet classifier abstains if it does not know on which side of the hyperplane the considered point lies.
Given a set $\setT$ of triplets and two reference points $x_j$ and $x_k$, we define a triplet classifier as:
\begin{align*}
\fh[_{j,k}]{x} = \left\lbrace \begin{array}{ll}
o_j & \text{if $(x,x_j,x_k) \in \setT$,} \\
o_k & \text{if $(x,x_k,x_j) \in \setT$,} \\
\abstain & \text{otherwise.}
\end{array}\right.
\end{align*}

In our multi-class setting, $o_j$ and $o_k$ will be sets of labels, that is $o_j, o_k \subseteq \spaceY$. In Section~\ref{sec:tripletboost} we describe how we choose them in a data dependent fashion to obtain classifiers with minimal error on the training set. 
The prediction $\abstain$ simply means that the triplet classifier abstains on the example. Let $\spaceH$ denote the set of all possible triplet classifiers.

The triplet classifiers are very simple and, in practice, we do not expect them to perform well at all.
But we prove in Section~\ref{sec:tripletisweak} that, for appropriate choices of $o_j$ and $o_k$, they are at least as good as random predictors. This is all we need to ensure that they can be used successfully in a boosting framework. The next section describes how this works.

\subsection{TripletBoost}
\label{sec:tripletboost}

\begin{algorithm}[tb]
\caption[TripletBoost]{TripletBoost: Boosting with triplet classifiers.\label{alg:tripletboost}}
\label{alg:algorithm}
\textbf{Input}: $\setS = \left\lbrace (x_i,y_i) \right\rbrace_{i=1}^n$ a set of $n$ examples,\\ \phantom{\textbf{Input}:} $\setT = \left\lbrace (x_i,x_j,x_k) : x_j \neq x_k \right\rbrace$ a set of $m$ triplets.\\
\textbf{Output}: $\fH{\cdot}$ a strong classifier. 

\begin{algorithmic}[1] 
\STATE Let $\edistrW_1$ be the empirical uniform distribution:\\ $\forall (x_i,y_i) \in \setS, \forall y \in \spaceY, w_{1,x_i,y} = \frac{1}{n\abs{\spaceY}}$.
\FOR{$c = 1,\ldots,C$}
\STATE Choose a triplet classifier $\fh[_c]{}$ according to $\edistrW_c$ (Equation~\eqref{eq:labelstrategy}).
\STATE Compute the weight $\alpha_c$ of $\fh[_c]{}$ according to its performance on $\edistrW_c$ (Equation~\eqref{eq:classifweights}).
\STATE Update the weights of the examples to obtain a new distribution $\edistrW_{c+1}$ (Equation~\eqref{eq:updateweights}).
\ENDFOR
\STATE \textbf{return} $\displaystyle{\fH{\cdot} = \argmax_{y \in \spaceY}\left(\sum_{c=1}^C \alpha_c\indic{\fh[_c]{\cdot} \neq \abstain \wedge y \in \fh[_c]{\cdot}}\right)}$
\end{algorithmic}
\end{algorithm}

Boosting is based on the insight that weak classifiers (that is classifiers marginally better than random predictors) are usually easy to obtain and can be combined in a weighted linear combination to obtain a strong classifier.
This weighted combination can be obtained in an iterative fashion where, at each iteration, a weak-classifier is chosen and weighted so as to minimize the error on the training examples. 
The weights of the points are then updated to put more focus on hard-to-classify examples \citep{schapire2012boosting}.
In this paper we use a well-known boosting algorithm called AdaBoost.MO \citep{schapire1999improved,schapire2012boosting}. This method can handle multi-class problems with a one-against-all approach, works with abstaining classifiers and is theoretically well founded. Algorithm~\ref{alg:tripletboost} summarizes the main steps of our approach that we detail below.

\paragraph{Choosing a triplet classifier.}
To choose a triplet classifier we proceed in two steps.
In the first step, we select two reference points $x_j$ and $x_k$ such that $y_j \neq y_k$. This is done by randomly sampling from an empirical distribution $\edistrW_{c,\spaceX}$ on the examples. 
Here $\edistrW_{c,\spaceX}$ denotes the marginal distribution of $\edistrW_c$ with respect to the examples.
This distribution is updated at each iteration to put more focus on those parts of the space that are hard to classify while promoting triplet classifiers that are able to separate different classes (see Equation~\ref{eq:updateweights}).

In the second step, we choose $o_j$ and $o_k$, the sets of labels that should be predicted for each half space of the triplet classifier.
Given one of the half spaces, we propose to add a label to the set of predicted labels if the weight of examples of this class is greater than the weight of examples of different classes.
Formally, with $w_{c,x_i,y}$ defined as in Algorithm~\ref{alg:tripletboost}, we construct $o_j$ as follows:
\begin{align}
o_j ={}& \Bigg\lbrace y : \sum_{\substack{(x_i,y_i) \in \setS,\\ (x_i,x_j,x_k) \in \setT}} \left(\indic{y=y_i} - \indic{y \neq y_i}\right) w_{c,x_i,y} > 0 \Bigg\rbrace \text{.} \label{eq:labelstrategy}
\end{align}
We construct $o_k$ in a similar way.
The underlying idea is that adding $\fh[_c]{}$ to the current combination of triplet classifiers should improve the predictions on the training set as much as possible.
In Section~\ref{sec:tripletisweak} we show that this strategy is optimal and that it ensures that the selected triplet classifier is either a weak classifier or has a weight $\alpha_c$ of $0$.

\paragraph{Computing the weight of the triplet classifier.}
To choose the weight of the triplet classifier $\fh[_c]{}$ we start by computing $W_{c,+}$ and $W_{c,-}$, the weights of correctly and incorrectly classified examples:
\begin{align}
W_{c,+} ={}& \!\!\!\!\!\! \sum_{\substack{(x_i,y_i) \in \setS,\\ \fh[_c]{x_i} \neq \abstain}} \!\! \Bigg( \indic{y_i \in \fh[_c]{x_i}} w_{c,x_i,y_i} \! + \! \sum_{y \neq y_i} \indic{y \notin \fh[_c]{x_i}} w_{c,x_i,y} \Bigg)  \text{,} \nonumber\\
W_{c,-} ={}& \!\!\!\!\!\! \sum_{\substack{(x_i,y_i) \in \setS,\\ \fh[_c]{x_i} \neq \abstain}} \!\! \Bigg( \indic{y_i \notin \fh[_c]{x_i}} w_{c,x_i,y_i} \! + \! \sum_{y \neq y_i} \indic{y \in \fh[_c]{x_i}} w_{c,x_i,y} \Bigg)  \text{.} \label{eq:classifweights}
\end{align}
We then set $\displaystyle{\alpha_c = \nicefrac{\log{\left(\frac{W_{c,+} + \frac{1}{n}}{W_{c,-} + \frac{1}{n}}\right)}}{2}}$. The term $\frac{1}{n}$ is a smoothing constant \citep{schapire2000boostexter}: in our setting with few, passively queried triplets it helps to avoid numerical problems that might arise when $W_{c,+}$ or $W_{c,-}=0$. 
In Theorem~\ref{th:trainerror} we show that this choice of $\alpha_c$ leads to a decrease in training error as the number of iterations increases.

\paragraph{Updating the weights of the examples.}
In each iteration of our algorithm, a new triplet classifier $\fh[_c]{}$ is added to the weighted combination of classifiers, and we need to update the empirical distribution $\edistrW_{c}$ over the examples for the next iteration. 
The idea is (i) to reduce the weights of correctly classified examples, (ii) to keep constant the weights of the examples for which the current triplet classifier abstains, and (iii) to increase the weights of incorrectly classified examples.
The weights are then normalized by a factor $Z_c$ so that $\edistrW_{c+1}$ remains an empirical distribution over the examples.
Formally, $\forall (x_i,y_i) \in \setS, \forall y \in \spaceY$, if $\fh[_c]{x_i} = \abstain$ then $w_{c+1,x_i,y} = \frac{w_{c,x_i,y}}{Z_c}$ and if $\fh[_c]{x_i} \neq \abstain$ then 
\begin{align}
w_{c+1,x_i,y} = \frac{w_{c,x_i,y}}{Z_c}
 \exp\left\lbrace-\alpha_c\left(\indic{y=y_i} - \indic{y \neq y_i}\right)\right. & \nonumber\\
 &\negspace{9.5em}\left.\times\left(\indic{y \in \fh[_c]{x_i}} - \indic{y \notin \fh[_c]{x_i}}\right)\right\rbrace\text{.} \label{eq:updateweights}
\end{align}

\paragraph{Using $\fH{}$ for prediction.}
Given a new example $x$, TripletBoost predicts its label as 
\begin{align}
\fH{x} = \argmax_{y \in \spaceY}\left(\sum_{c=1}^C \alpha_c\indic{\fh[_c]{x} \neq \abstain \wedge y \in \fh[_c]{x}}\right) \label{eq:strongclassifier}
\end{align}
that is the label with the highest weight as predicted by the weighted combination of selected triplet classifiers.
However, recall that we are in a passive setting, and thus we assume that we are given a set of triplets $\setT_x = \left\lbrace (x,x_j,x_k) : (x_j,y_j),(x_k,y_k) \in \setS, x_j \neq x_k \right\rbrace$ associated with the example $x$ (but there is no way to choose them).
Hence, some of the triplets in $\setT_x$ correspond to triplet classifiers in $\fH{}$ (that is $(x,x_j,x_k) \in \setT_x$ and the reference points for $\fh[_c]{}$ were $x_j$ and $x_k$) and some do not.
In particular, it might happen that none of the triplets in $\setT_x$ corresponds to a triplet classifier in $\fH{}$ and, in this case, $\fH{}$ can only randomly predict a label.
In Section~\ref{sec:lowerbound} we provide a lower bound on the number of triplets necessary to avoid this behaviour.
The main computational bottleneck when predicting the label of a new example $x$ is to check whether the triplets in $\setT_x$ match a triplet classifier in $\fH{}$.
A naive implementation would compare each triplet in $\setT_x$ to each triplet classifier, which can be as expensive as $O(\left|\setT_x\right| C)$.
Fortunately, by first sorting the triplets and the triplet classifiers, a far more reasonable complexity of $O(\left|\setT_x\right|\log(\left|\setT_x\right|) + C\log(C))$ can be achieved.

\section{Theoretical Analysis}\label{sec:theory}

In this section we show that our approach is theoretically well founded.
First we prove that the triplet classifiers with non-zero weights are weak learners: they are slightly better than random predictors (Theorem~\ref{th:tripletisweak}).
Building upon this result we show that, as the number of iterations increases, the training error of the strong classifier learned by TripletBoost is decreased (Theorem~\ref{th:trainerror}).
Then, to ensure that TripletBoost does not over-fit, we derive a generalization bound showing that, given a sufficient amount of training examples, the test error is bounded (Theorem~\ref{th:testerror}).
Finally, we derive a lower bound on the number of triplets necessary to ensure that TripletBoost does not learn a random predictor (Theorem~\ref{th:limit}).

\subsection{Triplet Classifiers and Weak Learners}
\label{sec:tripletisweak}

We start this theoretical analysis by showing that the strategy to choose the predicted labels of triplet classifiers described in Equation~\eqref{eq:labelstrategy} is optimal: it ensures that their error is minimal on the training set (compared to any other labelling strategy). We also show that the triplet classifiers are never worse than random predictors and in fact, that only those triplets  classifiers that are weak classifiers (strictly better than random classifiers) are affected a non-zero weight. This is summarized in the next theorem.

\begin{myth}[Triplet classifiers and weak learners\label{th:tripletisweak}]
Let $\edistrW_c$ be an empirical distribution over $\setS \times \spaceY$ and $\fh[_c]{}$ be the corresponding triplet classifier chosen as described in Section~\ref{sec:tripletboost}.
It holds that:
\begin{enumerate}
    \item the error of $\fh[_c]{}$ on $\edistrW_c$ is at most the error of a random predictor and is minimal compared to other labelling strategies,
    \item the weight $\alpha_c$ of the classifier is non-zero if and only if, $\fh[_c]{}$ is a weak classifier, that is is strictly better than a random predictor.
\end{enumerate}
\end{myth}
\begin{proof} The proof is given in Appendix~\ref{app:sec:weak}.
\end{proof}

\subsection{Boosting guarantees}

From a theoretical point of view the boosting framework has been extensively investigated and it has been shown that most AdaBoost-based methods decrease the training error at each iteration \citep{freund1997decision}.
Another question that has attracted a lot of attention is the problem of generalization.
It is known that when the training error has been minimized, AdaBoost-based methods often do not over-fit and it might even be beneficial to further increase the number of weak learners.
A popular explanation is the margin theory which says that as the number of iterations increases, the confidence of the algorithm in its predictions increases and thus, the test accuracy is improved \citep{schapire1998boosting,breiman1999prediction,wang2011refined,gao2013doubt}.
TripletBoost is based on AdaBoost.MO, and thus it inherits the theoretical guarantees presented above.
In this section, we provide two theorems which show (i) that TripletBoost reduces the training error as the number of iterations increases, and (ii) that it generalizes well to new examples.

The following theorem shows that, as the number of iterations increases, TripletBoost decreases the training error.
\begin{myth}[Reduction of the Training Error\label{th:trainerror}]
Let $\setS$ be a set of $n$ examples and $\setT$ be a set of $m$ triplets (obtained as described in Section~\ref{sec:contrib}).
Let $\fH{\cdot}$ be the classifier obtained after $C$ iterations of TripletBoost (Algorithm~\ref{alg:tripletboost}) using $\setS$ and $\setT$ as input. It holds that:
\begin{align*}
\prob_{(x,y) \in \setS}\left[ \fH{x} \neq y \right] \leq{}& \frac{\abs{\spaceY}}{2}\prod_{c=1}^C Z_c
\end{align*}
with $Z_c = (1-W_{c,+}-W_{c,-}) + (W_{c,+})\cdot\sqrt{\frac{W_{c,-} + \frac{1}{n}}{W_{c,+} + \frac{1}{n}}} + (W_{c,-})\cdot\sqrt{\frac{W_{c,+} + \frac{1}{n}}{W_{c,-} + \frac{1}{n}}} \leq 1$.
\end{myth}
\begin{proof}
This result, inherited from AdaBoost.MO \citep{schapire2012boosting}, is proven in Appendix~\ref{app:sec:train}.
\end{proof}

The next theorem shows that the true error of a classifier learned by TripletBoost can be bounded by a quantity related to the confidence of the classifier on the training examples, with respect to a margin, plus a term which decreases as the number of examples increases.
The confidence of the classifier on the training examples is also bounded and decreases for sufficiently small margins.
\begin{myth}[Generalization Guarantees\label{th:testerror}]
Let $\distrS$ be a distribution over $\spaceX \times \spaceY$, let $\setS$ be a set of $n$ examples drawn i.i.d. from $\distrS$, and let $\setT$ be a set of $m$ triplets (obtained as described in Section~\ref{sec:contrib}).
Let $\fH{\cdot}$ be the classifier obtained after $C$ iterations of TripletBoost (Algorithm~\ref{alg:tripletboost}) using $\setS$ and $\setT$ as input.
Let $\spaceH$ be a set of triplet classifiers as defined in Section~\ref{sec:tripletclassifiers}.
Then, given a margin parameter $\theta > \sqrt{\frac{\log{\abs{\spaceH}}}{16\abs{\spaceY}^2n}}$ and a measure of the confidence of $\fH{\cdot}$ in its predictions $\ftheta[_{\fH{}}]{x,y}$, with probability at least $1-\delta$, we have that
\begin{align*}
\prob_{(x,y) \sim \distrS}\left[H(x) \neq y\right] \leq{}& \prob_{(x,y) \in \setS}\left[ \ftheta[_{\fH{}}]{x,y} \leq \theta \right] \\
&\negspace{6em}+ \bigO{\sqrt{\frac{\log{\left(\frac{1}{\delta}\right)}}{n} + \log{\left(\frac{\abs{\spaceY}^2n\theta^2}{\log{\abs{\spaceH}}}\right)}\frac{\log{\abs{\spaceH}}}{n\theta^2}}}
\end{align*}
Furthermore we have that 
\begin{align*}
\prob_{(x,y) \in \setS}\left[ \ftheta[_{\fH{}}]{x,y} \leq \theta \right] \leq \frac{\abs{\spaceY}}{2} \prod_{c=1}^C Z_c \sqrt{\left(\frac{W_{c,+} + \frac{1}{n}}{W_{c,-} + \frac{1}{n}}\right)^\theta} \text{.}
\end{align*}
\end{myth}

\begin{proof}
This result, inherited from AdaBoost.MO \citep{schapire2012boosting}, is proven in Appendix~\ref{app:sec:test}.
\end{proof}

At a first glance it seems that this bound does not depend on $m$, the number of available triplets. However, this dependency is implicit: $m$ impacts the probability that the training examples are well classified with a large margin $\theta$.
If the number of triplets is small, the probability that the training examples are well classified with a given margin is small.
This probability increases when the number of triplets increases. We illustrate this behaviour in Appendix~\ref{app:sec:discussion}.

To prove a bound that explicitly depends on $m$ would be of significant interest. 
However this is a difficult problem, as it requires to use an explicit measure of complexity for general metric spaces, which is beyond the scope of this paper.

\begin{figure*}
    \centering
    \subfloat[Moons, Metric: Euclidean, Noise Level: $0\%$\label{fig:smallscaletriplets}]{\includegraphics[scale=0.305]{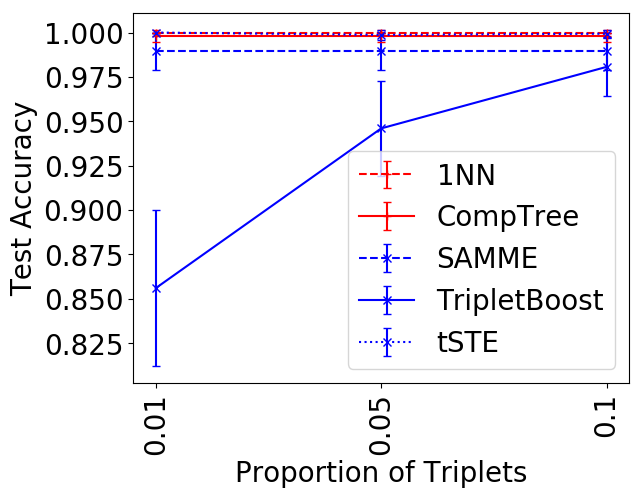}}
    \hspace{1em}
    \subfloat[Moons, Metric: Euclidean, Proportion of Triplets: $10\%$\label{fig:smallscalenoise}]{\includegraphics[scale=0.305]{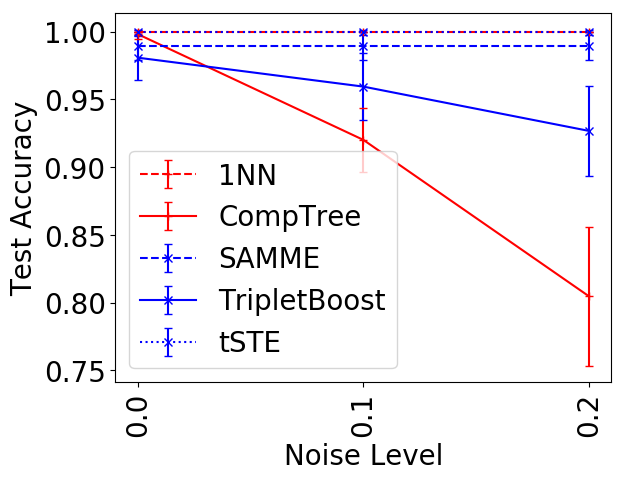}}
    \hspace{1em}
    \subfloat[Iris, Proportion of Triplets: $10\%$, Noise Level: $0\%$\label{fig:metricsensitivity}]{\includegraphics[scale=0.305]{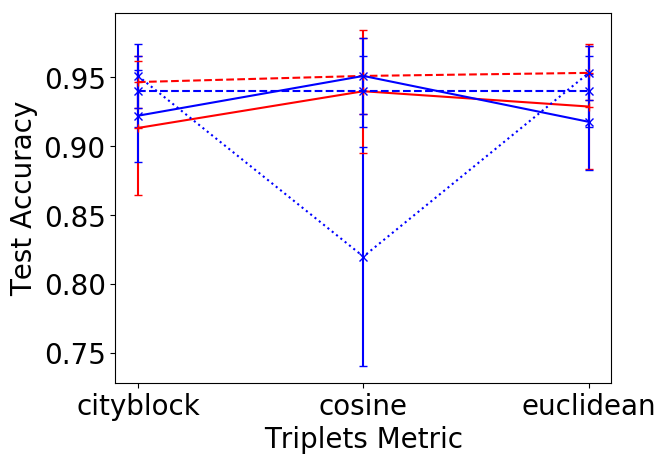}}
    
    \caption{Moons is a small scale dataset with $500$ examples and $2$ dimensions. We fix the triplets metric to the Euclidean distance. In Figure~\ref{fig:smallscaletriplets} we consider the noise free setting and we vary the proportion of available triplets from $1$ to $10\%$ of all the triplets.
    In Figure~\ref{fig:smallscalenoise} we fix this proportion to $10\%$ and we vary the noise level from $0$ to $20\%$.
    Iris is a small scale dataset with $150$ examples and $4$ dimensions.
    In Figure~\ref{fig:metricsensitivity} we fix the proportion of available triplets to $10\%$, the noise level to $0\%$ and we vary the metric considered to generate the triplets.
    \label{fig:smallscale}}
\end{figure*}

\subsection{Lower bound on the number of triplets}
\label{sec:lowerbound}

In this section we investigate the minimum number of triplets that are necessary to ensure that our algorithm performs well.
Ideally, we would like to obtain a lower bound on the number of triplets that are necessary to achieve a given accuracy. In this paper we take a first step in this direction by deriving a bound on the number of triplets that are necessary to ensure that the learned classifier does not abstain on any unseen example.
Theorems~\ref{th:limit}~and~\ref{th:equality} show that it abstains with high probability if it is learned using too few triplets or if it combines too few triplet classifiers. 

\begin{myth}[Lower bound on the probability that a strong classifier abstains\label{th:limit}]
Let $n \geq 2$ be the number of training examples, $p = \frac{2n^k}{n^3}$ with $k \in \left[0,3-\frac{\log(2)}{\log(n)}\right)$ be the probability that a triplet is available in the triplet set $\setT$ and $C = \frac{n^\beta}{2}$ with $\beta \in \left[0,1+\frac{\log{(n-1)}}{\log{(n)}}\right]$ be the number of classifiers combined in the learned classifier.
Let $\algA$ be any algorithm learning a classifier $\fH{\cdot} = \argmax_{y \in \spaceY}\left(\sum_{c=1}^C \alpha_c \indic{y \in \fh[_c]{\cdot}}\right)$ that combines several triplet classifiers using some weights $\alpha_c \in \nsetR$. Assume that triplet classifiers that abstain on all the training examples have a weight of $0$ (that is if $\fh[_c]{x_i} = \abstain$ for all the examples $(x_i,y_i) \in \setS$ then $\alpha_c = 0$).
Then the probability that $\fH{}$ abstains on a test example is bounded as follows:
\begin{align}
\prob_{(x,y) \sim \distrS} \left[ \fH{x} = \abstain \right] \geq{}& \Big(1-p + p\left(1-p\right)^n \Big)^C\text{.} \label{eq:boundexact}
\end{align}
\end{myth}

\begin{proof}
The proof is given in Appendix~\ref{app:sec:limit}.
\end{proof}

To understand the implications of this theorem we consider a concrete example.

\begin{myex}\label{ex:limit}
Assume that we build a linear combination of all possible triplet classifiers, that is 
$C = \frac{n(n-1)}{2}$.
Then we have 
\begin{align}
\!\!\!\lim_{n\rightarrow+\infty}\prob_{(x,y) \sim \distrS}\!\left[ \fH{x} = \abstain \right] \!\geq{}&\!\!
\left\lbrace
\begin{array}{ll}
\!\!\!1 & \!\!\!\!\text{if $k<\frac{3}{2}$,} \\
\!\!\!\exp(-2) & \!\!\!\!\text{if $k=\frac{3}{2}$,} \\
\!\!\!0 & \!\!\!\!\text{if $\frac{3}{2}< k$,}
\end{array}
\right. \label{eq:boundlimit}
\end{align}
where $k$ is the parameter that controls the probability $p$ that a particular triplet is available in the triplets set $T$.
The bottom line is that when $k<\frac{3}{2}$, that is when we do not have at least $\Omega(n\sqrt{n})$ random triplets, the learned classifier abstains on all the examples.
\end{myex}

\begin{proof}
The proof is given in Appendix~\ref{app:sec:limitex}.
\end{proof}

Theorem~\ref{th:limit} shows that when $p$ and $C$ are too small, then the strong classifier abstains with high probability. However, the theorem does not guarantee that the strong classifier does not abstain when $p$ and $C$ are large. The next theorem takes care of this other direction under slightly stronger assumptions on the weights learned by the algorithm. 

\begin{myth}[Exact bound on the probability that a strong classifier abstains\label{th:equality}]
In Theorem~\ref{th:limit}, further assume that each triplet classifier that does not abstain on at least one training example has a weight different from $0$ (if for at least one example $(x_i,y_i) \in \setS$ we have that $\fh[_c]{x_i} \neq \abstain$ then $\alpha_c \neq 0$). Then equality holds in Equation~\eqref{eq:boundexact}.
\end{myth}
\begin{proof}
The proof is given in Appendix~\ref{app:sec:equality}.
\end{proof}

Theorem~\ref{th:equality} implies that equality holds in Example~\ref{ex:limit}, thus when $C = \frac{n(n-1)}{2}$ we need at least $k>\frac{3}{2}$, that is at least $\Omega(n\sqrt{n})$ random triplets, to obtain a classifier that never abstains. In Appendix~\ref{app:sec:limitex} we extend Example~\ref{ex:limit} and we study the limit as $n \rightarrow \infty$ for general values of $C$ and $p$. We also provide a graphical illustration of the bound and a discussion on how this lower bound compares to existing results \citep{ailon2012active,jamieson2011low,jain2016finite}.

\section{Experiments}\label{sec:expes}

\begin{figure*}
    \centering
    \subfloat[Gisette, Metric: Euclidean, Noise Level: $0\%$\label{fig:largescaletriplets}]{\includegraphics[scale=0.305]{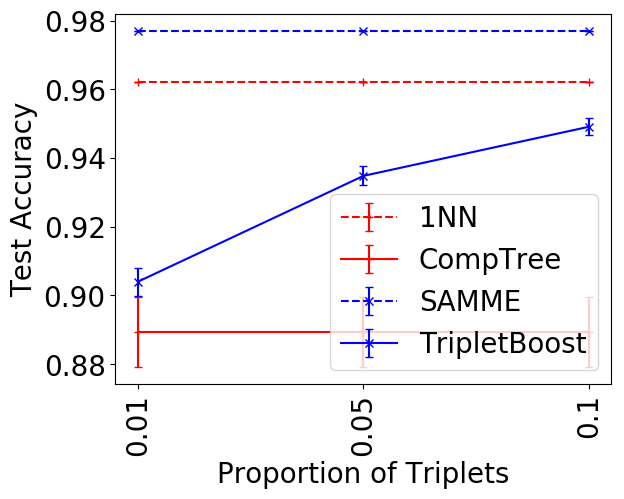}}
    \hspace{1em}
    \subfloat[Gisette, Metric: Euclidean, Proportion of Triplets: $5\%$\label{fig:largescalenoise}]{\includegraphics[scale=0.305]{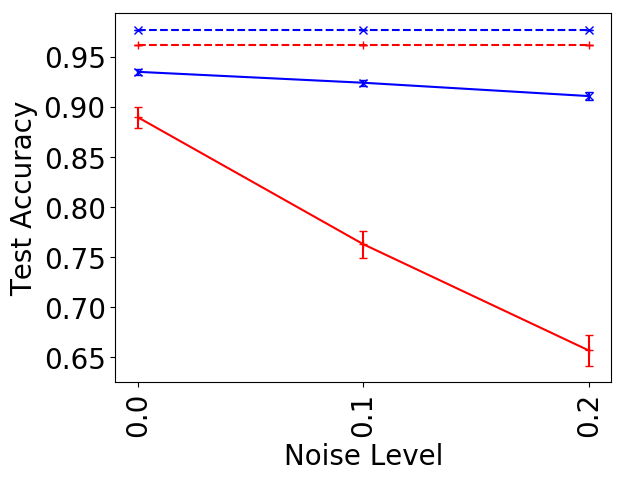}}
    \hspace{1em}
    \subfloat[Moons, Metric: Euclidean, Proportion of Triplets: $10\%$, Noise Level: $10\%$\label{fig:timing}]{\includegraphics[scale=0.305]{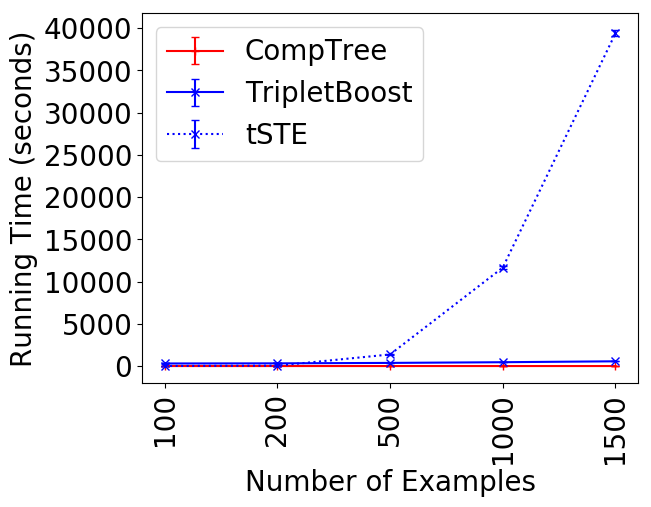}}
    
    \caption{The Gisette dataset has $7000$ examples and $5000$ dimensions. In Figure~\ref{fig:largescaletriplets} we consider the noise free setting and we vary the proportion of triplets available from $1$ to $10\%$ of all the triplets. In Figure~\ref{fig:largescalenoise} we fix the proportion of available triplets to $5\%$ and we vary the noise level from $0$ to $20\%$.
    Figure~\ref{fig:timing} presents the training time of the triplet-based methods with training samples of increasing sizes on the Moons dataset.
    The proportion of triplets and the noise level were both set to $10\%$.
    The results were obtained on a single core @3.40GHz. \label{fig:largescale}}
\end{figure*}

We propose an empirical evaluation of TripletBoost. We consider six datasets of varying scales and four baselines.

\paragraph{Baselines.} First, we consider an embedding approach. We use tSTE \citep{van2012stochastic} to embed the triplets in a Euclidean space and we use the 1-nearest neighbour algorithm for classification. We also would like to compare to alternative approaches able to learn directly from triplets (without embedding as a first step). However, to the best of our knowledge, TripletBoost is the only method able to do classification using only passively obtained triplets. The only option is to choose competing methods that have access to more information (providing them an unfair advantage). We settled for a method that uses actively chosen triplets to build a comparison-tree to retrieve nearest neighbours (CompTree) \citep{haghiri2017comparison}. Finally, to put the results obtained in the triplet setting in perspective, we consider two methods that use the original Euclidean representations of the data, the 1-nearest neighbour algorithm (1NN) and AdaBoost.SAMME (SAMME) \citep{hastie2009multi}.

\paragraph{Implementation details.} For tSTE we used the implementation distributed on the authors' website and we set the embedding dimension to the original dimension of the data. This method was only considered for small datasets with less than $10^3$ examples as it does not scale well to bigger datasets (Figure~\ref{fig:timing}).
For CompTree we used our own implementation and the leaf size of the comparison tree is set to $1$ as this is the only value for which this method can handle noise.
For 1NN and SAMME we used sk-learn \citep{scikit-learn}. The number of boosting iterations for SAMME is set to $10^3$.
Finally for TripletBoost we set the number of iterations to $10^6$.

\paragraph{Datasets and performance measure.} We consider six datasets: Iris, Moons, Gisette, Cod-rna, MNIST, and kMNIST.
For each dataset we generate some triplets as in Equation~\eqref{def:triplet} using three metrics: the Euclidean, Cosine, and Cityblock distances (details provided in Appendix~\ref{app:sec:details}).
Given a set of $n$ examples there are $\nicefrac{n^2(n-1)}{2}$ possible triplets.
We consider three different regimes where $1\%$, $5\%$ or $10\%$ of them are available, and we consider three noise levels where $0\%$, $10\%$ or $20\%$ of them are incorrect. 
We measure performances in terms of test accuracy (higher is better).
For all the experiments we report the mean and standard deviation of $10$ repetitions.
Since the results are mostly consistent across the datasets we present some representative ones here and defer the others to Appendix~\ref{app:sec:details}.

\paragraph{Small scale regime.} We first consider the datasets with less than $10^3$ training examples (Figure~\ref{fig:smallscale}).
In this setting our method does not perform well when the number of triplets is too small, but gets closer to the baselines when the number of triplets increases (Figure~\ref{fig:smallscaletriplets}).
This behaviour can be easily explained: when only $1\%$ of the triplets are available, the triplet classifiers abstain on all but $3$ or $4$ examples on average and thus their performance evaluations are not reliable.
Consequently their weights cannot be chosen in a satisfactory manner. 
This problem vanishes when the number of triplets increases.
With increasing noise levels (Figure~\ref{fig:smallscalenoise}) one can notice that TripletBoost is more robust than CompTree.
Indeed, CompTree generates a comparison-tree based on individual triplet queries, and the greedy decisions in the tree building procedure can easily be misleading in the presence of noise. 
Finally, our approach is less sensitive than tSTE to changes in the metric that generated the triplets (Figure~\ref{fig:metricsensitivity}).
Indeed, tSTE assumes that this metric is the Euclidean distance while our approach does not make any assumptions.

\paragraph{Large scale regime.} On larger datasets (Figure~\ref{fig:largescale}), our method does not reach the accuracy of 1NN and SAMME, who exploit a significant amount of extra information. Still, it performs quite well and is competitive with CompTree, the method that uses active rather than passive queries.
The ordinal embedding methods cannot compete in this regime, as they are too slow to even finish (Figure~\ref{fig:timing}). 
Once again TripletBoost is quite resistant to noise (Figure~\ref{fig:largescalenoise}).

\section{Conclusion}\label{sec:conclu}

In this paper we proposed TripletBoost to address the problem of comparison-based classification. It is particularly designed for situations where triplets cannot be queried actively, and we have to live with whatever set of triplets we get. We do not make any geometric assumptions on the underlying space. 
From a theoretical point of view we have shown that TripletBoost is well founded and we proved guarantees on both the training error and the generalization error of the learned classifier.
Furthermore we derived a new lower bound showing that to avoid learning a random predictor, at least $\Omega(n\sqrt{n})$ triplets are needed.
In practice we have shown that, given a sufficient amount of triplets, our method is competitive with state of the art methods and that it is quite resistant to noise.

To the best of our knowledge, TripletBoost is the first algorithm that is able to handle large scale datasets using only passively obtained triplets.
This means that the comparison-based setting could be considered for problems which were, until now, out of reach.
As an illustration, consider a platform where users can watch, comment and rate movies.
It is reasonable to assume that triplets of the form \emph{movie $m_i$ is closer to $m_j$ than to $m_k$} can be automatically obtained using the ratings of the users, their comments, or their interactions.
In this scenario, active learning methods are not applicable since the users might be reluctant to answer solicitations.
Similarly, embedding methods are too slow to handle large numbers of movies or users.
However, we can use TripletBoost to solve problems such as predicting the genres of the movies.
As a proof of concept we considered the 1m movielens dataset \citep{harper2016movielens}.
It contains 1 million ratings from $6040$ users on $3706$ movies.
We used the users' ratings to obtain some triplets about the movies and TripletBoost to learn a classifier able to predict the genres of a new movie (more details are given in Appendix~\ref{app:sec:detailsmovie}).
Given a new movie, in $\sim\!\!83\%$ of the cases the genre predicted as the most likely one is correct and, on average, the $5$ genres predicted as the most likely ones cover $\sim\!\!92\%$ of the true genres.

\section*{Acknowledgments}

Ulrike von Luxburg acknowledges funding by the DFG through the Institutional Strategy of the University of Tübingen (DFG, ZUK 63) and the Cluster of Excellence EXC 2064/1, project number 390727645.

\nocite{clanuwat2018deep,lecun1998gradient,uzilov2006detection,guyon2005result,Dua:2017,scikit-learn}
\bibliographystyle{named}
\bibliography{ijcai19}

\onecolumn
\appendix

\section{Discussion on Theorem~\ref{th:testerror}}
\label{app:sec:discussion}

At a first glance it seems that the bound presented in Theorem~\ref{th:testerror} does not depend on $m$, the number of available triplets. However, this dependency is implicit: $m$ impacts the probability that the training examples are well classified with a large margin $\theta$.
If the number of triplets is small, the probability that the training examples are well classified with a given margin is small.
This probability increases when the number of triplets increases.

To illustrate this behaviour, consider a fixed margin $\theta$.
Assume that each of the $n$ training examples is classified by exactly one triplet classifier and that each triplet classifier abstains on all but one example.
In this case, the only way to have no error on the training set is to combine the $n$ triplet classifiers that do not abstain.
For simplicity consider the case where all the triplet classifiers have uniform weight $\frac{1}{n}$ in the final classifier. Then the fixed margin will not be achieved when the number of training examples increases. The first term on the right hand side of the generalization bound will be $1$, that is the bound predicts that the learned classifier might not generalize.
This fact remains true for any weighting scheme.
In this example, the best weighting scheme classifies with a margin at least $\theta$, at most $\frac{1}{\theta}$ training examples.

When the number of triplets increases, the value of $Z_c$ decreases, because the proportion of examples on which the selected triplet classifier abstains, $1 - W_{c,+} - W_{c,-}$, decreases. Similarly, the proportion of training examples that are classified with a margin at least $\theta$ increases. Hence the first term on the right hand side of the generalization bound is greatly reduced and the learned classifier generalizes well.

\section{Illustration and Discussion on the Lower Bound of Theorems~\ref{th:limit}~and~\ref{th:equality}}
\label{app:sec:illustration}

When $n=100$ we illustrate the bound obtained in Theorems~\ref{th:limit}~and~\ref{th:equality} in Figure~\ref{app:fig:illustrationlimit}.
This figure shows that the transition between abstaining and non-abstaining classifier depends on both the proportion of triplets available and the number of classifiers considered.
In particular, when the number of combined classifiers increases, one needs a smaller number of triplets.
Conversely when the number of triplets available increases, one can consider combining fewer classifiers.
This illustration confirms the result obtained in Example~\ref{ex:limit} that shows that when $C = \frac{n(n-1)}{2}$ the number of available triplets should at least scale as $\Omega(n\sqrt{n})$ to obtain a non-trivial classifier that does not abstain on most of the test examples. The number of triplets that are necessary to achieve good classification accuracy might be higher but is lower bounded by this value.

This lower bound does not contradict existing results \citep{ailon2012active,jamieson2011low,jain2016finite}. They were developed in the different context of triplet recovery, where the goal is not classification, but to predict the outcome of unobserved triplet questions. 
For example it has been shown that to exactly recover all the triplets, the number of passively available triplets should scale in $\Omega(n^3)$ \citep{jamieson2011low}. Similarly \citet{jain2016finite} derive a finite error bound for approximate recovery of the Euclidean Gram matrix.
Our bound shows that, in a classification setting, it might be possible to do better than that. To set a complete picture, one would need to derive an upper bound on the number of triplets necessary for good classification accuracy. 

\section{Details on the Experiments}
\label{app:sec:details}

The characteristics of the different datasets are given in Table~\ref{app:tab:datasets}. To generate the triplets we used three different metrics:
\begin{itemize}
\item Euclidean distance: $\fd{x,y} = \normTwo{x-y}$,
\item Cityblock distance: $\fd{x,y} = \normOne{x-y}$,
\item Cosine distance: $\fd{x,y} = 1-\frac{\dotprod{x}{y}}{\normTwo{x}\normTwo{y}}$.
\end{itemize}
The set $\spaceT$ of all triplets is then defined as follows:
\begin{align*}
    \spaceT = \left\lbrace (x_i,x_j,x_k) : \fd{x_i,x_j} < \fd{x_i,x_k} \right\rbrace\text{.}
\end{align*}
In all the experiments we considered subsets $T$ of $\spaceT$ by selecting uniformly at random without replacement $1\%$, $5\%$, or $10\%$ of the triplets. Similarly we added some noise by randomly swapping $0\%$, $10\%$, or $20\%$ of the triplets, that is $(x_i,x_k,x_j) \in T$ while $(x_i,x_j,x_k) \in \spaceT$.

The results that were omitted in the main paper are given in Figures~\ref{app:fig:irisnoise}~to~\ref{app:fig:kmnisttriplets}.

\begin{figure}
\centering
\subfloat[With respect to $k$ and $\beta$.]{\includegraphics[scale=0.5]{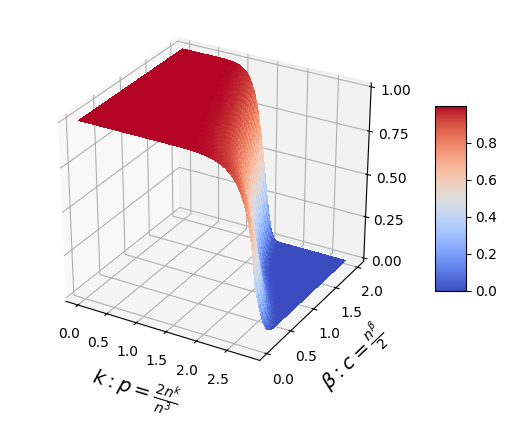}}
\subfloat[With $C=\frac{n(n-1)}{2}$.\label{app:fig:illustrationlimitC}]{\includegraphics[scale=0.42]{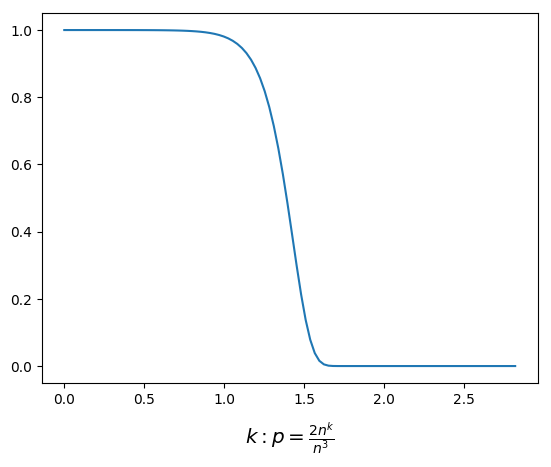}}
\caption{Illustration of the bound in Theorems~\ref{th:limit}~and~\ref{th:equality} when $n = 100$.\label{app:fig:illustrationlimit}}
\end{figure}

\begin{table}
\centering
\begin{tabular}{lrrrr}  
\toprule
Dataset  & Dimension & Train / Test & Classes & Results \\
\midrule
Iris & 4 & 105 / 45 & 3 & Figures~\ref{app:fig:irisnoise}~and~\ref{app:fig:iristriplets} \\
Moons & 2 & 350 / 150 & 2 & Figures~\ref{app:fig:moonsnoise}~and~\ref{app:fig:moonstriplets} \\
Gisette & 5000 & 6000 / 1000 & 2 & Figures~\ref{app:fig:gisettenoise}~and~\ref{app:fig:gisettetriplets}  \\
CodRna & 8 & 59535 / 271617 & 2 & Figures~\ref{app:fig:codrnanoise}~and~\ref{app:fig:codrnatriplets} \\
MNIST & 784 & 60000 / 10000 & 10 & Figures~\ref{app:fig:mnistnoise}~and~\ref{app:fig:mnisttriplets} \\
kMNIST & 784 & 60000 / 10000 & 10 & Figures~\ref{app:fig:kmnistnoise}~and~\ref{app:fig:kmnisttriplets} \\
\bottomrule
\end{tabular}
\caption{Summary of the different datasets.\label{app:tab:datasets}}
\end{table}

\section{Details on the Movielens experiment}
\label{app:sec:detailsmovie}

As a proof of concept we considered the 1m movielens dataset \citep{harper2016movielens}.
This dataset contains 1 million ratings from $6040$ users on $3706$ movies and each movie has one or several genres (there is $18$ genres in total).
To demonstrate the interest of our approach we proposed (i) to use the users' ratings to obtain some triplets of the form \emph{movie $m_i$ is closer to movie $m_j$ than to movie $m_k$}, and (ii) to use TripletBoost to learn a classifier predicting the genres of the movies.

To generate the triplets we propose to consider that movie $m_i$ is closer to $m_j$ than to $m_k$ if, on average, users that rated all three movies rated $m_i$ and $m_j$ more similarly than $m_i$ and $m_k$.
The underlying intuition is that users like and dislike genres of movies --- for example a user that dislikes horror movies but likes comedy movies will probably give low ratings to The Ring (2002) and Scream (1996) and a higher rating to The Big Lebowsky (1998). 
Formally let $r_{u,i}$ be the rating of user $u$ on movie $m_i$ and $r_{u,i,j} = \abs{r_{u,i}-r_{u,j}}$ then the triplet set $\setT$ is
\begin{align*}
\setT ={}& \left\lbrace (m_i,m_j,m_k) : \left(\sum_{u \in U_{i,j,k}} \frac{\indic{r_{u,i,j} < r_{u,i,k}}- \indic{r_{u,i,j} > r_{u,i,k}}}{\abs{U_{i,j,k}}}\right) > 0 \right\rbrace
\end{align*}
where $U_{i,j,k}$ is the set of users that rated all three movies.
Each user has only rated a small number of movies and might give a high, respectively low, rating to a movie with a genre that he usually rates lower, respectively higher.
Thus we only have access to a noisy subset of all the possible triplets.

We used a random sample of $2595$ movies, and their corresponding triplets, to learn a multi-label classifier using TripletBoost (with $10^6$ boosting iterations).
Since we are in a multi-label setting we would like to predict how relevant each genre is for a new movie rather than a single genre.
To obtain such a quantity we can simply ignore the $\argmax$ in Equation~(5) in the main paper to obtain a classifier $\fH{m,y}$ that predicts the weight of a genre $y$ for a movie $m$:
\begin{align*}
\fH{m,y} = \sum_{c=1}^C \alpha_c\indic{\fh[_c]{m} \neq \abstain \wedge y \in \fh[_c]{m}}\text{.}
\end{align*}

In the test phase we used the $1111$ remaining movies to measure the performance of the learned classifier. 
First we considered the precision of the genre predicted with the highest weight and obtained a value of $0.83168$. 
It means that in $\sim\!\!83\%$ of the cases the genre predicted with the highest weight is correct. 
We also considered the recall of the $5$ genres predicted with the highest weights and obtained a value of $0.92943$. 
It implies that, on average, the $5$ genres predicted with the highest weights cover $\sim\!\!92\%$ of the genres of the considered movie.
\vfill

\begin{figure}[H]
    \centering
    \subfloat{\rotatebox{90}{\hspace{5em}Euclidean}\hspace{2em}}  \subfloat{\includegraphics[scale=0.32]{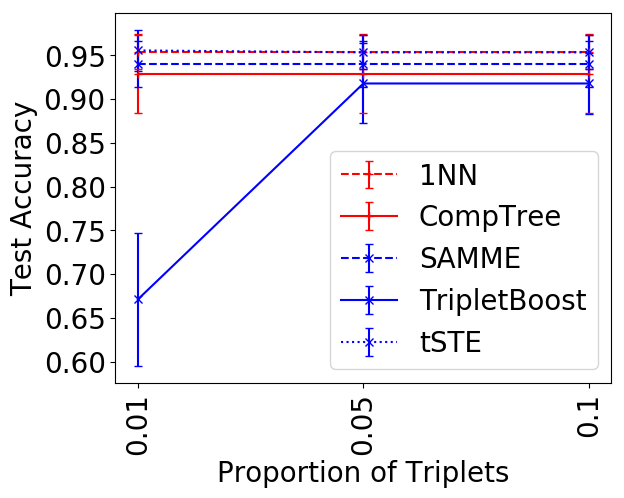}}
    \subfloat{\includegraphics[scale=0.32]{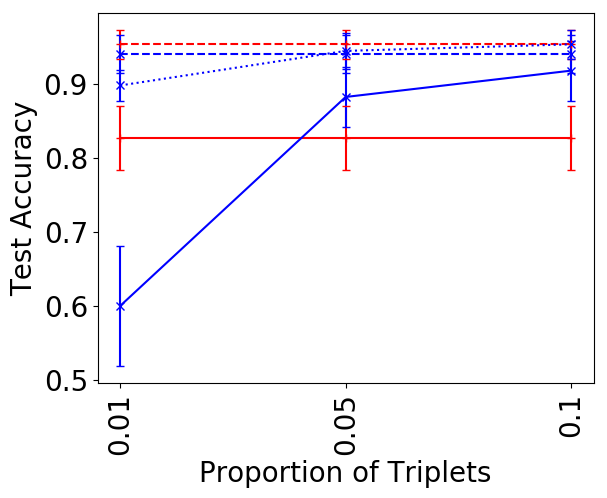}}
    \subfloat{\includegraphics[scale=0.32]{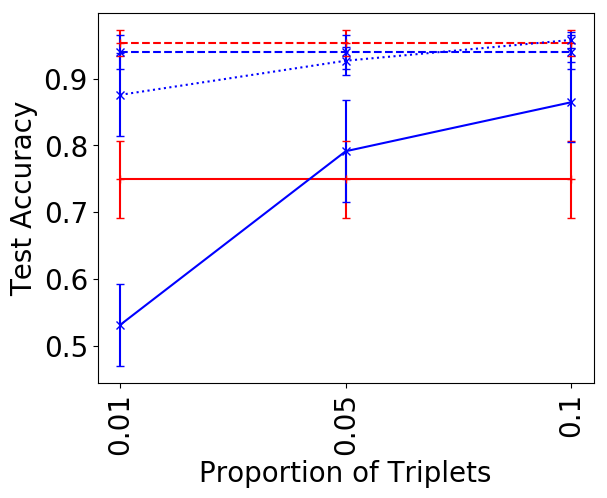}}
    
    \subfloat{\rotatebox{90}{\hspace{5em}Cosine}\hspace{2em}} 
    \subfloat{\includegraphics[scale=0.32]{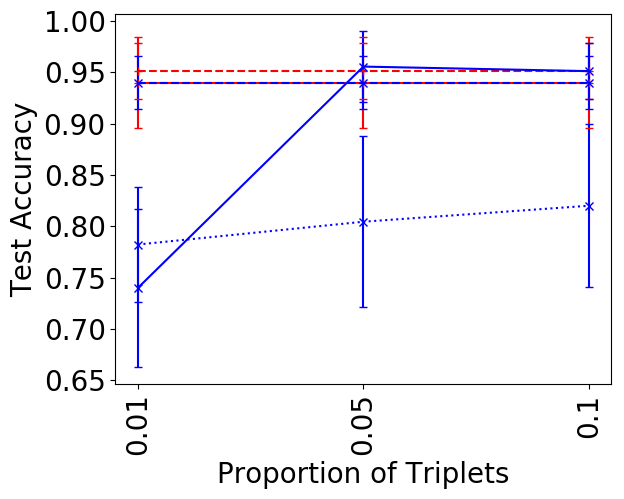}}
    \subfloat{\includegraphics[scale=0.32]{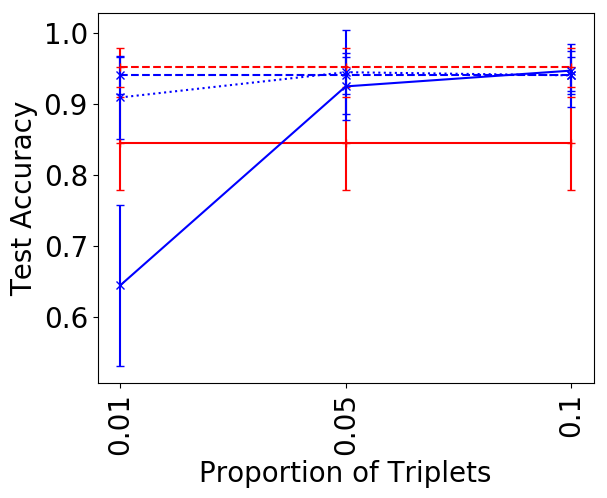}}
    \subfloat{\includegraphics[scale=0.32]{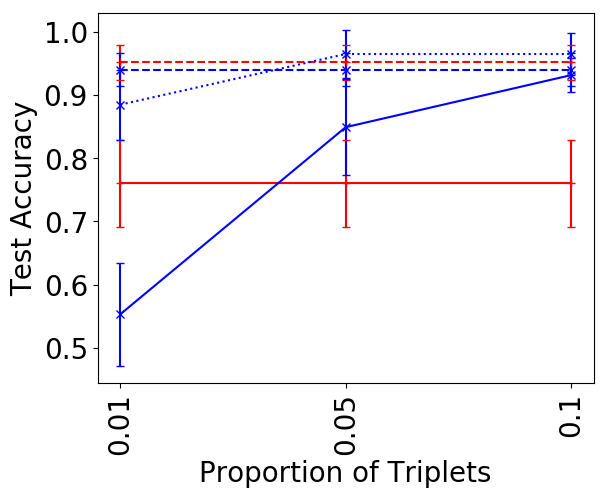}}
    
    \subfloat{\rotatebox{90}{\hspace{5em}Cityblock}\hspace{2em}} 
    \subfloat[Noise Level: $0\%$]{\includegraphics[scale=0.32]{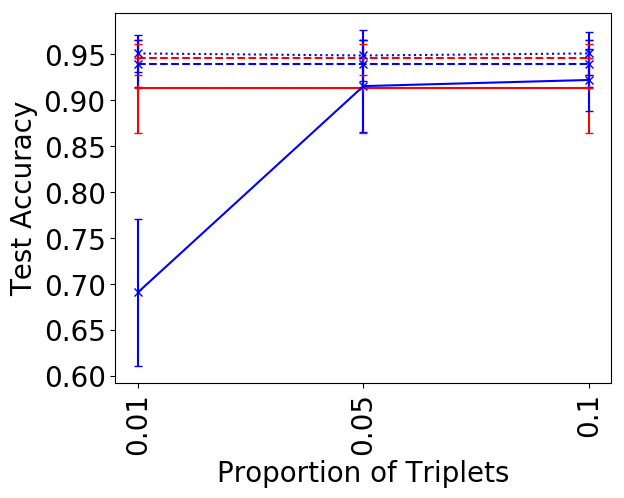}}
    \subfloat[Noise Level: $10\%$]{\includegraphics[scale=0.32]{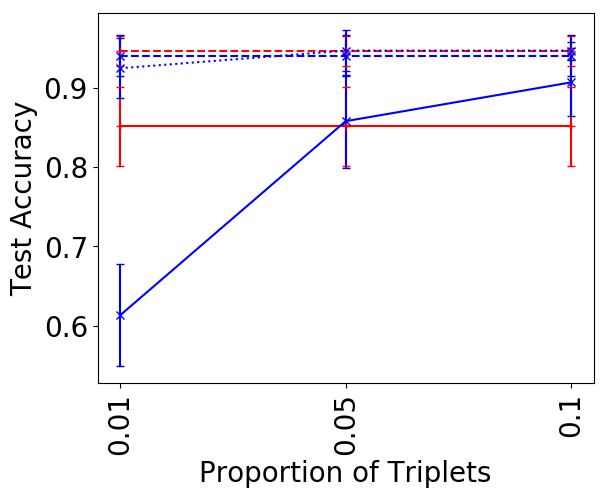}}
    \subfloat[Noise Level: $20\%$]{\includegraphics[scale=0.32]{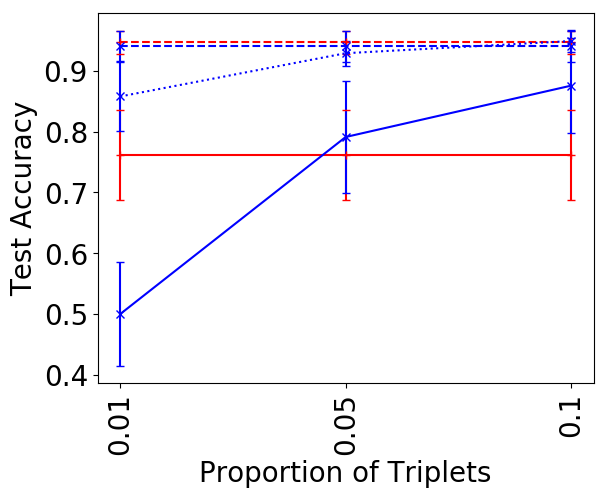}}
    
    \caption{Results on the Iris dataset. In each row we consider a different metric to generate the triplets. In each column we consider a different noise level from $0$ to $20\%$. In each plot we vary the proportion of triplets available from $1$ to $10\%$.\label{app:fig:irisnoise}}
\end{figure}

\begin{figure}[H]
    \centering
    \subfloat{\rotatebox{90}{\hspace{5em}Euclidean}\hspace{2em}}  \subfloat{\includegraphics[scale=0.32]{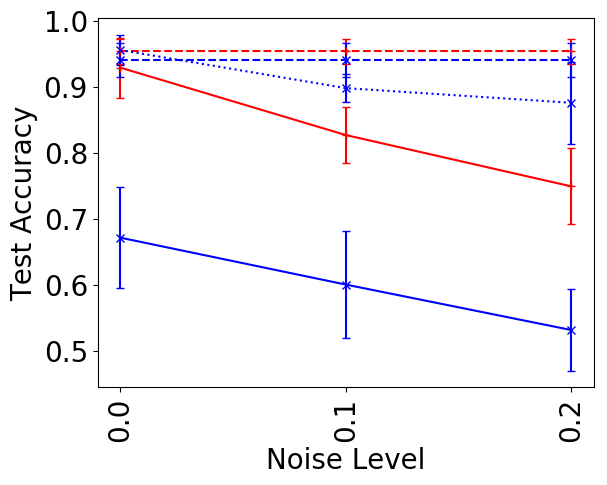}}
    \subfloat{\includegraphics[scale=0.32]{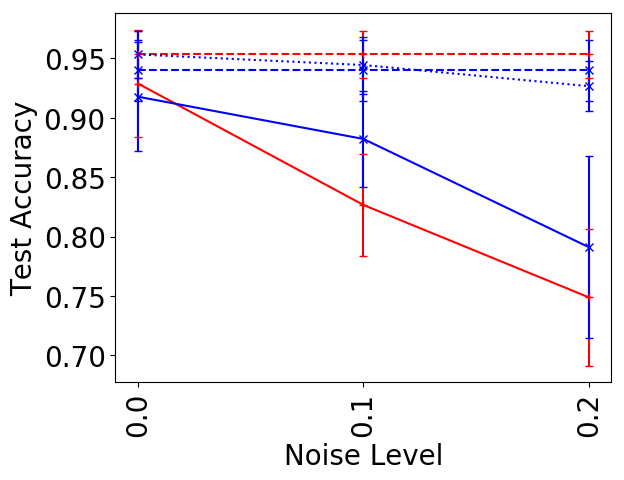}}
    \subfloat{\includegraphics[scale=0.32]{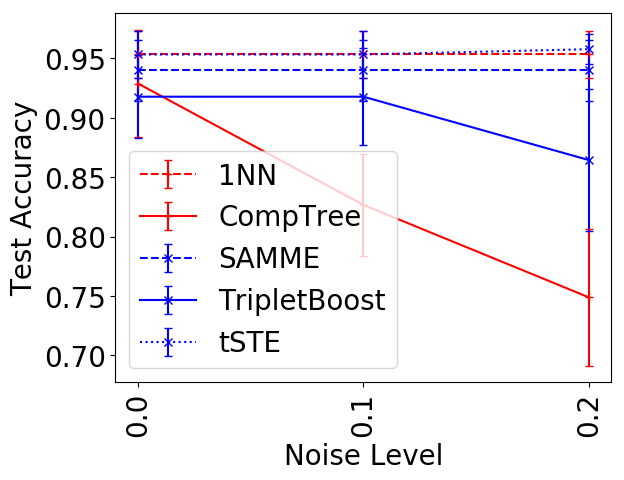}}
    
    \subfloat{\rotatebox{90}{\hspace{5em}Cosine}\hspace{2em}} 
    \subfloat{\includegraphics[scale=0.32]{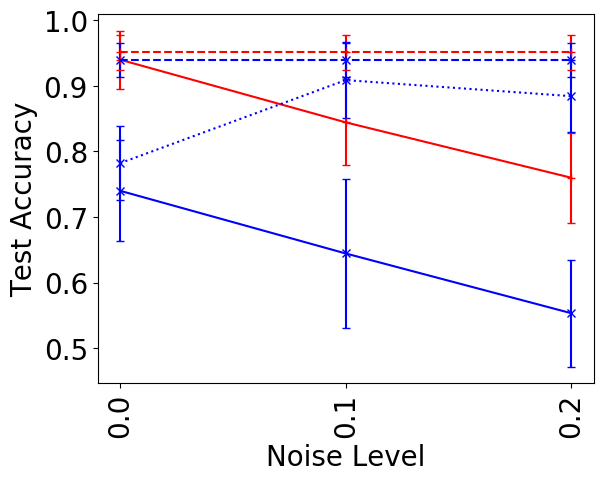}}
    \subfloat{\includegraphics[scale=0.32]{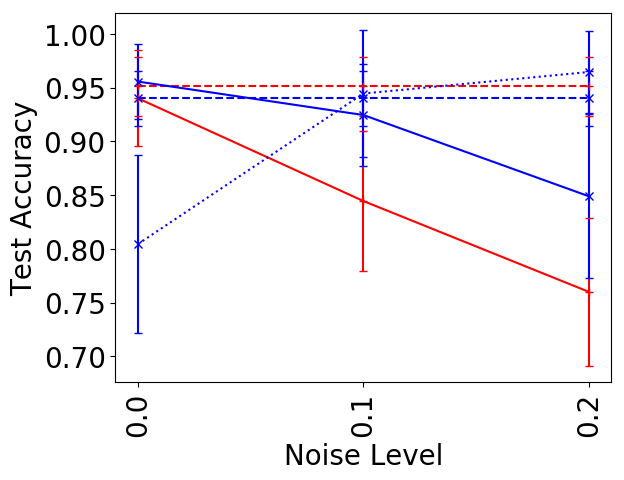}}
    \subfloat{\includegraphics[scale=0.32]{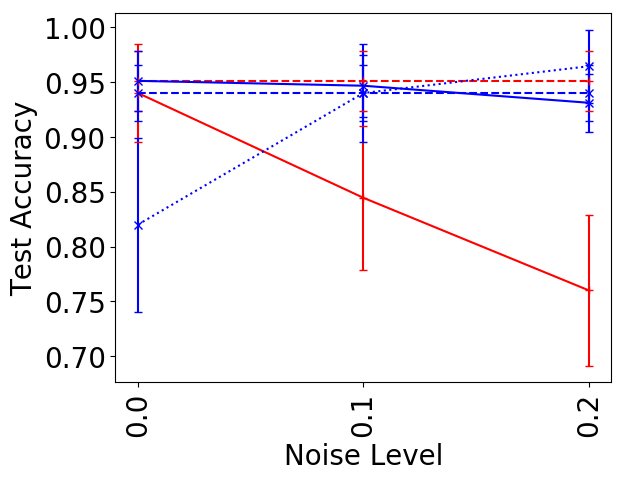}}
    
    \subfloat{\rotatebox{90}{\hspace{5em}Cityblock}\hspace{2em}} 
    \subfloat[Proportion of Triplets: $1\%$]{\includegraphics[scale=0.32]{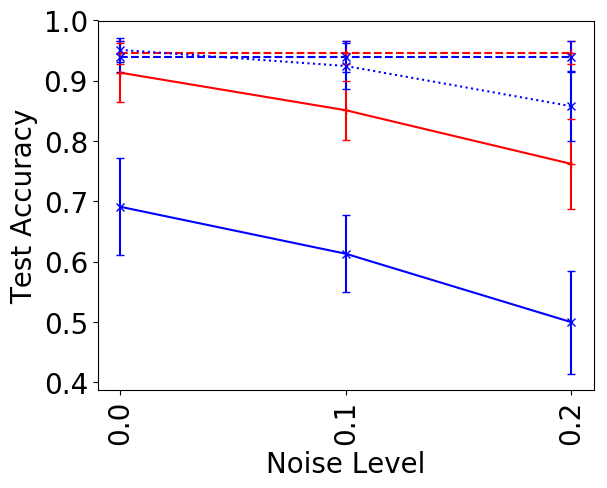}}
    \subfloat[Proportion of Triplets: $5\%$]{\includegraphics[scale=0.32]{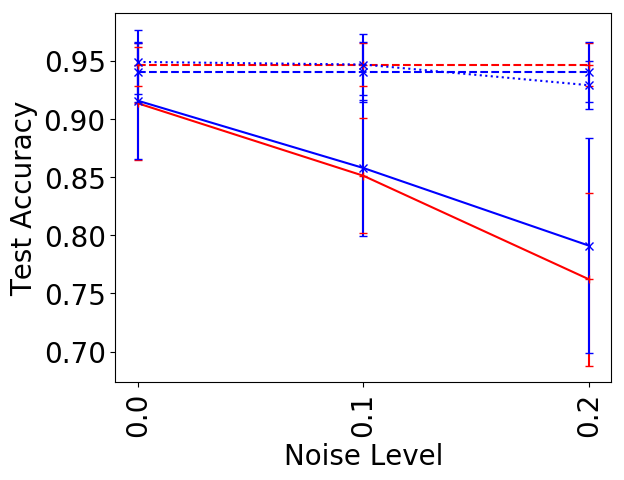}}
    \subfloat[Proportion of Triplets: $10\%$]{\includegraphics[scale=0.32]{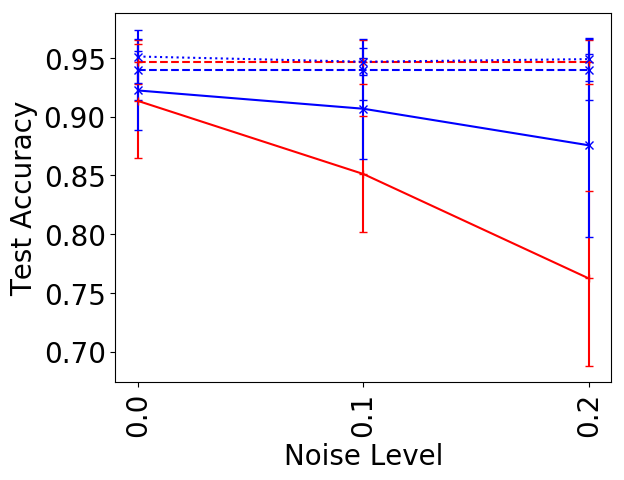}}
    
    \caption{Results on the Iris dataset. In each row we consider a different metric to generate the triplets. In each column we consider a different proportion of triplets available from $1$ to $10\%$. In each plot we vary the noise level from $0$ to $20\%$.\label{app:fig:iristriplets}}
\end{figure}


\begin{figure}[H]
    \centering
    \subfloat{\rotatebox{90}{\hspace{5em}Euclidean}\hspace{2em}}  \subfloat{\includegraphics[scale=0.32]{./images/moons/{baseMetric_euclidean_propNoise_0.0_propTriplets_accuracyTest_1NN_CompTree_SAMME_TripletBoost_tSTE}.png}}
    \subfloat{\includegraphics[scale=0.32]{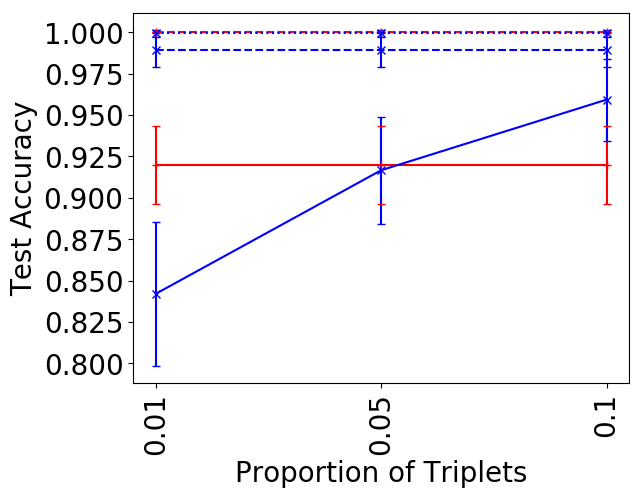}}
    \subfloat{\includegraphics[scale=0.32]{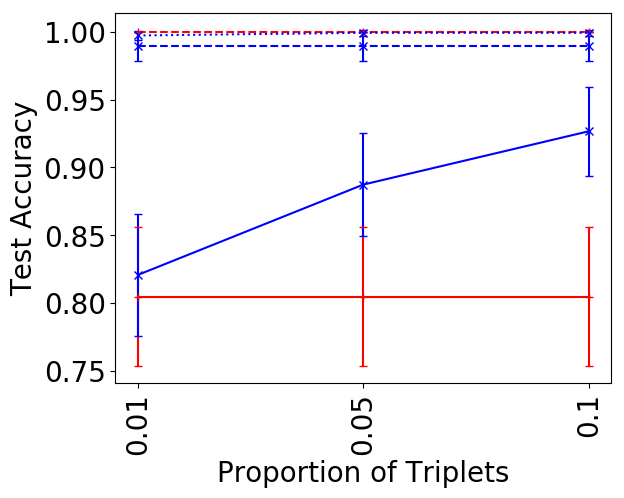}}
    
    \subfloat{\rotatebox{90}{\hspace{5em}Cosine}\hspace{2em}} 
    \subfloat{\includegraphics[scale=0.32]{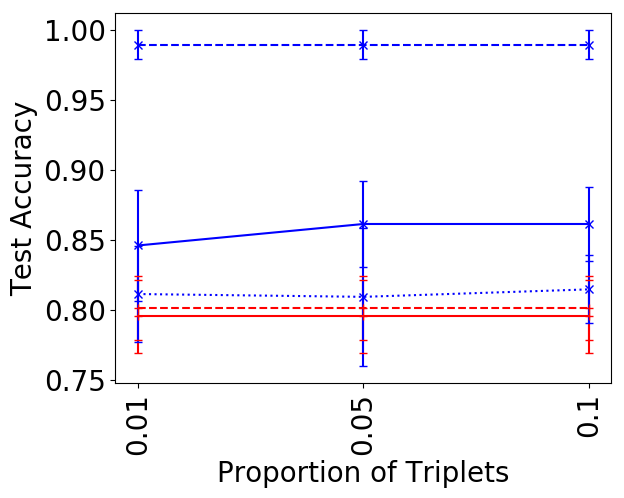}}
    \subfloat{\includegraphics[scale=0.32]{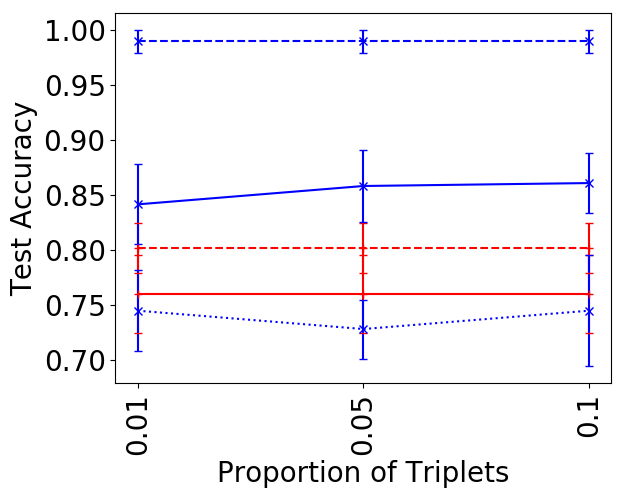}}
    \subfloat{\includegraphics[scale=0.32]{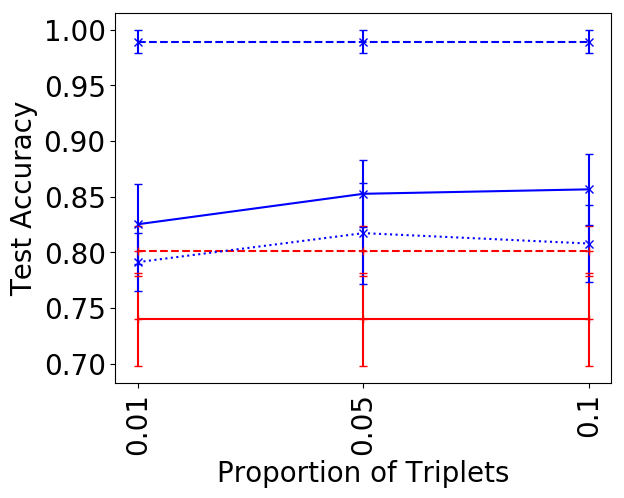}}
    
    \subfloat{\rotatebox{90}{\hspace{5em}Cityblock}\hspace{2em}} 
    \subfloat[Noise Level: $0\%$]{\includegraphics[scale=0.32]{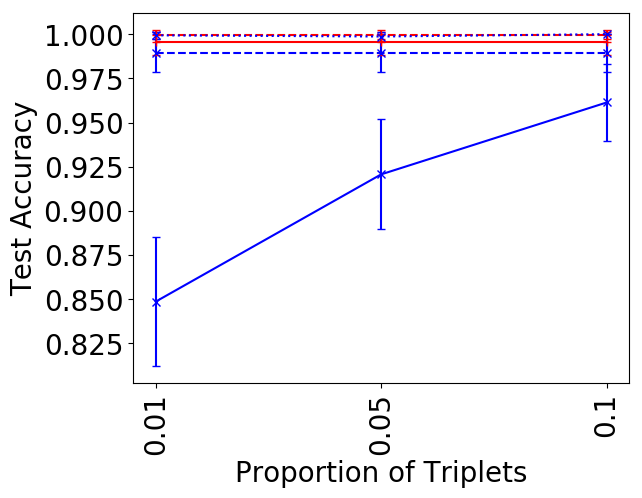}}
    \subfloat[Noise Level: $10\%$]{\includegraphics[scale=0.32]{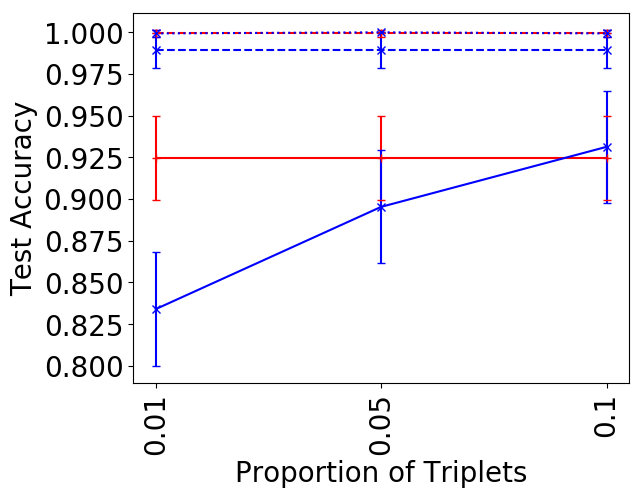}}
    \subfloat[Noise Level: $20\%$]{\includegraphics[scale=0.32]{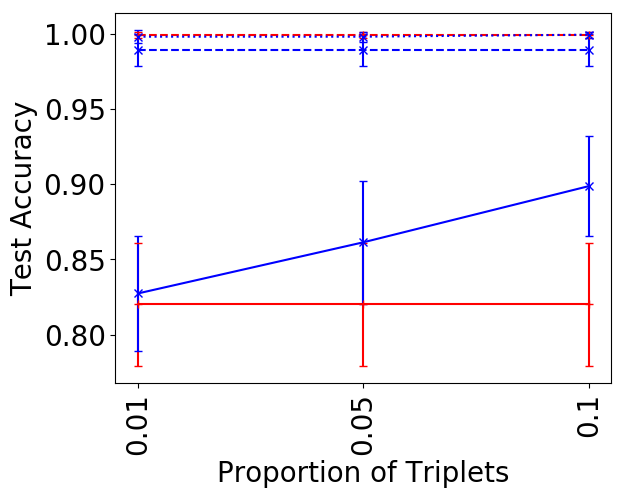}}
    
    \caption{Results on the Moons dataset. In each row we consider a different metric to generate the triplets. In each column we consider a different noise level from $0$ to $20\%$. In each plot we vary the proportion of triplets available from $1$ to $10\%$.\label{app:fig:moonsnoise}}
\end{figure}

\begin{figure}[H]
    \centering
    \subfloat{\rotatebox{90}{\hspace{5em}Euclidean}\hspace{2em}}  \subfloat{\includegraphics[scale=0.32]{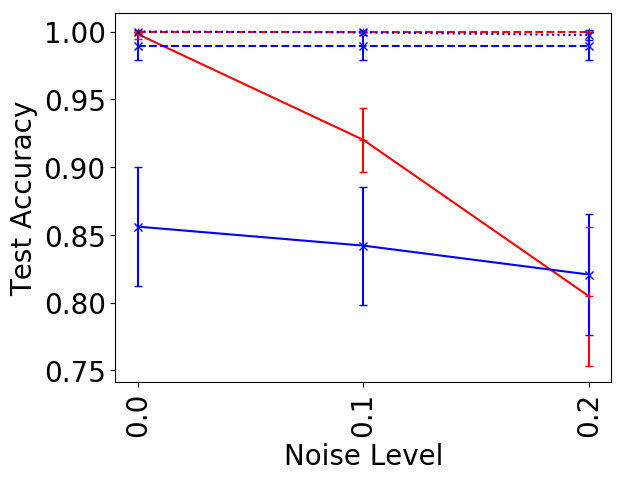}}
    \subfloat{\includegraphics[scale=0.32]{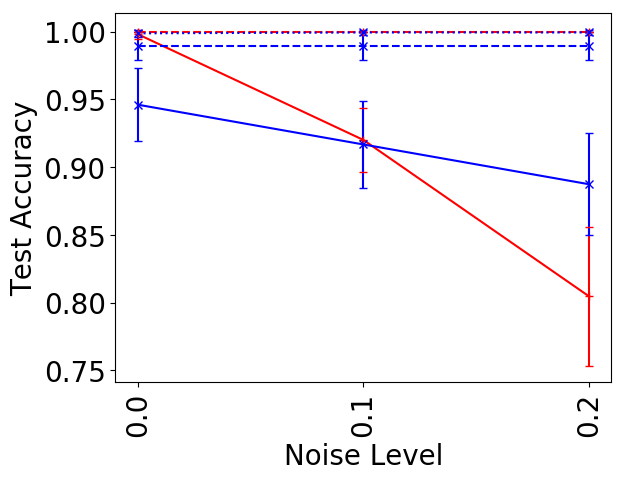}}
    \subfloat{\includegraphics[scale=0.32]{./images/moons/{baseMetric_euclidean_propTriplets_0.1_propNoise_accuracyTest_1NN_CompTree_SAMME_TripletBoost_tSTE}.png}}
    
    \subfloat{\rotatebox{90}{\hspace{5em}Cosine}\hspace{2em}} 
    \subfloat{\includegraphics[scale=0.32]{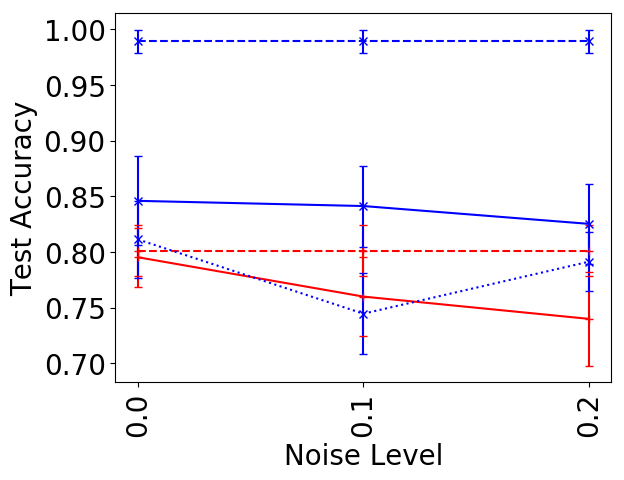}}
    \subfloat{\includegraphics[scale=0.32]{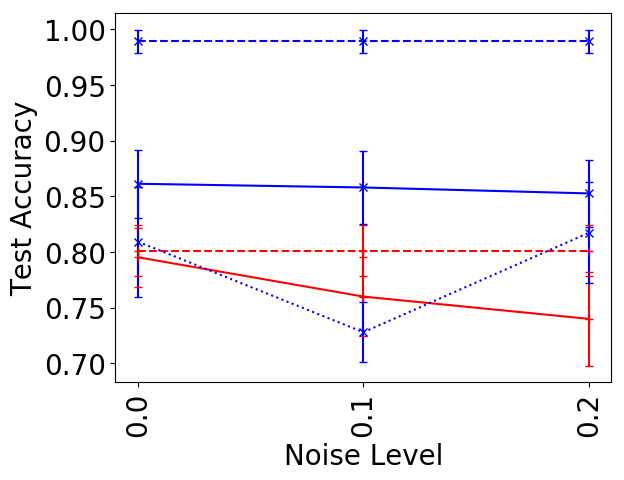}}
    \subfloat{\includegraphics[scale=0.32]{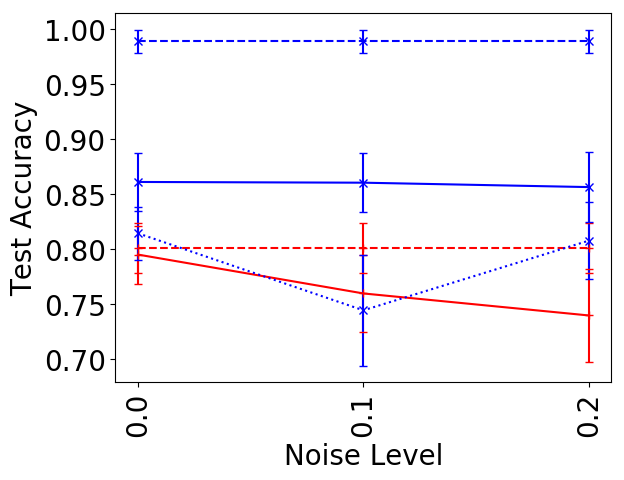}}
    
    \subfloat{\rotatebox{90}{\hspace{5em}Cityblock}\hspace{2em}} 
    \subfloat[Proportion of Triplets: $1\%$]{\includegraphics[scale=0.32]{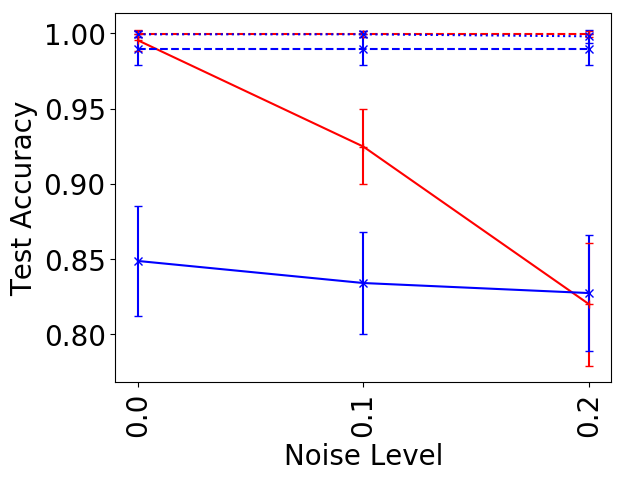}}
    \subfloat[Proportion of Triplets: $5\%$]{\includegraphics[scale=0.32]{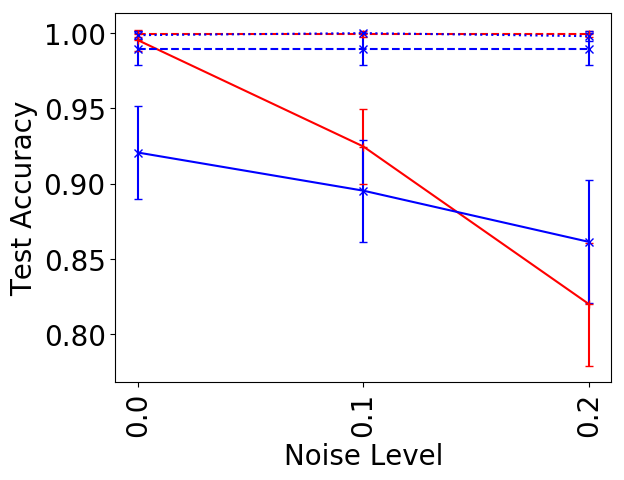}}
    \subfloat[Proportion of Triplets: $10\%$]{\includegraphics[scale=0.32]{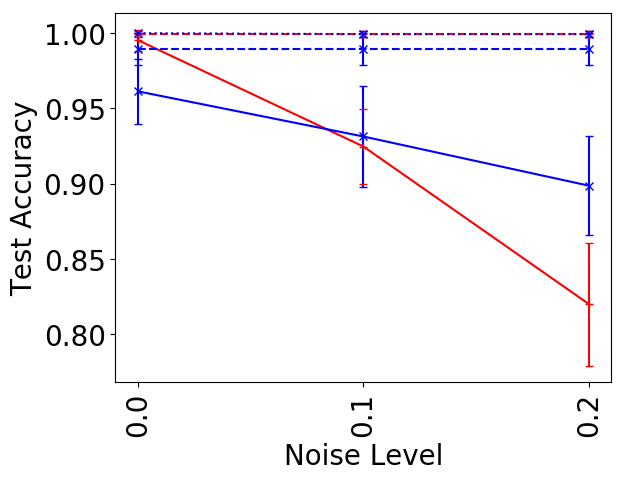}}
    
    \caption{Results on the Moons dataset. In each row we consider a different metric to generate the triplets. In each column we consider a different proportion of triplets available from $1$ to $10\%$. In each plot we vary the noise level from $0$ to $20\%$.\label{app:fig:moonstriplets}}
\end{figure}


\begin{figure}[H]
    \centering
    \subfloat{\rotatebox{90}{\hspace{5em}Euclidean}\hspace{2em}}  \subfloat{\includegraphics[scale=0.32]{./images/gisette/{baseMetric_euclidean_propNoise_0.0_propTriplets_accuracyTest_1NN_CompTree_SAMME_TripletBoost}.png}}
    \subfloat{\includegraphics[scale=0.32]{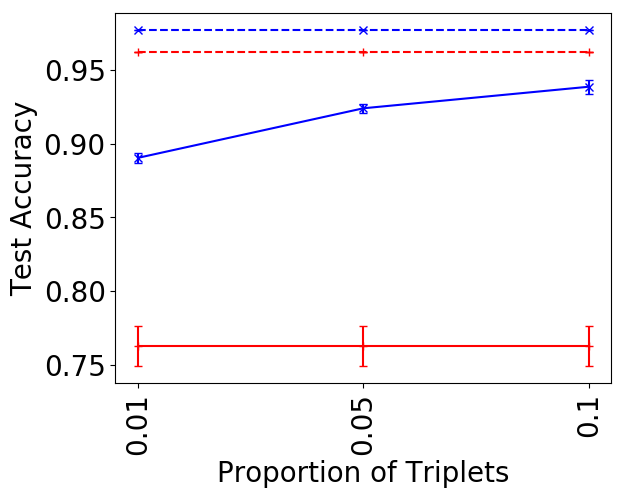}}
    \subfloat{\includegraphics[scale=0.32]{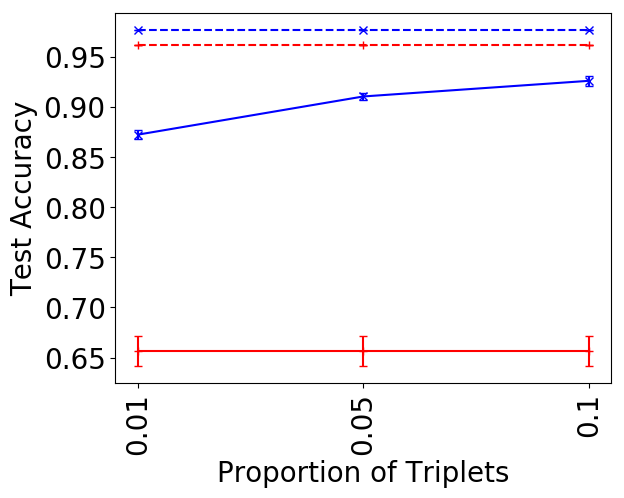}}
    
    \subfloat{\rotatebox{90}{\hspace{5em}Cosine}\hspace{2em}} 
    \subfloat{\includegraphics[scale=0.32]{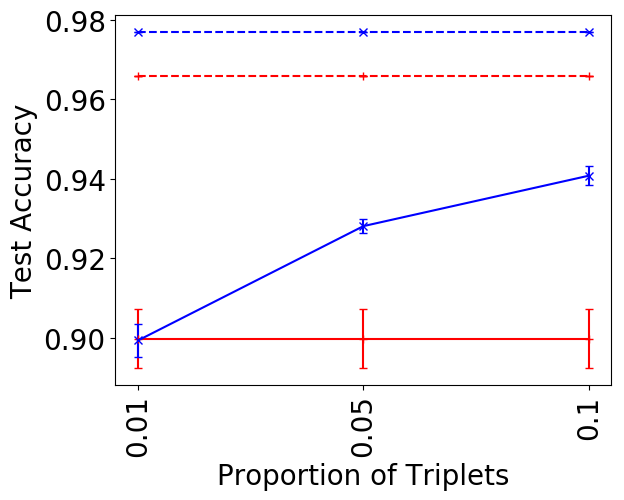}}
    \subfloat{\includegraphics[scale=0.32]{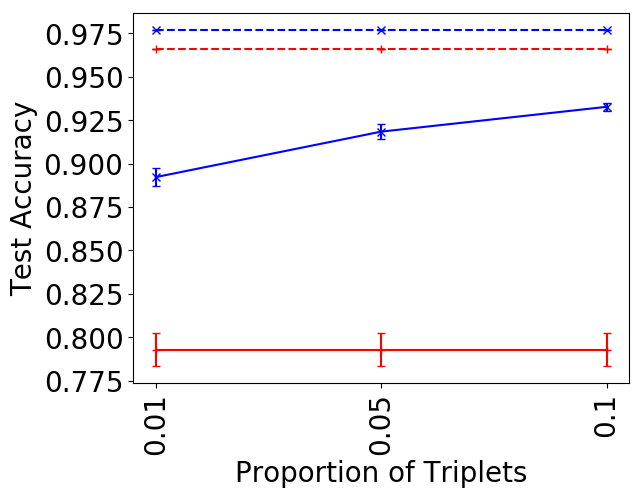}}
    \subfloat{\includegraphics[scale=0.32]{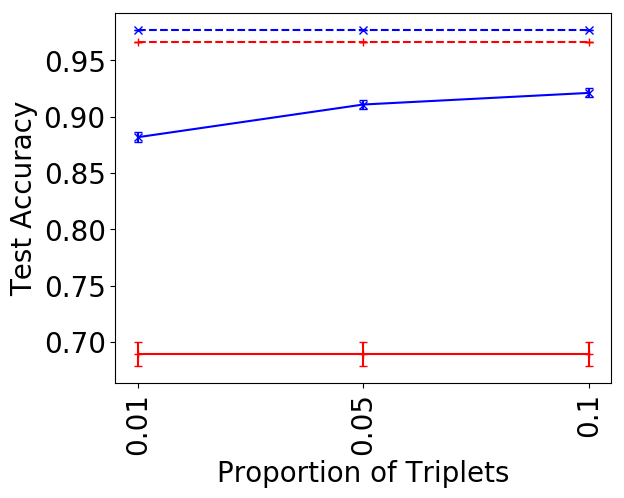}}
    
    \subfloat{\rotatebox{90}{\hspace{5em}Cityblock}\hspace{2em}} 
    \subfloat[Noise Level: $0\%$]{\includegraphics[scale=0.32]{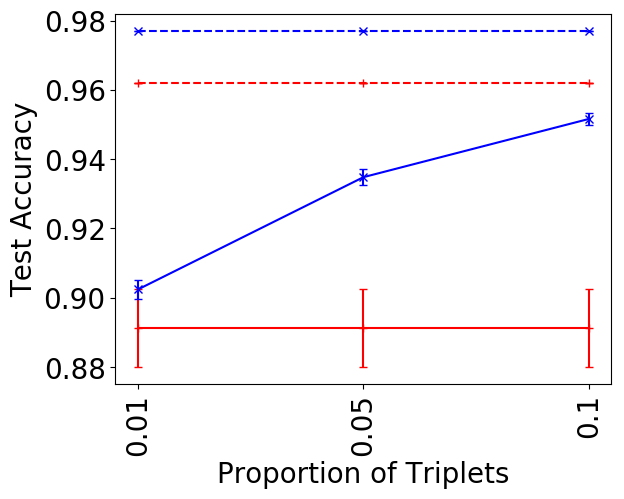}}
    \subfloat[Noise Level: $10\%$]{\includegraphics[scale=0.32]{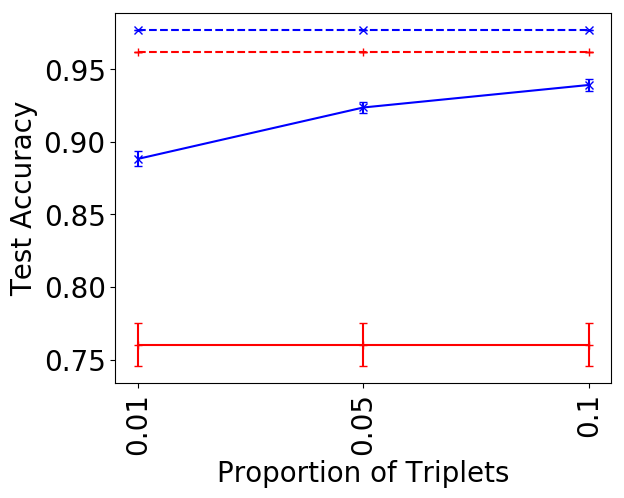}}
    \subfloat[Noise Level: $20\%$]{\includegraphics[scale=0.32]{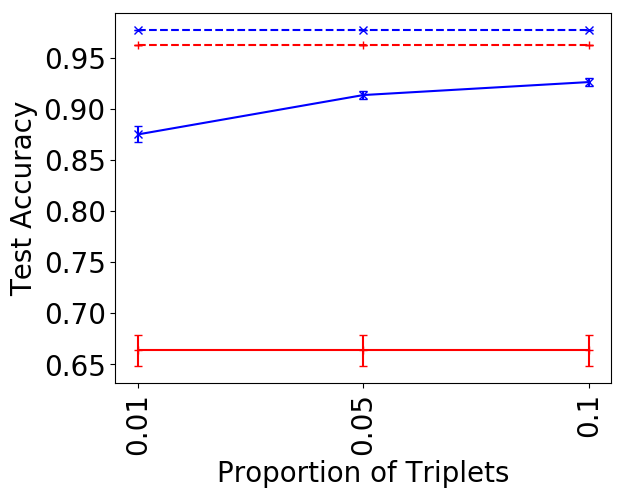}}
    
    \caption{Results on the Gisette dataset. In each row we consider a different metric to generate the triplets. In each column we consider a different noise level from $0$ to $20\%$. In each plot we vary the proportion of triplets available from $1$ to $10\%$.\label{app:fig:gisettenoise}}
\end{figure}

\begin{figure}[H]
    \centering
    \subfloat{\rotatebox{90}{\hspace{5em}Euclidean}\hspace{2em}}  \subfloat{\includegraphics[scale=0.32]{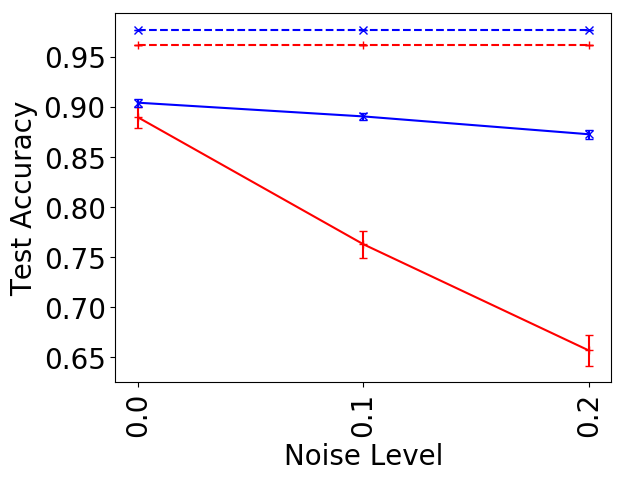}}
    \subfloat{\includegraphics[scale=0.32]{./images/gisette/{baseMetric_euclidean_propTriplets_0.05_propNoise_accuracyTest_1NN_CompTree_SAMME_TripletBoost}.png}}
    \subfloat{\includegraphics[scale=0.32]{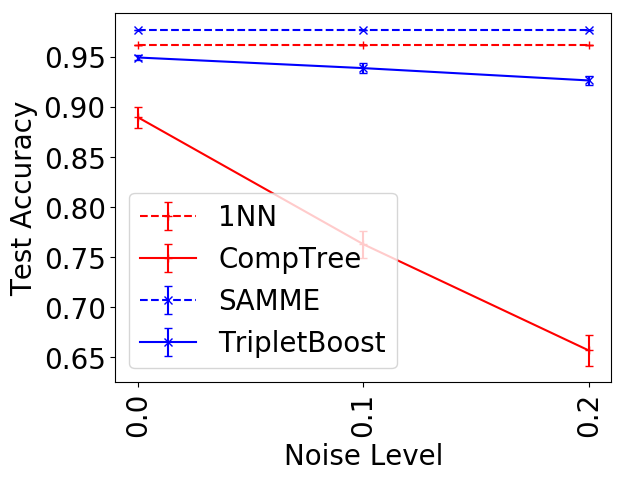}}
    
    \subfloat{\rotatebox{90}{\hspace{5em}Cosine}\hspace{2em}} 
    \subfloat{\includegraphics[scale=0.32]{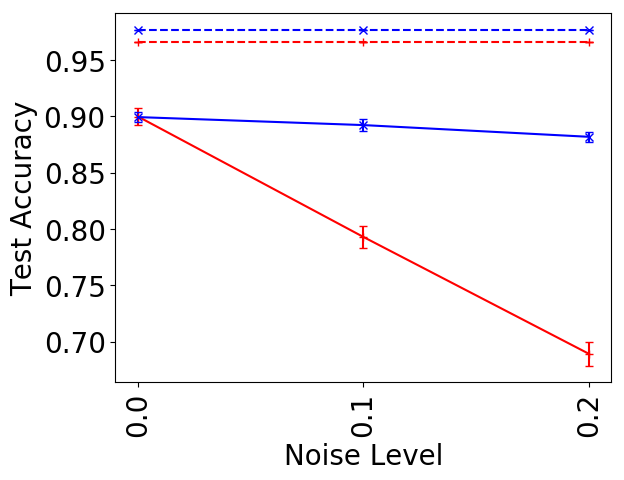}}
    \subfloat{\includegraphics[scale=0.32]{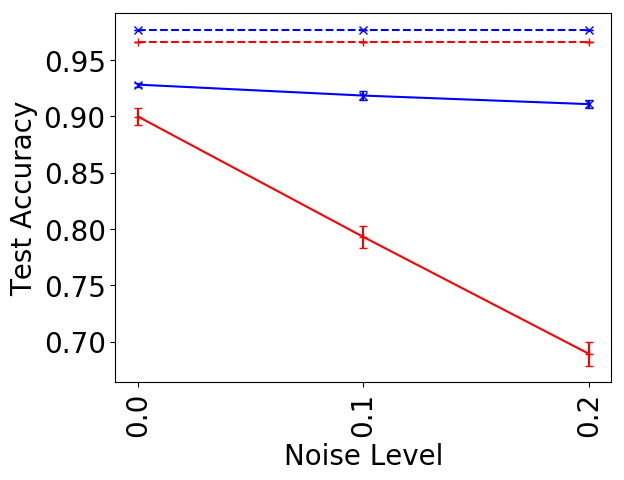}}
    \subfloat{\includegraphics[scale=0.32]{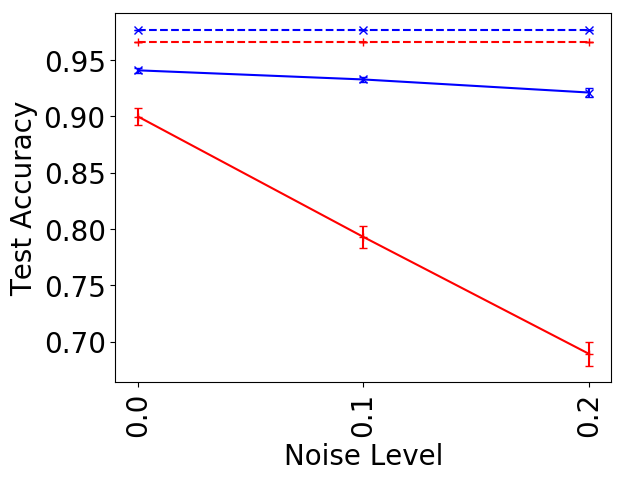}}
    
    \subfloat{\rotatebox{90}{\hspace{5em}Cityblock}\hspace{2em}} 
    \subfloat[Proportion of Triplets: $1\%$]{\includegraphics[scale=0.32]{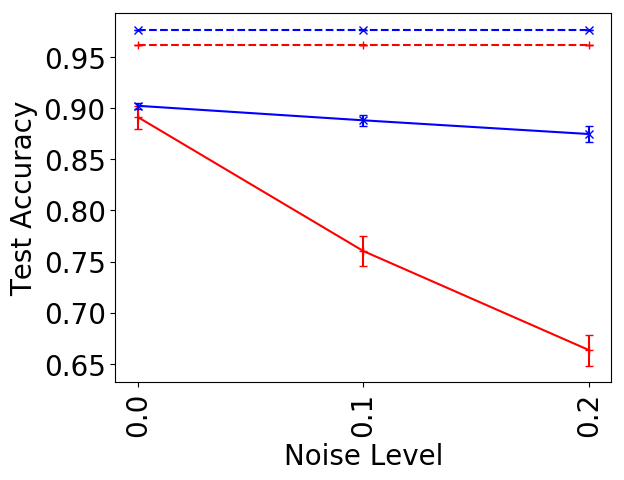}}
    \subfloat[Proportion of Triplets: $5\%$]{\includegraphics[scale=0.32]{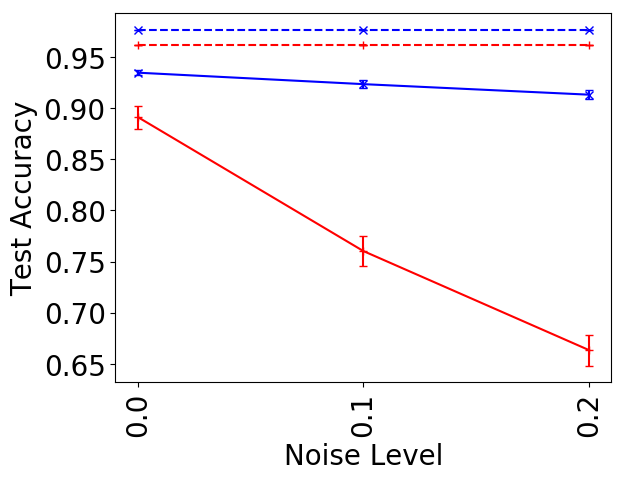}}
    \subfloat[Proportion of Triplets: $10\%$]{\includegraphics[scale=0.32]{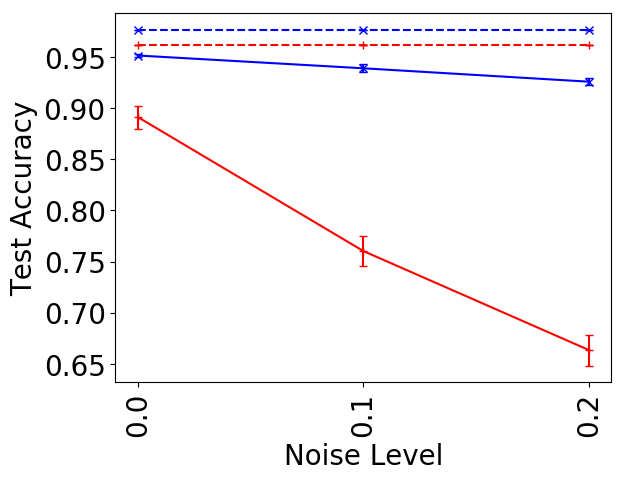}}
    
    \caption{Results on the Gisette dataset. In each row we consider a different metric to generate the triplets. In each column we consider a different proportion of triplets available from $1$ to $10\%$. In each plot we vary the noise level from $0$ to $20\%$.\label{app:fig:gisettetriplets}}
\end{figure}

\begin{figure}[H]
    \centering
    \subfloat{\rotatebox{90}{\hspace{5em}Euclidean}\hspace{2em}}  \subfloat{\includegraphics[scale=0.32]{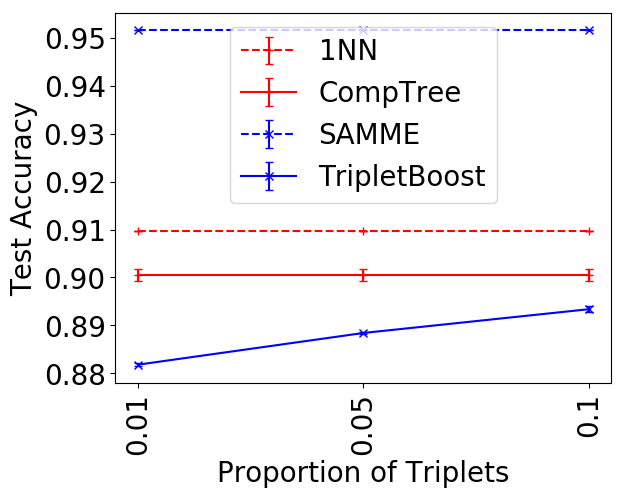}}
    \subfloat{\includegraphics[scale=0.32]{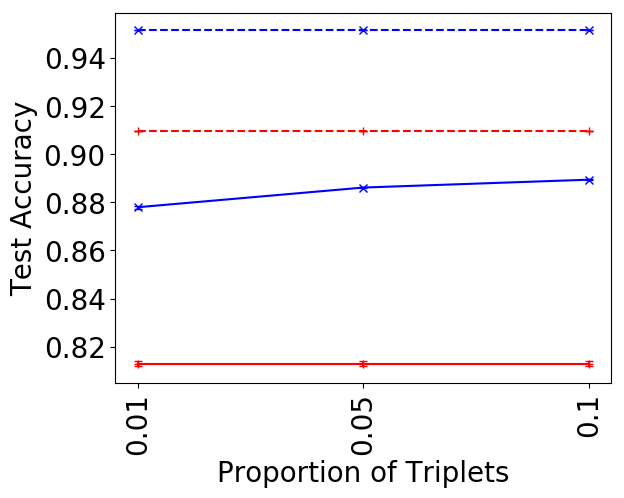}}
    \subfloat{\includegraphics[scale=0.32]{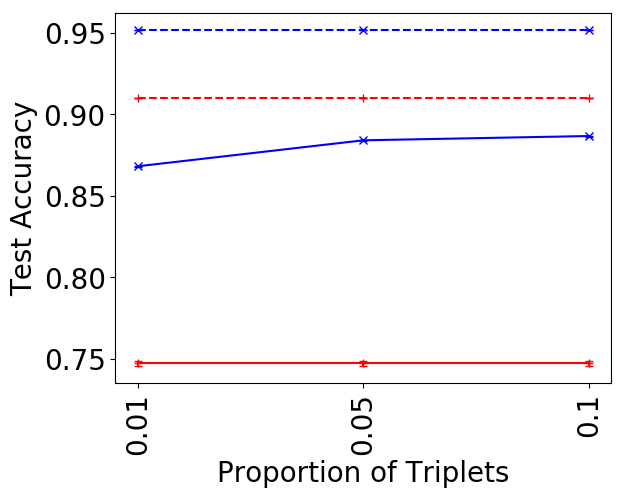}}
    
    \subfloat{\rotatebox{90}{\hspace{5em}Cosine}\hspace{2em}} 
    \subfloat{\includegraphics[scale=0.32]{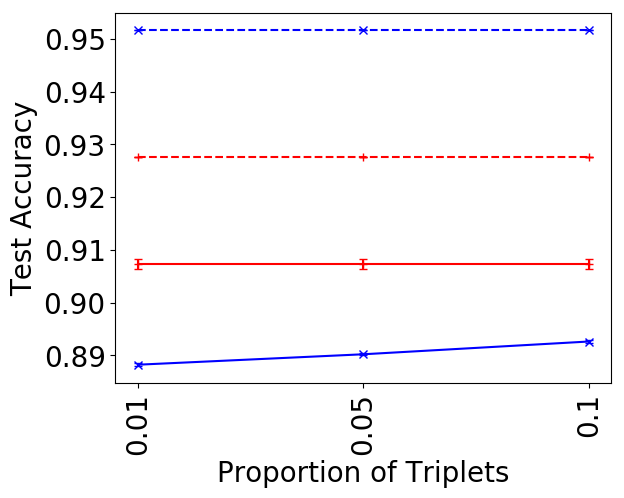}}
    \subfloat{\includegraphics[scale=0.32]{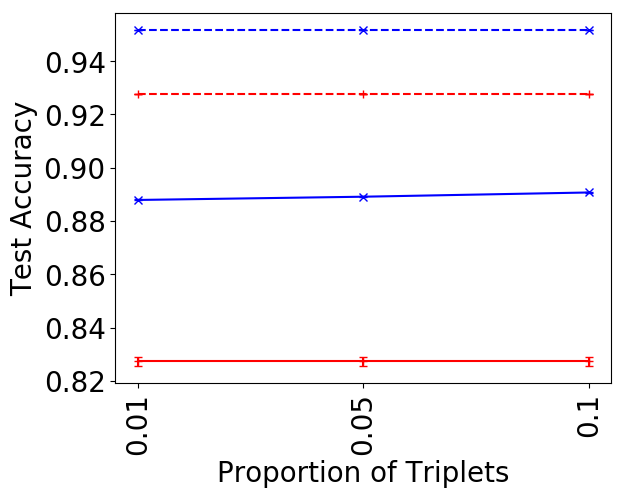}}
    \subfloat{\includegraphics[scale=0.32]{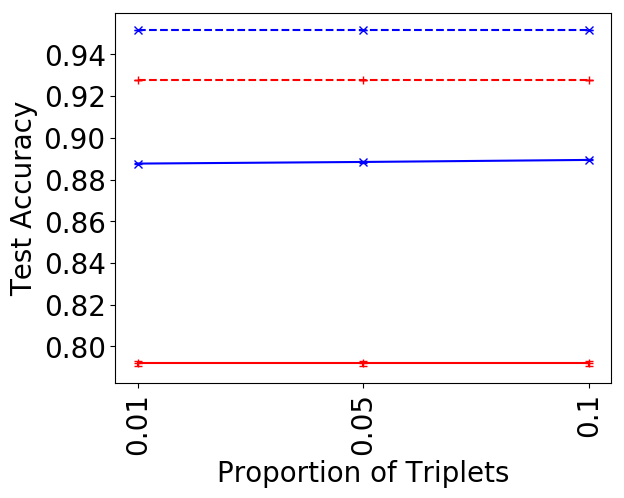}}
    
    \subfloat{\rotatebox{90}{\hspace{5em}Cityblock}\hspace{2em}} 
    \subfloat[Noise Level: $0\%$]{\includegraphics[scale=0.32]{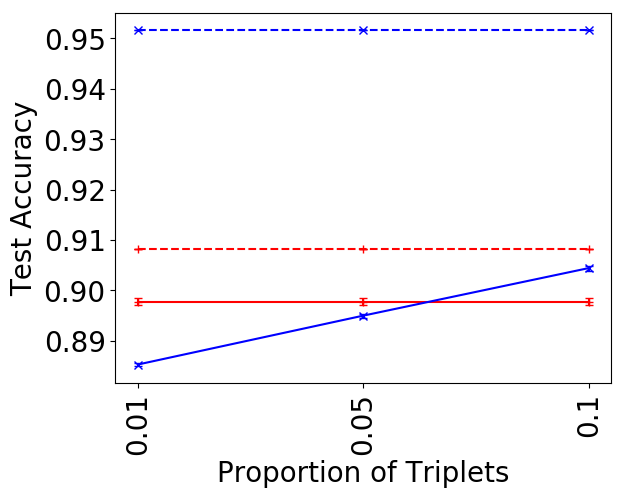}}
    \subfloat[Noise Level: $10\%$]{\includegraphics[scale=0.32]{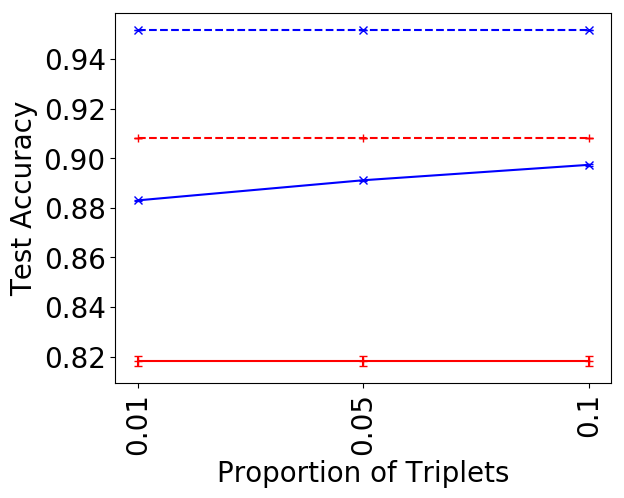}}
    \subfloat[Noise Level: $20\%$]{\includegraphics[scale=0.32]{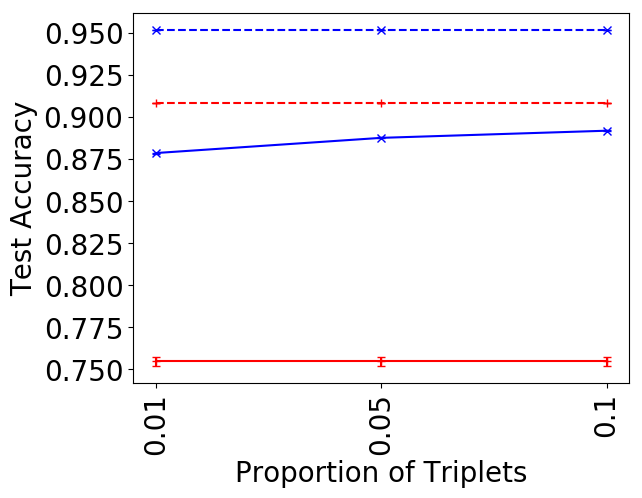}}
    
    \caption{Results on the Cod-rna dataset. In each row we consider a different metric to generate the triplets. In each column we consider a different noise level from $0$ to $20\%$. In each plot we vary the proportion of triplets available from $1$ to $10\%$.\label{app:fig:codrnanoise}}
\end{figure}

\begin{figure}[H]
    \centering
    \subfloat{\rotatebox{90}{\hspace{5em}Euclidean}\hspace{2em}}  \subfloat{\includegraphics[scale=0.32]{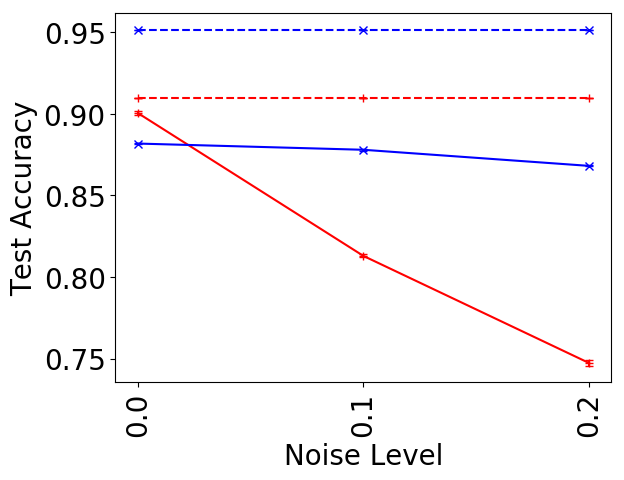}}
    \subfloat{\includegraphics[scale=0.32]{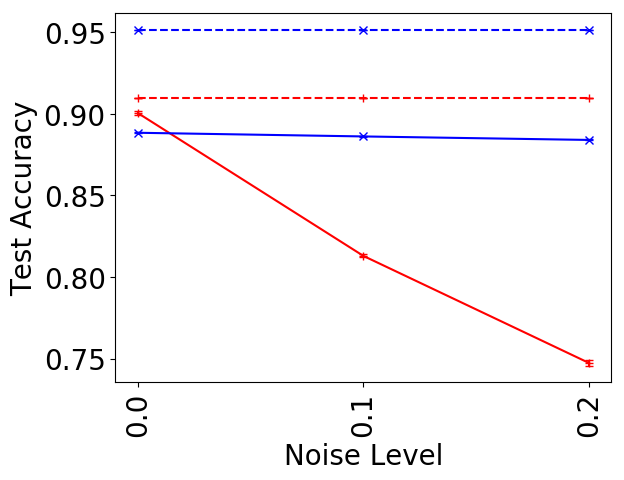}}
    \subfloat{\includegraphics[scale=0.32]{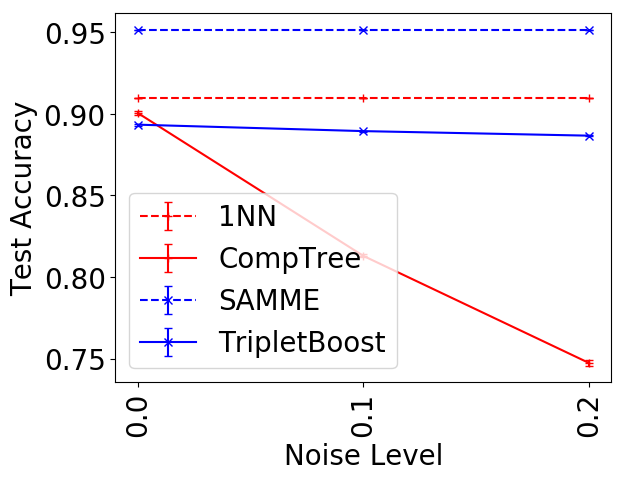}}
    
    \subfloat{\rotatebox{90}{\hspace{5em}Cosine}\hspace{2em}} 
    \subfloat{\includegraphics[scale=0.32]{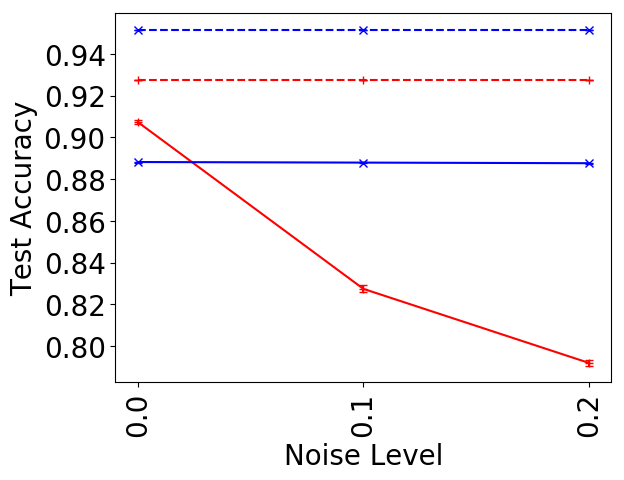}}
    \subfloat{\includegraphics[scale=0.32]{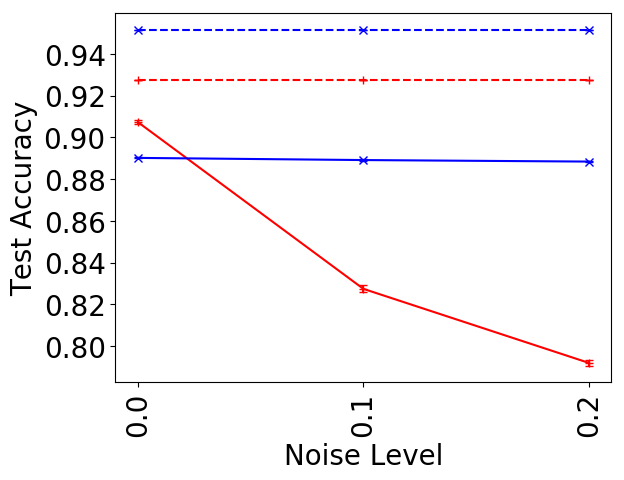}}
    \subfloat{\includegraphics[scale=0.32]{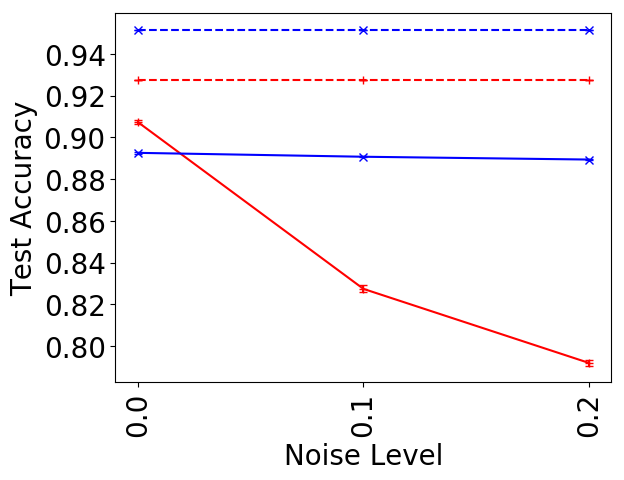}}
    
    \subfloat{\rotatebox{90}{\hspace{5em}Cityblock}\hspace{2em}} 
    \subfloat[Proportion of Triplets: $1\%$]{\includegraphics[scale=0.32]{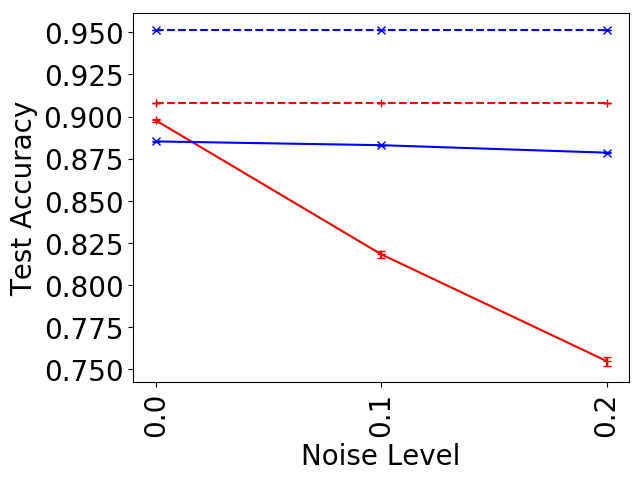}}
    \subfloat[Proportion of Triplets: $5\%$]{\includegraphics[scale=0.32]{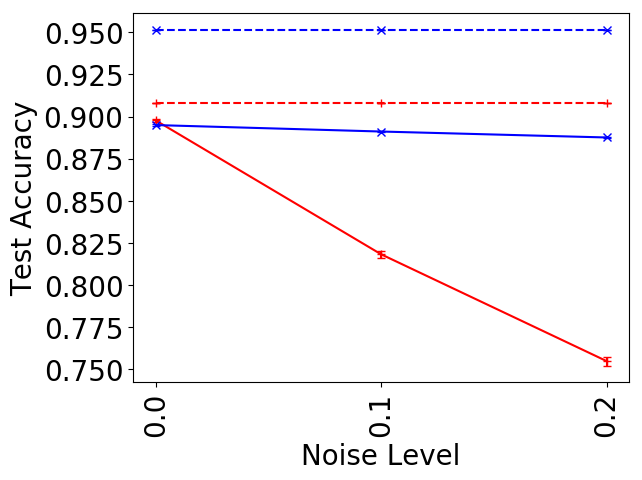}}
    \subfloat[Proportion of Triplets: $10\%$]{\includegraphics[scale=0.32]{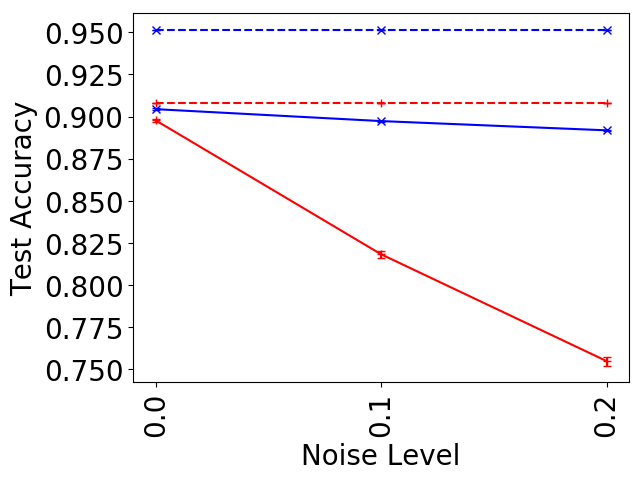}}
    
    \caption{Results on the Cod-rna dataset. In each row we consider a different metric to generate the triplets. In each column we consider a different proportion of triplets available from $1$ to $10\%$. In each plot we vary the noise level from $0$ to $20\%$.\label{app:fig:codrnatriplets}}
\end{figure}


\begin{figure}[H]
    \centering
    \subfloat{\rotatebox{90}{\hspace{5em}Euclidean}\hspace{2em}}  \subfloat{\includegraphics[scale=0.32]{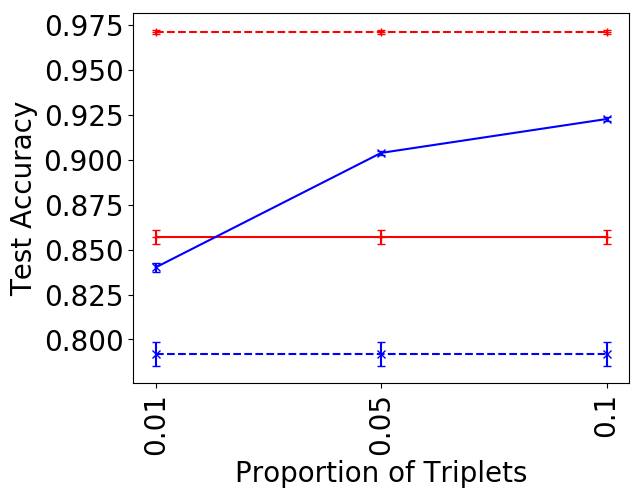}}
    \subfloat{\includegraphics[scale=0.32]{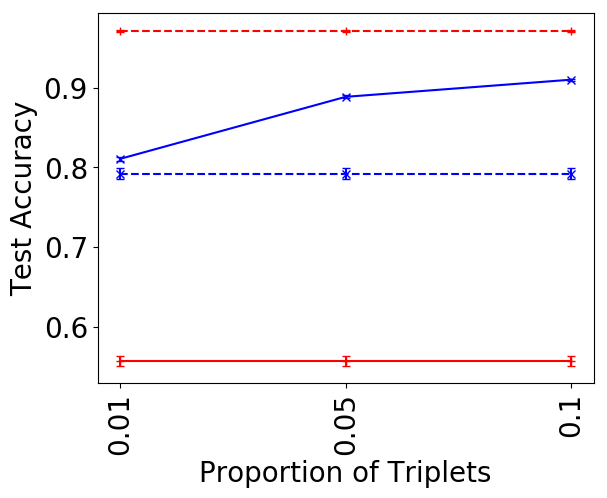}}
    \subfloat{\includegraphics[scale=0.32]{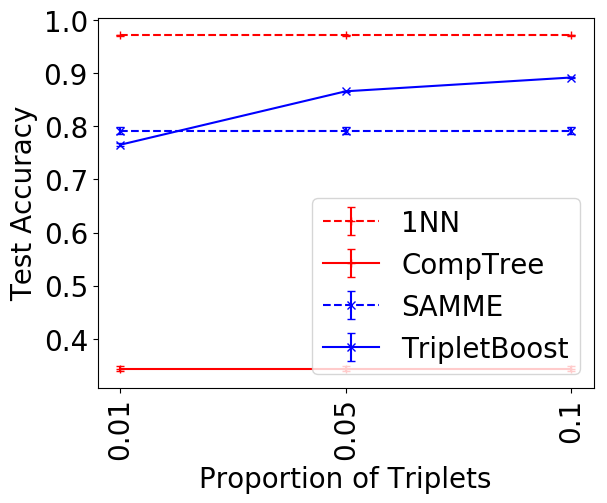}}
    
    \subfloat{\rotatebox{90}{\hspace{5em}Cosine}\hspace{2em}} 
    \subfloat{\includegraphics[scale=0.32]{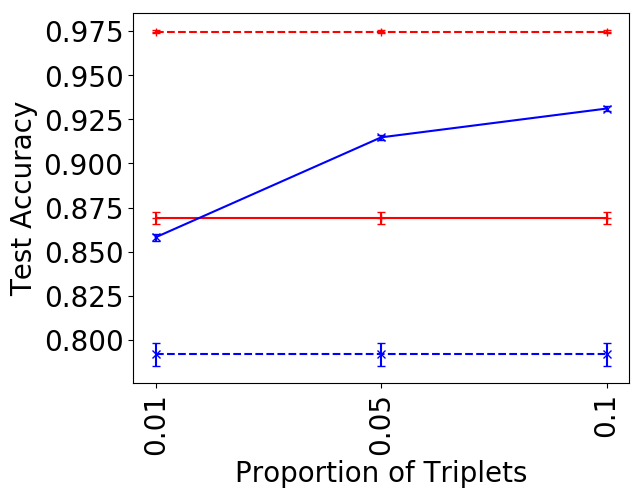}}
    \subfloat{\includegraphics[scale=0.32]{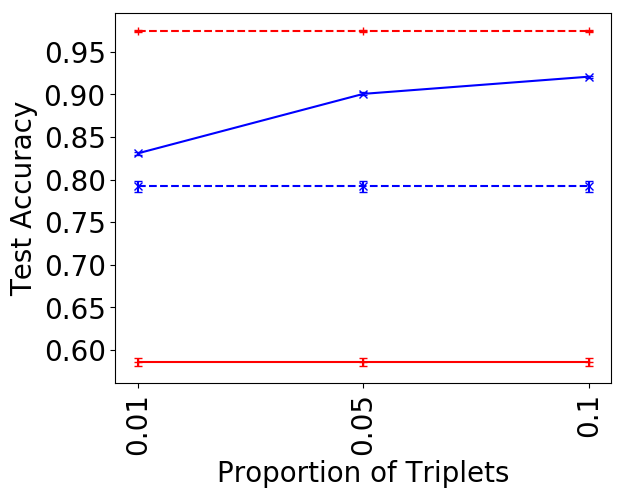}}
    \subfloat{\includegraphics[scale=0.32]{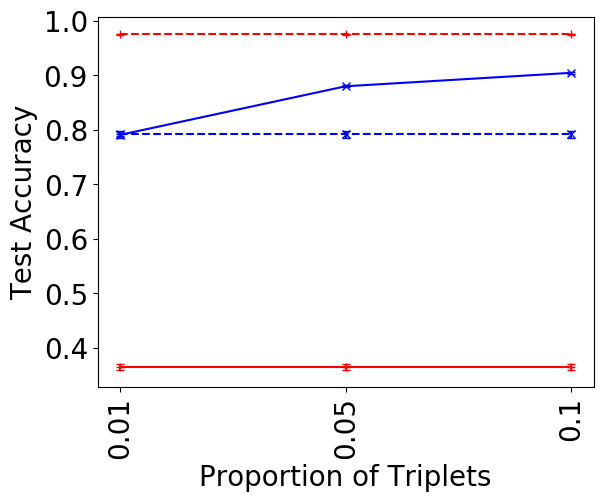}}
    
    \subfloat{\rotatebox{90}{\hspace{5em}Cityblock}\hspace{2em}} 
    \subfloat[Noise Level: $0\%$]{\includegraphics[scale=0.32]{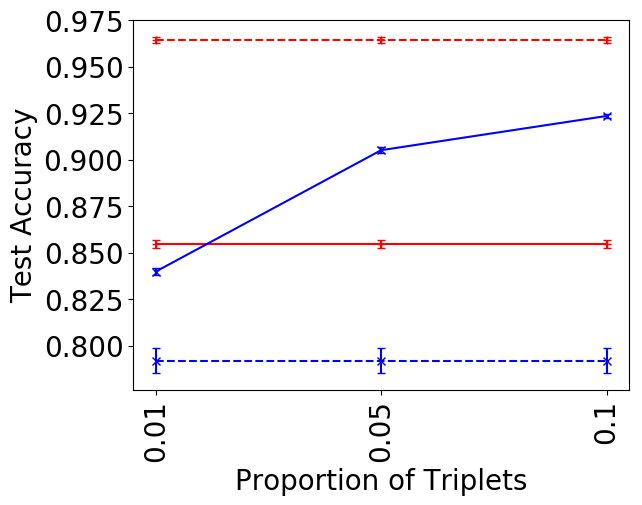}}
    \subfloat[Noise Level: $10\%$]{\includegraphics[scale=0.32]{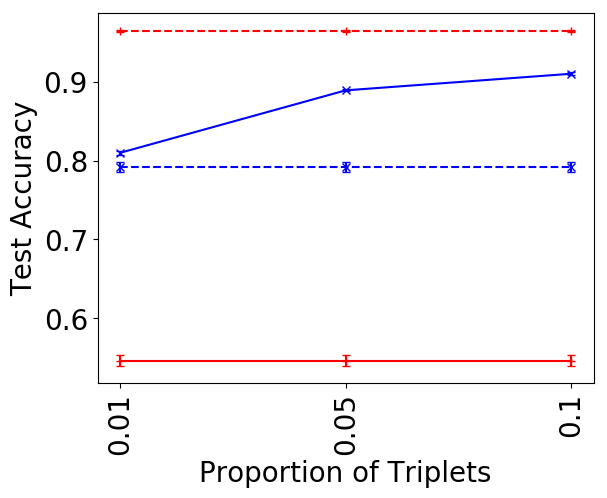}}
    \subfloat[Noise Level: $20\%$]{\includegraphics[scale=0.32]{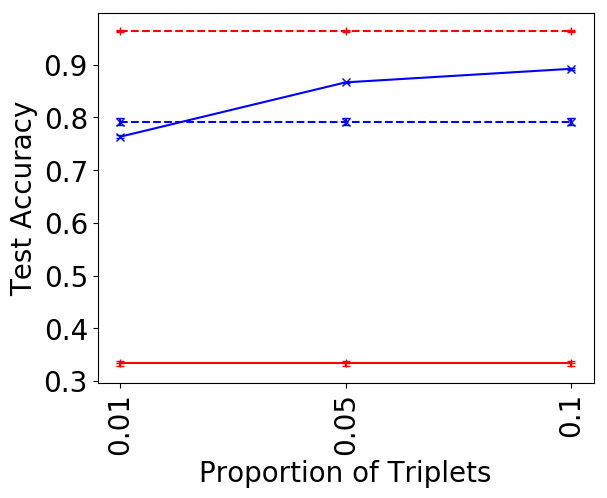}}
    
    \caption{Results on the MNIST dataset. In each row we consider a different metric to generate the triplets. In each column we consider a different noise level from $0$ to $20\%$. In each plot we vary the proportion of triplets available from $1$ to $10\%$.\label{app:fig:mnistnoise}}
\end{figure}

\begin{figure}[H]
    \centering
    \subfloat{\rotatebox{90}{\hspace{5em}Euclidean}\hspace{2em}}  \subfloat{\includegraphics[scale=0.32]{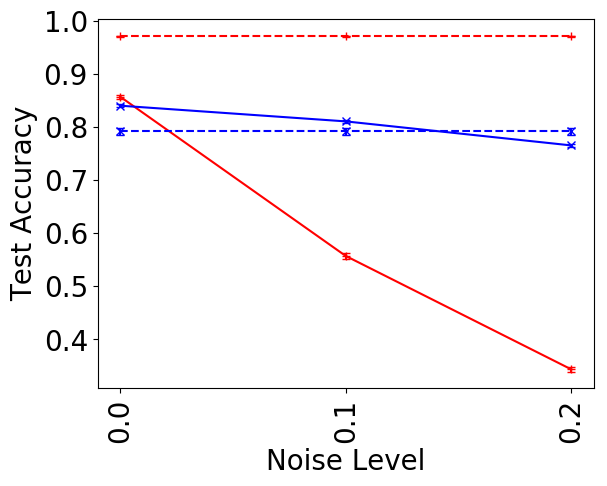}}
    \subfloat{\includegraphics[scale=0.32]{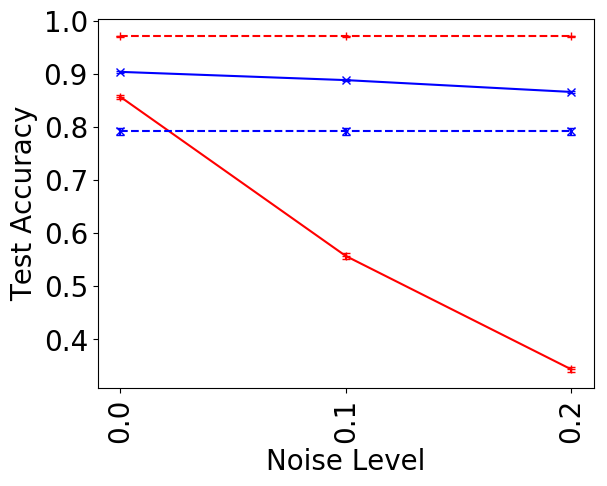}}
    \subfloat{\includegraphics[scale=0.32]{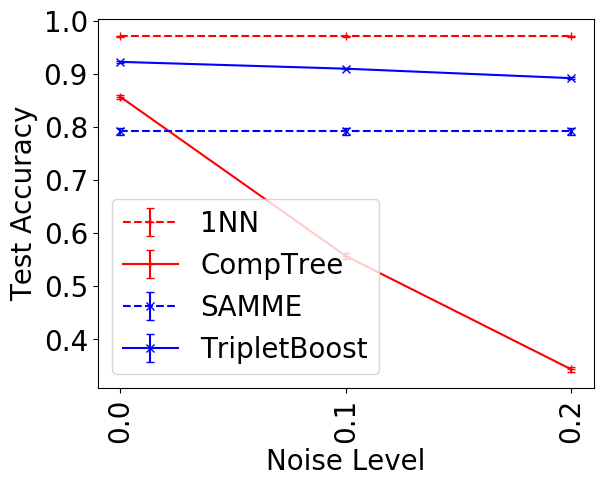}}
    
    \subfloat{\rotatebox{90}{\hspace{5em}Cosine}\hspace{2em}} 
    \subfloat{\includegraphics[scale=0.32]{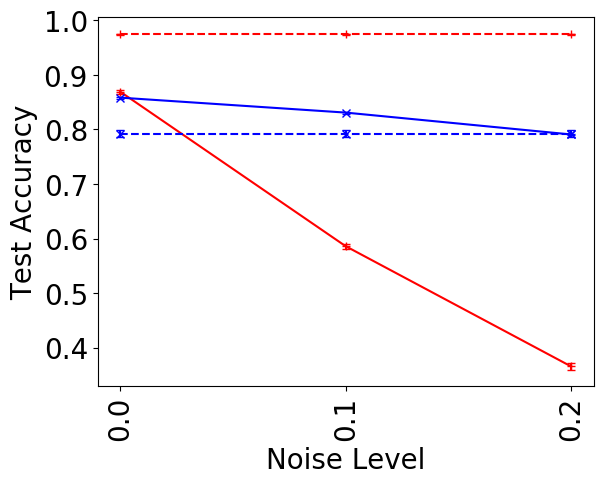}}
    \subfloat{\includegraphics[scale=0.32]{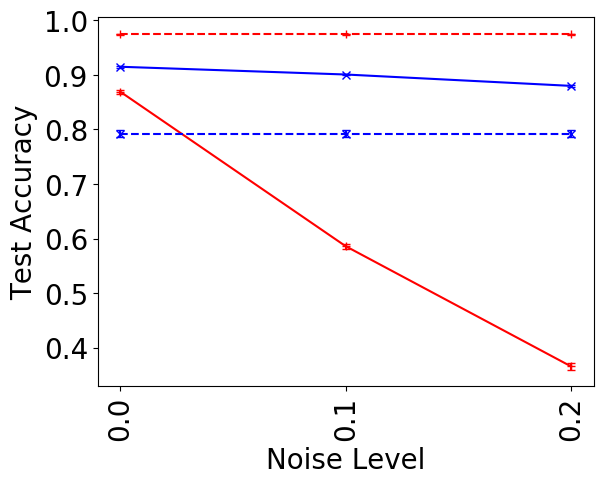}}
    \subfloat{\includegraphics[scale=0.32]{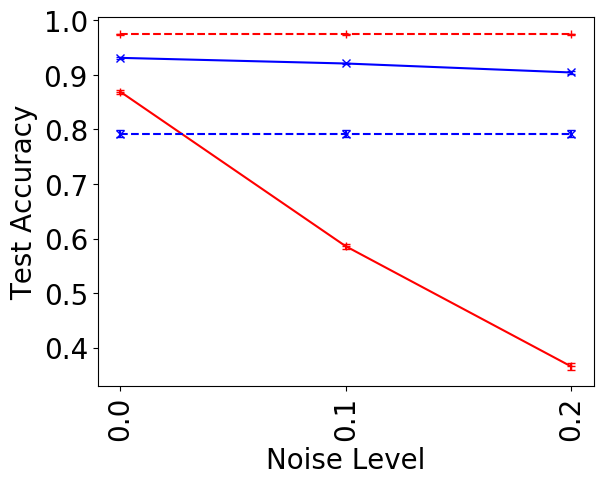}}
    
    \subfloat{\rotatebox{90}{\hspace{5em}Cityblock}\hspace{2em}} 
    \subfloat[Proportion of Triplets: $1\%$]{\includegraphics[scale=0.32]{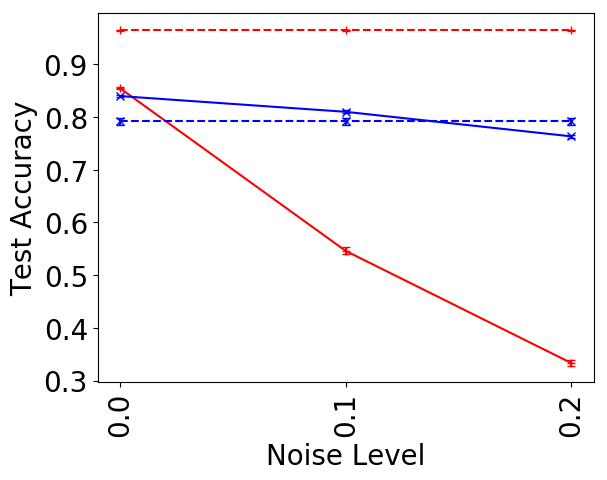}}
    \subfloat[Proportion of Triplets: $5\%$]{\includegraphics[scale=0.32]{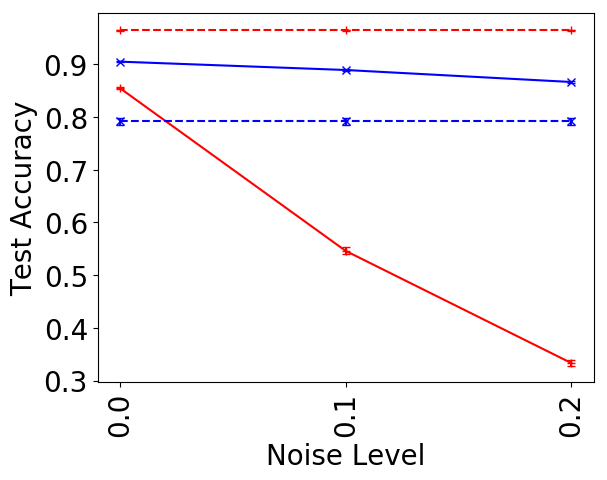}}
    \subfloat[Proportion of Triplets: $10\%$]{\includegraphics[scale=0.32]{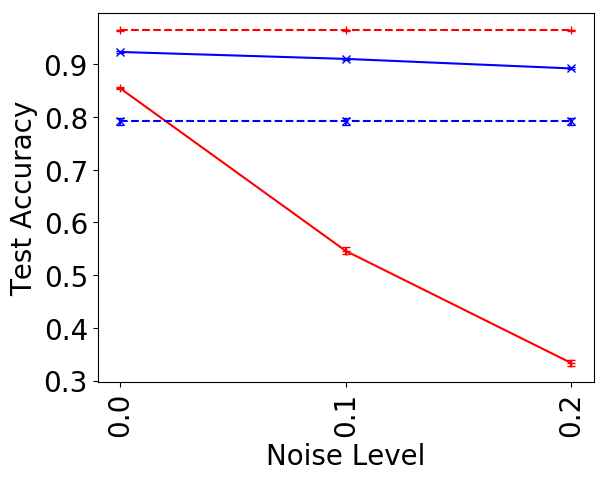}}
    
    \caption{Results on the MNIST dataset. In each row we consider a different metric to generate the triplets. In each column we consider a different proportion of triplets available from $1$ to $10\%$. In each plot we vary the noise level from $0$ to $20\%$.\label{app:fig:mnisttriplets}}
\end{figure}


\begin{figure}[H]
    \centering
    \subfloat{\rotatebox{90}{\hspace{5em}Euclidean}\hspace{2em}}  \subfloat{\includegraphics[scale=0.32]{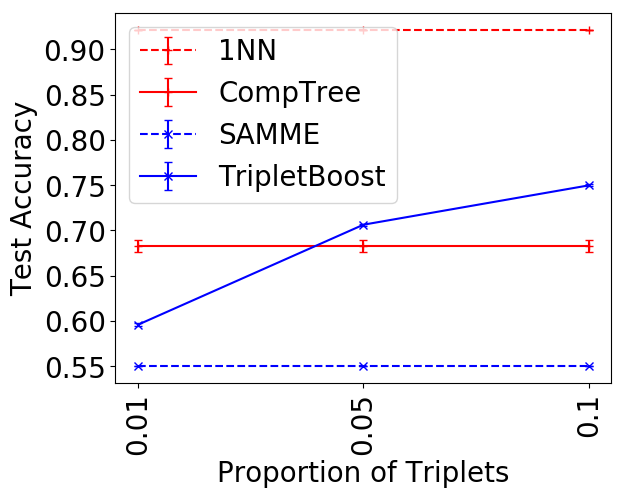}}
    \subfloat{\includegraphics[scale=0.32]{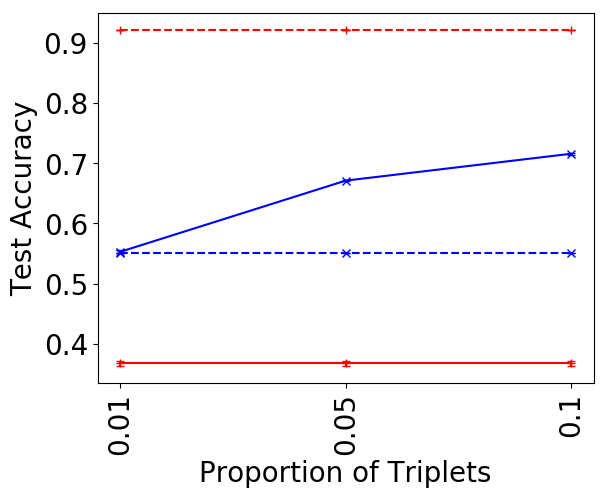}}
    \subfloat{\includegraphics[scale=0.32]{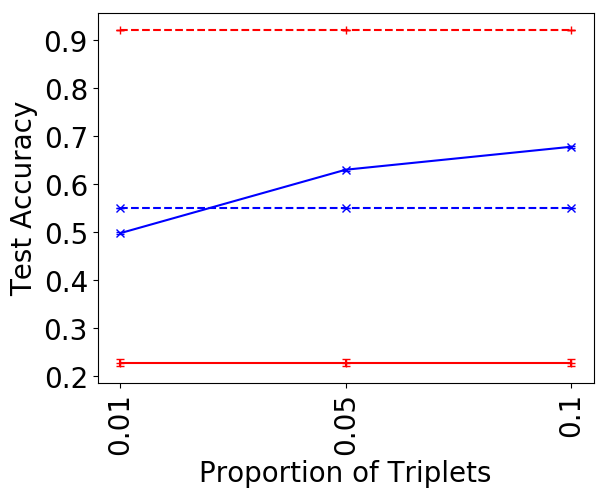}}
    
    \subfloat{\rotatebox{90}{\hspace{5em}Cosine}\hspace{2em}} 
    \subfloat{\includegraphics[scale=0.32]{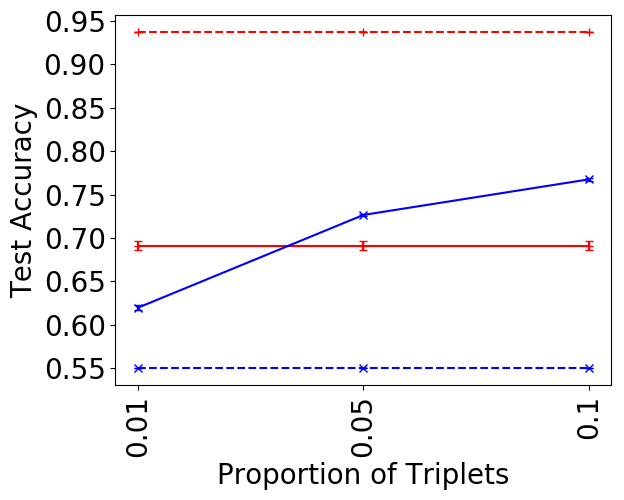}}
    \subfloat{\includegraphics[scale=0.32]{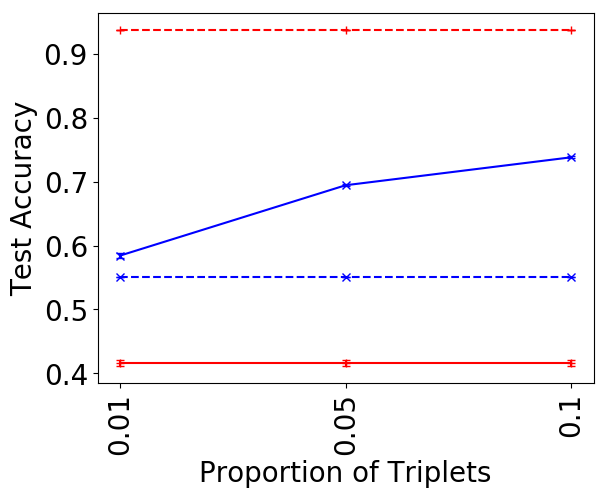}}
    \subfloat{\includegraphics[scale=0.32]{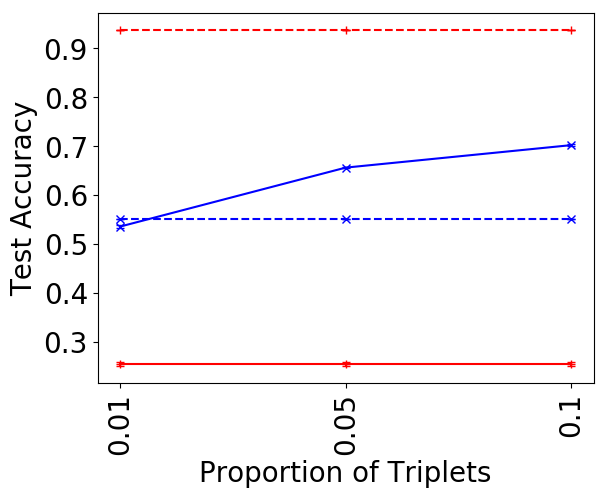}}
    
    \subfloat{\rotatebox{90}{\hspace{5em}Cityblock}\hspace{2em}} 
    \subfloat[Noise Level: $0\%$]{\includegraphics[scale=0.32]{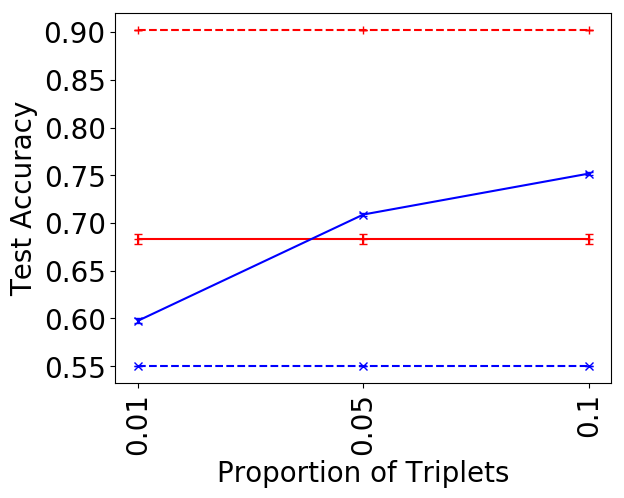}}
    \subfloat[Noise Level: $10\%$]{\includegraphics[scale=0.32]{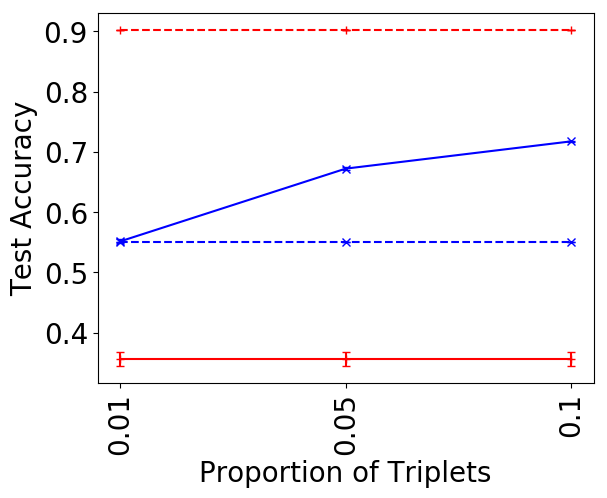}}
    \subfloat[Noise Level: $20\%$]{\includegraphics[scale=0.32]{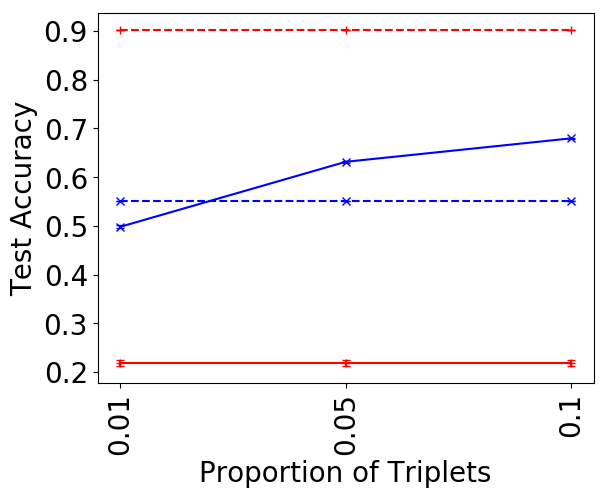}}
    
    \caption{Results on the kMNIST dataset. In each row we consider a different metric to generate the triplets. In each column we consider a different noise level from $0$ to $20\%$. In each plot we vary the proportion of triplets available from $1$ to $10\%$.\label{app:fig:kmnistnoise}}
\end{figure}

\begin{figure}[H]
    \centering
    \subfloat{\rotatebox{90}{\hspace{5em}Euclidean}\hspace{2em}}  \subfloat{\includegraphics[scale=0.32]{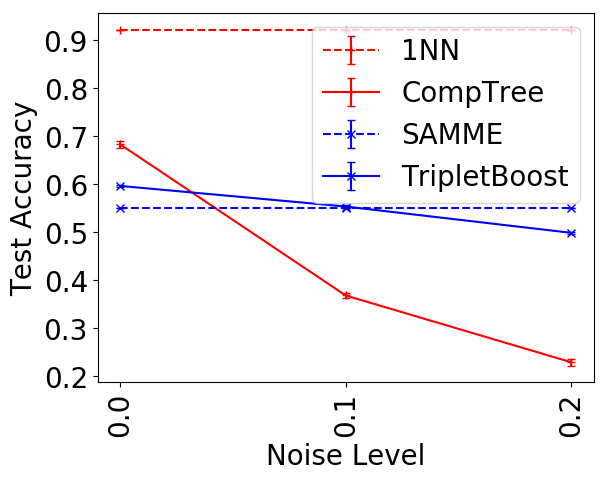}}
    \subfloat{\includegraphics[scale=0.32]{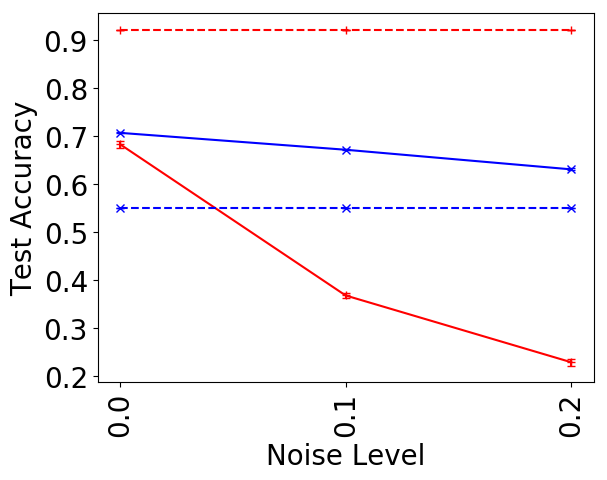}}
    \subfloat{\includegraphics[scale=0.32]{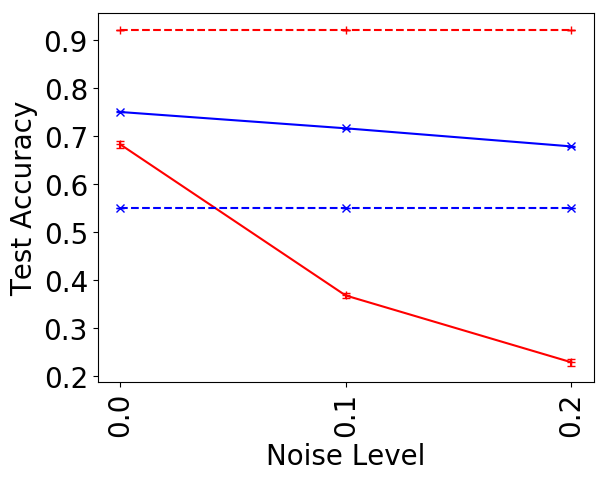}}
    
    \subfloat{\rotatebox{90}{\hspace{5em}Cosine}\hspace{2em}} 
    \subfloat{\includegraphics[scale=0.32]{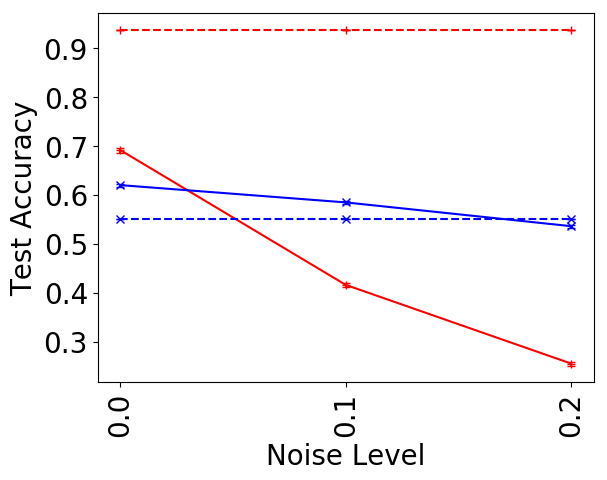}}
    \subfloat{\includegraphics[scale=0.32]{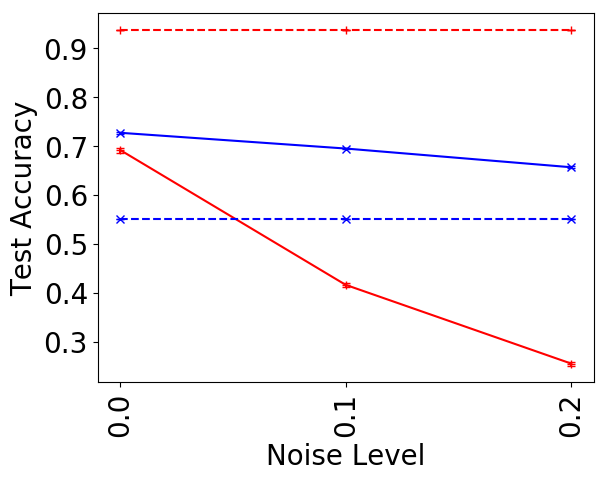}}
    \subfloat{\includegraphics[scale=0.32]{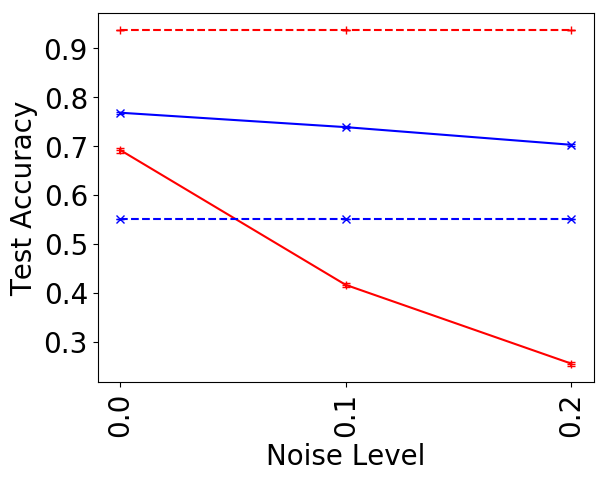}}
    
    \subfloat{\rotatebox{90}{\hspace{5em}Cityblock}\hspace{2em}} 
    \subfloat[Proportion of Triplets: $1\%$]{\includegraphics[scale=0.32]{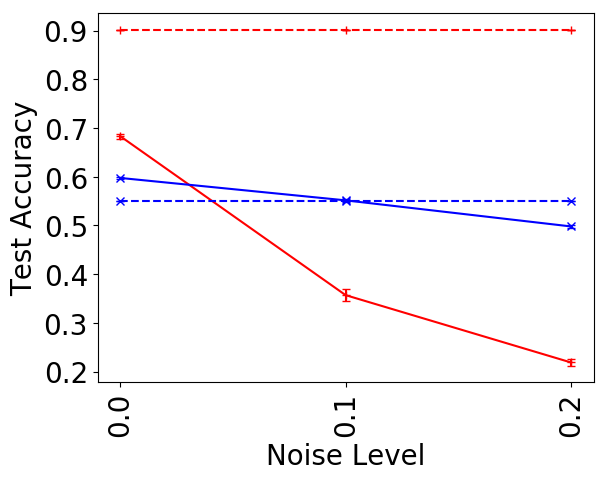}}
    \subfloat[Proportion of Triplets: $5\%$]{\includegraphics[scale=0.32]{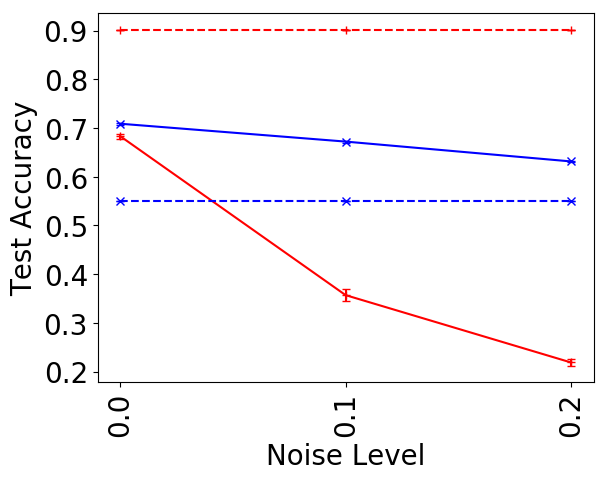}}
    \subfloat[Proportion of Triplets: $10\%$]{\includegraphics[scale=0.32]{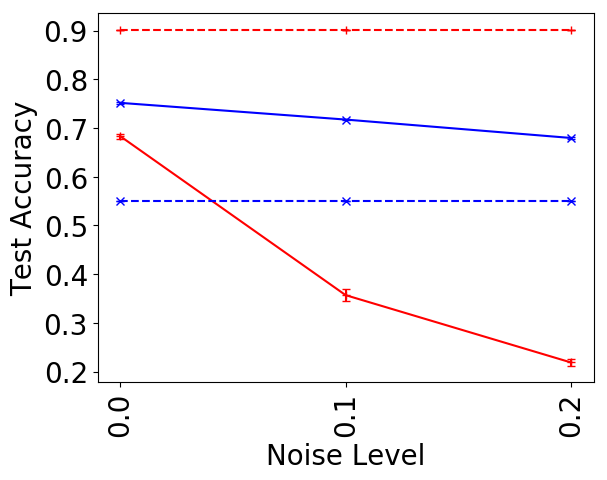}}
    
    \caption{Results on the kMNIST dataset. In each row we consider a different metric to generate the triplets. In each column we consider a different proportion of triplets available from $1$ to $10\%$. In each plot we vary the noise level from $0$ to $20\%$.\label{app:fig:kmnisttriplets}}
\end{figure}

\section{Proof of Theorem~\ref{th:tripletisweak}}
\label{app:sec:weak}
  
\begin{reth}{\ref{th:tripletisweak}}[Triplet classifiers and weak learners]
Let $\edistrW_c$ be an empirical distribution over $\setS \times \spaceY$ and $\fh[_c]{}$ be the corresponding triplet classifier chosen as described in Section~\ref{sec:tripletboost}.
It holds that:
\begin{enumerate}
    \item the error of $\fh[_c]{}$ on $\edistrW_c$ is at most the error of a random predictor and is minimal compared to other labelling strategies,
    \item the weight $\alpha_c$ of the classifier is non-zero if and only if, $\fh[_c]{}$ is a weak classifier, that is is strictly better than a random predictor.
\end{enumerate}
\end{reth}

\begin{proof}
To prove the first claim of the theorem we first study the error of $h_c$ on $\edistrW_c$ and then we show that any labelling strategy different from our would increase the proportion of incorrectly classified examples. To prove the second claim we simply use the definition of the weights $\alpha_c$.

\paragraph{First claim}
First of all notice that the probability that a triplet classifier $h_c$ makes an error on $\edistrW_c$ is
\begin{align*}
\prob_{\substack{(x_i,y) \sim \edistrW_c}}\left[ \left(\indic{y=y_i} - \indic{y \neq y_i}\right) \neq \left(\indic{y \in \fh[_c]{x_i}} - \indic{y \notin \fh[_c]{x_i}}\right) \right]
\end{align*}
that is, by the law of total probabilities:
\begin{align*}
\prob_{\substack{(x_i,y) \sim \edistrW_c \\ \fh[_c]{x_i} \neq \abstain}}\left[ \left(\indic{y=y_i} - \indic{y \neq y_i}\right) \neq \left(\indic{y \in \fh[_c]{x_i}} - \indic{y \notin \fh[_c]{x_i}}\right) \right]& \prob_{\substack{(x_i,y) \sim \edistrW_c}}\left[ \fh[_c]{x_i} \neq \abstain \right] \\
&\negspace{8em}+ \prob_{\substack{(x_i,y) \sim \edistrW_c \\ \fh[_c]{x_i} = \abstain}}\left[ \left(\indic{y=y_i} - \indic{y \neq y_i}\right) \neq \left(\indic{y \in \fh[_c]{x_i}} - \indic{y \notin \fh[_c]{x_i}}\right) \right] \prob_{\substack{(x_i,y) \sim \edistrW_c}}\left[ \fh[_c]{x_i} = \abstain \right]\text{.}
\end{align*}
Notice that when $\fh[_c]{}$ abstains the best possible behaviour is to randomly predict the labels and thus 
\begin{align*}
\prob_{\substack{(x_i,y) \sim \edistrW_c \\ \fh[_c]{x_i} = \abstain}}\left[ \left(\indic{y=y_i} - \indic{y \neq y_i}\right) \neq \left(\indic{y \in \fh[_c]{x_i}} - \indic{y \notin \fh[_c]{x_i}}\right) \right] = \frac{1}{2}\text{.}
\end{align*}
Hence to show that $\fh[_c]{}$ has at most the error of a random predictor it is sufficient to show that
\begin{align*}
\prob_{\substack{(x_i,y) \sim \edistrW_c \\ \fh[_c]{x_i} \neq \abstain}}\left[ \left(\indic{y=y_i} - \indic{y \neq y_i}\right) \neq \left(\indic{y \in \fh[_c]{x_i}} - \indic{y \notin \fh[_c]{x_i}}\right) \right] = \frac{W_{c,-}}{W_{c,+}+W_{c,-}} \leq \frac{1}{2}
\end{align*}
where $W_{c,+}$ and $W_{c,-}$ are respectively the proportions of correctly and incorrectly classified examples.

On the one hand assume that $\fh[_c]{}$ uses $x_j$ and $x_k$ as reference points. Let $\setS_{c,j}$ be the set of all the training examples for which $\fh[_c]{}$ does not abstain and which are closer to $x_j$ than to $x_k$.
If $y \in o_j$ then we have that:
\begin{align}
\sum_{z_i \in \setS_{c,j}} \indic{y=y_i} w_{c,x_i,y} > \sum_{z_i \in \setS_{c,j}} \indic{y \neq y_i} w_{c,x_i,y} \text{.} \label{app:eq:ineqojpos}
\end{align}
Similarly, if $y \notin o_j$ then it implies that:
\begin{align}
\sum_{z_i \in \setS_{c,j}} \indic{y \neq y_i} w_{c,x_i,y}  \geq \sum_{z_i \in \setS_{c,j}} \indic{y=y_i} w_{c,x_i,y} \text{.} \label{app:eq:ineqojneg}
\end{align}
The same result holds for $o_k$.

On the other hand recall that
\begin{align}
W_{c,+} ={}& \sum_{\substack{(x_i,y_i) \in \setS \\ \fh[_c]{x_i} \neq \abstain}} \left( \indic{y_i \in \fh[_c]{x_i}} w_{c,x_i,y_i} + \sum_{y \neq y_i} \indic{y \notin \fh[_c]{x_i}} w_{c,x_i,y} \right)  \text{,} \nonumber\\
W_{c,-} ={}& \sum_{\substack{(x_i,y_i) \in \setS \\ \fh[_c]{x_i} \neq \abstain}} \left( \indic{y_i \notin \fh[_c]{x_i}} w_{c,x_i,y_i} + \sum_{y \neq y_i} \indic{y \in \fh[_c]{x_i}} w_{c,x_i,y} \right)  \text{.} \label{app:eq:classifweights}
\end{align}

Combining Equations~\eqref{app:eq:ineqojpos},~\eqref{app:eq:ineqojneg},~and~\eqref{app:eq:classifweights} we obtain
\begin{align*}
&& W_{c,-} \leq{}& W_{c,+} \\
&\Leftrightarrow& 2W_{c,-} \leq{}& W_{c,+} + W_{c,-} \\
&\Leftrightarrow& \frac{W_{c,-}}{W_{c,+} + W_{c,-}} \leq{}& \frac{1}{2}
\end{align*}
which proves that $\fh[_c]{}$ has at most the error of a random predictor.

To see that our labelling strategy is optimal just notice that any other labelling strategy would either predict some labels in $o_j$ or $o_k$ that do not conform with Equation~\eqref{app:eq:ineqojpos} or not predict some labels from $o_j$ or $o_k$ that do not conform with Equation~\eqref{app:eq:ineqojneg}. As a consequence $W_{c,-}$ would be increased and thus the error of $\fh[_c]{}$ would also be increased. This concludes the proof of the first claim.

\paragraph{Second claim}
Recall that $\alpha_c$ is computed as
\begin{align*}
\alpha_c = \frac{1}{2}\log{\left(\frac{W_{c,+} + \frac{1}{n}}{W_{c,-} + \frac{1}{n}}\right)}\text{.}
\end{align*}
Also notice that $\fh[_c]{}$ is a weak classifier, that is better than a random predictor, when $W_{c,-} < W_{c,+}$. We have that:
\begin{align*}
&& W_{c,-} &< W_{c,+} \\
&\Leftrightarrow& W_{c,-}+\frac{1}{n} &< W_{c,+}+\frac{1}{n} \\
&\Leftrightarrow& 1 &< \frac{W_{c,+}+\frac{1}{n}}{W_{c,-}+\frac{1}{n}} \\
&\Leftrightarrow& 0 &< \frac{1}{2}\log\left(\frac{W_{c,+}+\frac{1}{n}}{W_{c,-}+\frac{1}{n}}\right)
\end{align*}
and it proves the second claim.
\end{proof}

\section{Proof of Theorem~\ref{th:trainerror}}
\label{app:sec:train}

\begin{reth}{\ref{th:trainerror}}[Reduction of the Training Error]
Let $\setS$ be a set of $n$ examples and $\setT$ be a set of $m$ triplets (obtained as described in Section~\ref{sec:contrib}).
Let $\fH{\cdot}$ be the classifier obtained after $C$ iterations of TripletBoost (Algorithm~1) using $\setS$ and $\setT$ as input. It holds that:
\begin{align*}
\prob_{(x,y) \in \setS}\left[ \fH{x} \neq y \right] \leq{}& \frac{\abs{\spaceY}}{2}\prod_{c=1}^C Z_c
\end{align*}
with $Z_c = (1-W_{c,+}-W_{c,-}) + (W_{c,+})\cdot\sqrt{\frac{W_{c,-} + \frac{1}{n}}{W_{c,+} + \frac{1}{n}}} + (W_{c,-})\cdot\sqrt{\frac{W_{c,+} + \frac{1}{n}}{W_{c,-} + \frac{1}{n}}} \leq 1$.
\end{reth}

\begin{proof}
This result is inherited from AdaBoost.MO and can be obtained directly by applying Theorem~10.4 in Chapter 10 of the book of \cite{schapire2012boosting} using a so-called loss-based decoding and noticing that in our case $\Bar{K} = \abs{\spaceY}$ and $\rho = 2$. For the sake of completeness we detail the complete proof below.

First of all, let $\fF{x,y} = \sum_{c=1}^C\alpha_c \indic{\fh[_c]{x}\neq \abstain} \left(\indic{y\in\fh[_c]{x}} - \indic{y\notin\fh[_c]{x}}\right)$ and notice that the strong classifier $\fH{\cdot}$ can be equivalently rewritten as:
\begin{align*}
&&\fH{\cdot} ={}& \argmax_{y\in\spaceY}\left( \sum_{c=1}^C\alpha_c \indic{\fh[_c]{\cdot}\neq \abstain\wedge y\in\fh[_c]{\cdot}}\right) \\
\tag{$x > y \Leftrightarrow x-(B-x) > y-(B-y)$.} \\
&\Leftrightarrow&  \fH{\cdot} ={}& \argmax_{y\in\spaceY}\left( \sum_{c=1}^C\alpha_c \indic{\fh[_c]{\cdot}\neq \abstain} \left(\indic{y\in\fh[_c]{\cdot}} - \indic{y\notin\fh[_c]{\cdot}}\right)\right) \\
\tag{$\fF{\cdot,y} = \sum_{c=1}^C\alpha_c \indic{\fh[_c]{\cdot}\neq \abstain} \left(\indic{y\in\fh[_c]{\cdot}} - \indic{y\notin\fh[_c]{\cdot}}\right)$.} \\
&\Leftrightarrow&  \fH{\cdot} ={}& \argmax_{y\in\spaceY}\fF{\cdot,y} \\
&\Leftrightarrow&  \fH{\cdot} ={}& \argmin_{y\in\spaceY}-\fF{\cdot,y} \\
&\Leftrightarrow&  \fH{\cdot} ={}& \argmin_{y\in\spaceY}\exp\left(- \fF{\cdot,y}\right) \\
\tag{$\argmin_{x} \exp\left(-\ff{x}\right) = \argmin_{x} -\exp\left(\ff{x}\right)$.}\\
&\Leftrightarrow&  \fH{\cdot} ={}& \argmin_{y\in\spaceY}\exp\left(- \fF{\cdot,y}\right) - \exp\left(\fF{\cdot,y}\right) \\
\tag{Adding a constant does not change the value of $\argmin$.} \\
&\Leftrightarrow&  \fH{\cdot} ={}& \argmin_{y\in\spaceY}\exp\left(- \fF{\cdot,y}\right) - \exp\left(\fF{\cdot,y}\right) + \sum_{y^\prime \in \spaceY}\exp\left(\fF{\cdot,y^\prime}\right)\\
&\Leftrightarrow&  \fH{\cdot} ={}& \argmin_{y\in\spaceY}\sum_{y^\prime \in \spaceY}\exp\left(-(\indic{y=y^\prime} - \indic{y\neq y^\prime})\fF{\cdot,y^\prime}\right)\text{.}
\end{align*}

Given an example $(x,y) \in \setS$, the classifier $\fH{\cdot}$ makes an error if and only if there exists a label $\ell \neq y$ such that
\begin{align*}
\sum_{y^\prime \in \spaceY}\exp\left(-(\indic{\ell=y^\prime} - \indic{\ell\neq y^\prime})\fF{x,y^\prime}\right) \leq \sum_{y^\prime \in \spaceY}\exp\left(-(\indic{y=y^\prime} - \indic{y\neq y^\prime})\fF{x,y^\prime}\right)\text{.}
\end{align*}
It implies that, when $\fH{\cdot}$ makes an error
\begin{align*}
\sum_{y^\prime \in \spaceY} & \exp\left(-(\indic{y=y^\prime} - \indic{y\neq y^\prime})\fF{x,y^\prime}\right) \\
\geq{}& \frac{1}{2}\sum_{y^\prime \in \spaceY}\exp\left(-(\indic{\ell=y^\prime} - \indic{\ell\neq y^\prime})\fF{x,y^\prime}\right) + \frac{1}{2} \sum_{y^\prime \in \spaceY}\exp\left(-(\indic{y=y^\prime} - \indic{y\neq y^\prime})\fF{x,y^\prime}\right) \\
\tag{Dropping all the terms that do not depend on $y$ and $\ell$ since $\forall x \in \nsetR,\exp{x} > 0$.} \\
\geq{}& \frac{1}{2}\sum_{y^\prime \in \left\lbrace y,\ell \right\rbrace} \left[\exp\left(-(\indic{\ell=y^\prime} - \indic{\ell\neq y^\prime})\fF{x,y^\prime}\right) + \exp\left(-(\indic{y=y^\prime} - \indic{y\neq y^\prime})\fF{x,y^\prime}\right) \right] \\
\tag{$\forall y^\prime \in \left\lbrace y,\ell \right\rbrace, (\indic{\ell=y^\prime} - \indic{\ell\neq y^\prime}) = -(\indic{y=y^\prime} - \indic{y\neq y^\prime})$.} \\
\geq{}& \frac{1}{2}\sum_{y^\prime \in \left\lbrace y,\ell \right\rbrace} \left[\frac{1}{\exp\left(-(\indic{y=y^\prime} - \indic{y\neq y^\prime})\fF{x,y^\prime}\right)} + \exp\left(-(\indic{y=y^\prime} - \indic{y\neq y^\prime})\fF{x,y^\prime}\right) \right] \\
\tag{$\forall x>0, x + \frac{1}{x} \geq 2$.} \\
\geq{}& 2\text{.}
\end{align*}
Assuming that the classifier $\fH{\cdot}$ makes $M$ errors, we have that:
\begin{align}
2M \leq{}& \sum_{(x,y) \in \setS} \sum_{y^\prime \in \spaceY} \exp\left(-(\indic{y=y^\prime} - \indic{y\neq y^\prime})\fF{x,y^\prime}\right) \text{.} \label{app:ineq:boundnumerrors}
\end{align}

Now, we need to show that the right hand side of Inequality~\eqref{app:ineq:boundnumerrors} is bounded by $n\abs{\spaceY}\prod_{c=1}^C Z_c$ where $n$ is the number of training examples and $Z_c$ is the normalization factor used in TripletBoost. Recall that, given an example $(x, y) \in \setS$ and a label $y^\prime \in \spaceY$, the weight $w_{c+1,x,y^\prime}$, obtained after $c$ iterations of our algorithm, is
\begin{align}
&& w_{c+1,x,y^\prime} ={}& \left\lbrace \begin{array}{ll}
\frac{w_{c,x,y^\prime}}{Z_c} & \text{if $\fh[_c]{x} = \abstain$,} \\
\frac{w_{c,x,y^\prime}}{Z_c}\exp\left(-\alpha_c\left(\indic{y^\prime=y} - \indic{y^\prime\neq y}\right)\left(\indic{y^\prime\in \fh[_c]{x}} - \indic{y^\prime\notin \fh[_c]{x}}\right)\right) & \text{if $\fh[_c]{x} \neq \abstain$.}
\end{array} \right. \nonumber \\
&\Leftrightarrow& w_{c+1,x,y^\prime} ={}& \frac{w_{c,x,y^\prime}}{Z_c}\exp\left(-\left(\indic{y^\prime=y} - \indic{y^\prime\neq y}\right)\alpha_c\indic{\fh[_c]{x} \neq \abstain}\left(\indic{y^\prime\in \fh[_c]{x}} - \indic{y^\prime\notin \fh[_c]{x}}\right)\right) \label{app:eq:recursivedef}
\end{align}
By recursively using Equation~\eqref{app:eq:recursivedef}, we obtain
\begin{align}
&& w_{C+1,x,y^\prime} ={}& \frac{w_{1,x,y^\prime}}{\prod_{c=1}^C Z_c}\prod_{c=1}^C\exp\left(-\left(\indic{y^\prime=y} - \indic{y^\prime\neq y}\right)\alpha_c\indic{\fh[_c]{x} \neq \abstain}\left(\indic{y^\prime\in \fh[_c]{x}} - \indic{y^\prime\notin \fh[_c]{x}}\right)\right) \\
\tag{$\exp(x)\exp(y) = \exp(x+y)$.} \\
&\Leftrightarrow& w_{C+1,x,y^\prime} ={}& \frac{w_{1,x,y^\prime}}{\prod_{c=1}^C Z_c}\exp\left(-\left(\indic{y^\prime=y} - \indic{y^\prime\neq y}\right)\sum_{c=1}^C\alpha_c\indic{\fh[_c]{x} \neq \abstain}\left(\indic{y^\prime\in \fh[_c]{x}} - \indic{y^\prime\notin \fh[_c]{x}}\right)\right) \\
\tag{$\fF{x,y} = \sum_{c=1}^C\alpha_c \indic{\fh[_c]{x}\neq \abstain} \left(\indic{y\in\fh[_c]{x}} - \indic{y\notin\fh[_c]{x}}\right)$.} \\
&\Leftrightarrow& w_{C+1,x,y^\prime} ={}& \frac{w_{1,x,y^\prime}}{\prod_{c=1}^C Z_c}\exp\left(-\left(\indic{y^\prime=y} - \indic{y^\prime\neq y}\right)\fF{x,y^\prime}\right) \\
\tag{$\sum_{(x,y) \in \setS} \sum_{y^\prime \in \spaceY} w_{C+1,x,y^\prime} = 1$ and $\forall (x, y) \in \setS, \forall y^\prime \in \spaceY, w_{1,x,y^\prime} = \frac{1}{n\abs{\spaceY}}$.} \\
&\Leftrightarrow& \prod_{c=1}^C Z_c ={}& \frac{1}{n\abs{\spaceY}}\sum_{(x,y) \in \setS} \sum_{y^\prime \in \spaceY}  \exp\left(-\left(\indic{y^\prime=y} - \indic{y^\prime\neq y}\right)\fF{x,y^\prime}\right) \text{.} \label{app:eq:prodzcvalue}
\end{align}
Combining this result with Inequality~\eqref{app:ineq:boundnumerrors} we obtain:
\begin{align*}
2M \leq{}& n\abs{\spaceY}\prod_{c=1}^C Z_c \text{.}
\end{align*}
Noticing that $\prob_{(x,y) \in \setS}\left[ \fH{x} \neq y \right] = \frac{M}{n}$ gives the first part of the theorem:
\begin{align*}
\prob_{(x,y) \in \setS}\left[ \fH{x} \neq y \right] \leq{}& \frac{\abs{\spaceY}}{2}\prod_{c=1}^C Z_c \text{.}
\end{align*}

In the last part of the proof we show that $Z_c = (1-W_{c,+}-W_{c,-}) + (W_{c,+})\cdot\sqrt{\frac{W_{c,-} + \frac{1}{n}}{W_{c,+} + \frac{1}{n}}} + (W_{c,-})\cdot\sqrt{\frac{W_{c,+} + \frac{1}{n}}{W_{c,-} + \frac{1}{n}}} \leq 1$. Using the result obtained in Equation~\eqref{app:eq:recursivedef} we have that
\begin{align}
&& w_{c+1,x,y^\prime} ={}& \frac{w_{c,x,y^\prime}}{Z_c}\exp\left(-\left(\indic{y^\prime=y} - \indic{y^\prime\neq y}\right)\alpha_c\indic{\fh[_c]{x} \neq \abstain}\left(\indic{y^\prime\in \fh[_c]{x}} - \indic{y^\prime\notin \fh[_c]{x}}\right)\right) \nonumber \\
\tag{$\sum_{(x,y) \in \setS} \sum_{y^\prime \in \spaceY} w_{c+1,x,y^\prime} = 1$.} \\
&\Leftrightarrow& Z_c ={}& \sum_{(x,y) \in \setS} \sum_{y^\prime \in \spaceY} w_{c,x,y^\prime}\exp\left(-\left(\indic{y^\prime=y} - \indic{y^\prime\neq y}\right)\alpha_c\indic{\fh[_c]{x} \neq \abstain}\left(\indic{y^\prime\in \fh[_c]{x}} - \indic{y^\prime\notin \fh[_c]{x}}\right)\right) \nonumber \\
\tag{$\exp(0) = 1$.} \\
&\Leftrightarrow& Z_c ={}& \sum_{\substack{(x,y) \in \setS \\ \fh[_c]{x} = \abstain}} \sum_{y^\prime \in \spaceY} w_{c,x,y^\prime} + \sum_{\substack{(x,y) \in \setS \\ \fh[_c]{x} \neq \abstain}} \sum_{y^\prime \in \spaceY} w_{c,x,y^\prime} \exp\left(-\left(\indic{y^\prime=y} - \indic{y^\prime\neq y}\right)\alpha_c\left(\indic{y^\prime\in \fh[_c]{x}} - \indic{y^\prime\notin \fh[_c]{x}}\right)\right) \label{app:eq:zcvalue}
\end{align}
Notice that $W_{c,+}$ and $W_{c,-}$, the weights of correctly and incorrectly classified examples, can be rewritten as
\begin{align*}
W_{c,+} ={}& \sum_{\substack{(x,y) \in \setS \\ \fh[_c]{x} \neq \abstain}} \sum_{y^\prime \in \spaceY} \left(\indic{y^\prime = y \wedge y^\prime \in \fh[_c]{x}} + \indic{y^\prime \neq y \wedge y^\prime \notin \fh[_c]{x}} \right) w_{c,x,y^\prime} \text{,}\\
W_{c,-} ={}& \sum_{\substack{(x,y) \in \setS \\ \fh[_c]{x} \neq \abstain}} \sum_{y^\prime \in \spaceY} \left(\indic{y^\prime = y \wedge y^\prime \notin \fh[_c]{x}} + \indic{y^\prime \neq y \wedge y^\prime \in \fh[_c]{x}} \right) w_{c,x,y^\prime} \text{.}
\end{align*}
Replacing $\alpha_c = \frac{1}{2}\log{\left(\frac{W_{c,+} + \frac{1}{n}}{W_{c,-} + \frac{1}{n}}\right)}$ in Equation~\eqref{app:eq:zcvalue} gives
\begin{align*}
\sum_{\substack{(x,y) \in \setS \\ \fh[_c]{x} \neq \abstain}} & \sum_{y^\prime \in \spaceY} w_{c,x,y^\prime} \exp\left(-\left(\indic{y^\prime=y} - \indic{y^\prime\neq y}\right)\alpha_c\left(\indic{y^\prime\in \fh[_c]{x}} - \indic{y^\prime\notin \fh[_c]{x}}\right)\right) \\
={}& \sum_{\substack{(x,y) \in \setS \\ \fh[_c]{x} \neq \abstain}} \sum_{y^\prime \in \spaceY} w_{c,x,y^\prime} \exp\left(-\left(\indic{y^\prime=y} - \indic{y^\prime\neq y}\right)\frac{1}{2}\log{\left(\frac{W_{c,+} + \frac{1}{n}}{W_{c,-} + \frac{1}{n}}\right)}\left(\indic{y^\prime\in \fh[_c]{x}} - \indic{y^\prime\notin \fh[_c]{x}}\right)\right) \\
\tag{$\exp(a\log(b)) = b^a$} \\
={}& \sum_{\substack{(x,y) \in \setS \\ \fh[_c]{x} \neq \abstain}} \sum_{y^\prime \in \spaceY} w_{c,x,y^\prime} \left(\frac{W_{c,+} + \frac{1}{n}}{W_{c,-} + \frac{1}{n}}\right)^{\left(-\left(\indic{y^\prime=y} - \indic{y^\prime\neq y}\right)\frac{1}{2}\left(\indic{y^\prime\in \fh[_c]{x}} - \indic{y^\prime\notin \fh[_c]{x}}\right)\right)} \\
\tag{$\exp(a\log(b)) = b^a$} \\
={}& \sum_{\substack{(x,y) \in \setS \\ \fh[_c]{x} \neq \abstain}} \sum_{y^\prime \in \spaceY} w_{c,x,y^\prime} \left(\indic{y^\prime = y \wedge y^\prime \in \fh[_c]{x}} + \indic{y^\prime \neq y \wedge y^\prime \notin \fh[_c]{x}} \right) \sqrt{\frac{W_{c,-} + \frac{1}{n}}{W_{c,+} + \frac{1}{n}}} \\
&+ \sum_{\substack{(x,y) \in \setS \\ \fh[_c]{x} \neq \abstain}} \sum_{y^\prime \in \spaceY} w_{c,x,y^\prime} \left(\indic{y^\prime = y \wedge y^\prime \notin \fh[_c]{x}} + \indic{y^\prime \neq y \wedge y^\prime \in \fh[_c]{x}} \right) \sqrt{\frac{W_{c,+} + \frac{1}{n}}{W_{c,-} + \frac{1}{n}}} \\
={}& (W_{c,+}) \cdot \sqrt{\frac{W_{c,-} + \frac{1}{n}}{W_{c,+} + \frac{1}{n}}} + (W_{c,-}) \cdot \sqrt{\frac{W_{c,+} + \frac{1}{n}}{W_{c,-} + \frac{1}{n}}}
\end{align*}
Similarly, using the fact that $\sum_{(x,y) \in \setS} \sum_{y^\prime \in \spaceY} w_{c,x,y^\prime} = 1$, we obtain
\begin{align*}
\sum_{\substack{(x,y) \in \setS \\ \fh[_c]{x} = \abstain}} \sum_{y^\prime \in \spaceY} w_{c,x,y^\prime} = 1 - W_{c,+} - W_{c,-}\text{.}
\end{align*}
Combining these two results with Equation~\eqref{app:eq:zcvalue} gives the value of $Z_c$:
\begin{align*}
Z_c = (1-W_{c,+}-W_{c,-}) + (W_{c,+})\cdot\sqrt{\frac{W_{c,-} + \frac{1}{n}}{W_{c,+} + \frac{1}{n}}} + (W_{c,-})\cdot\sqrt{\frac{W_{c,+} + \frac{1}{n}}{W_{c,-} + \frac{1}{n}}} \text{.}
\end{align*}
To show that $Z_c \leq 1$ note that
\begin{align*}
&& \left(\sqrt{W_{c,+} + \frac{1}{n}} - \sqrt{W_{c,-} + \frac{1}{n}}\right)^2 \geq{}& 0 \\
&\Leftrightarrow& W_{c,+} + \frac{1}{n} + W_{c,-} + \frac{1}{n} - 2\sqrt{W_{c,+} + \frac{1}{n}}\sqrt{W_{c,-} + \frac{1}{n}} \geq{}& 0 \\
&\Leftrightarrow& W_{c,+} + W_{c,-} + \frac{2}{n} - \left(W_{c,+} + \frac{1}{n}\right)\sqrt{\frac{W_{c,-} + \frac{1}{n}}{W_{c,+} + \frac{1}{n}}} - \left(W_{c,-} + \frac{1}{n}\right)\sqrt{\frac{W_{c,+} + \frac{1}{n}}{W_{c,-} + \frac{1}{n}}} \geq{}& 0 \\
&\Leftrightarrow& W_{c,+} + W_{c,-} + \frac{2}{n} - W_{c,+}\sqrt{\frac{W_{c,-} + \frac{1}{n}}{W_{c,+} + \frac{1}{n}}} - \frac{1}{n}\sqrt{\frac{W_{c,-} + \frac{1}{n}}{W_{c,+} + \frac{1}{n}}} - W_{c,-}\sqrt{\frac{W_{c,+} + \frac{1}{n}}{W_{c,-} + \frac{1}{n}}} - \frac{1}{n}\sqrt{\frac{W_{c,+} + \frac{1}{n}}{W_{c,-} + \frac{1}{n}}} \geq{}& 0 \\
&\Leftrightarrow& - W_{c,+} - W_{c,-} - \frac{2}{n} + W_{c,+}\sqrt{\frac{W_{c,-} + \frac{1}{n}}{W_{c,+} + \frac{1}{n}}} + \frac{1}{n}\sqrt{\frac{W_{c,-} + \frac{1}{n}}{W_{c,+} + \frac{1}{n}}} + W_{c,-}\sqrt{\frac{W_{c,+} + \frac{1}{n}}{W_{c,-} + \frac{1}{n}}} + \frac{1}{n}\sqrt{\frac{W_{c,+} + \frac{1}{n}}{W_{c,-} + \frac{1}{n}}} \leq{}& 0 \\
&\Leftrightarrow& 1 - W_{c,+} - W_{c,-} - \frac{2}{n} + W_{c,+}\sqrt{\frac{W_{c,-} + \frac{1}{n}}{W_{c,+} + \frac{1}{n}}} + \frac{1}{n}\sqrt{\frac{W_{c,-} + \frac{1}{n}}{W_{c,+} + \frac{1}{n}}} + W_{c,-}\sqrt{\frac{W_{c,+} + \frac{1}{n}}{W_{c,-} + \frac{1}{n}}} + \frac{1}{n}\sqrt{\frac{W_{c,+} + \frac{1}{n}}{W_{c,-} + \frac{1}{n}}} \leq{}& 1 \\
&\Leftrightarrow& Z_c - \frac{2}{n} + \frac{1}{n}\sqrt{\frac{W_{c,-} + \frac{1}{n}}{W_{c,+} + \frac{1}{n}}} + \frac{1}{n}\sqrt{\frac{W_{c,+} + \frac{1}{n}}{W_{c,-} + \frac{1}{n}}} \leq{}& 1 \\
&\Leftrightarrow& Z_c - \frac{2}{n} + \frac{1}{n}\left(\sqrt{\frac{W_{c,-} + \frac{1}{n}}{W_{c,+} + \frac{1}{n}}} + \sqrt{\frac{W_{c,+} + \frac{1}{n}}{W_{c,-} + \frac{1}{n}}}\right) \leq{}& 1 \\
\tag{$\forall x>0,x+\frac{1}{x} \geq 2$.} \\
&\Rightarrow& Z_c \leq{}& 1 \text{.}
\end{align*}
It concludes the proof of the theorem.
\end{proof}

\section{Proof of Theorem~\ref{th:testerror}}
\label{app:sec:test}

\begin{reth}{\ref{th:testerror}}[Generalization Guarantees]
Let $\distrS$ be a distribution over $\spaceX \times \spaceY$, let $\setS$ be a set of $n$ examples drawn i.i.d. from $\distrS$, and let $\setT$ be a set of $m$ triplets (obtained as described in Section~\ref{sec:contrib}).
Let $\fH{\cdot}$ be the classifier obtained after $C$ iterations of TripletBoost (Algorithm~\ref{alg:tripletboost}) using $\setS$ and $\setT$ as input.
Let $\spaceH$ be a set of triplet classifiers as defined in Section~\ref{sec:tripletclassifiers}.
Then, given a margin parameter $\theta > \sqrt{\frac{\log{\abs{\spaceH}}}{16\abs{\spaceY}^2n}}$ and a measure of the confidence of $\fH{\cdot}$ in its predictions $\ftheta[_{\fH{}}]{x,y}$, with probability at least $1-\delta$, we have that
\begin{align*}
\prob_{(x,y) \sim \distrS}\left[H(x) \neq y\right] \leq{}& \prob_{(x,y) \in \setS}\left[ \ftheta[_{\fH{}}]{x,y} \leq \theta \right] \\
&\negspace{6em}+ \bigO{\sqrt{\frac{\log{\left(\frac{1}{\delta}\right)}}{n} + \log{\left(\frac{\abs{\spaceY}^2n\theta^2}{\log{\abs{\spaceH}}}\right)}\frac{\log{\abs{\spaceH}}}{n\theta^2}}}
\end{align*}
Furthermore we have that 
\begin{align*}
\prob_{(x,y) \in \setS}\left[ \ftheta[_{\fH{}}]{x,y} \leq \theta \right] \leq \frac{\abs{\spaceY}}{2} \prod_{c=1}^C Z_c \sqrt{\left(\frac{W_{c,+} + \frac{1}{n}}{W_{c,-} + \frac{1}{n}}\right)^\theta} \text{.}
\end{align*}
\end{reth}

\begin{proof}
This result is inherited from AdaBoost.MO and can be obtained directly by following the steps of Exercise 10.3 in Chapter 10 of the book of \cite{schapire2012boosting}. Following our notations $\ftheta[_{\ff{},\eta}]{x,y}$ is equal to $\mathcal{M}_{f,\eta}\left(x,y\right)$ and the free parameter $n$ in their proof has been chosen as $\ceil{\frac{16}{\theta^2}\log{\left(\frac{4\abs{\spaceY}^2n\theta^2}{\log{\abs{\spaceH}}}\right)}}$.
The second part of the theorem  can be obtained in a similar way by following the steps of Exercise 10.4 in Chapter 10 of the book of \cite{schapire2012boosting}. For the sake of completeness we write the complete proof below.

First, we define several quantities used throughout. Recall that our approach outputs a strong classifier $\fH{\cdot}$ defined as
 \begin{align}
 \fH{x} = \argmax_{y \in \spaceY} \sum_{c=1}^C \alpha_c\indic{\fh[_c]{x} \neq \abstain \wedge y \in \fh[_c]{x}} \label{app:eq:strongclassifier}
 \end{align}
 As in the proof of Theorem~\ref{th:trainerror}, let $\fF{x,y} = \sum_{c=1}^C\alpha_c \indic{\fh[_c]{x}\neq \abstain} \left(\indic{y\in\fh[_c]{x}} - \indic{y\notin\fh[_c]{x}}\right)$ and let $\ff{} : \spaceX \times \spaceY \rightarrow \nsetR$ be the normalized version of $\fF{}$:
 \begin{align}
 \ff{x,y} =\sum_{c=1}^C a_c\indic{\fh[_c]{x}\neq \abstain} \left(\indic{y\in\fh[_c]{x}} - \indic{y\notin\fh[_c]{x}}\right) \text{.} \label{app:eq:normalizedstrongclassifier}
 \end{align}
 where $a_c = \frac{\alpha_c}{\sum_{c=1}^C \alpha_c} \in [0,1]$ since we have that $\alpha_c \geq 0$ because of our labelling strategy for the triplet classifiers. From the proof of Theorem~\ref{th:trainerror}, recall that we have $\displaystyle{\fH{x} = \argmax_{y \in \spaceY} \fF{x,y} = \argmax_{y \in \spaceY} \ff{x,y}}$.

Notice that $\ff{}$ is part of the following convex set:
\begin{align}
 \text{co}\left(\spaceH\right) = \left\lbrace \ff{} : x,y \mapsto \sum_{c=1}^C a_c\indic{\fh[_c]{x}\neq \abstain} \left(\indic{y\in\fh[_c]{x}} - \indic{y\notin\fh[_c]{x}}\right) \left| a_1,\ldots,a_C \geq 0; \sum_{c=1}^Ca_c = 1;\fh[_1]{},\ldots \fh[_C]{} \in \spaceH \right.  \right\rbrace\text{.} \label{app:eq:convexset}
\end{align}
Let us define the following set of unweighted averages of size $T$:
\begin{align}
 \text{un}_T\left(\spaceH\right) = \left\lbrace \left. \ff{} : x,y \mapsto \frac{1}{T}\sum_{t=1}^T \indic{\fh[_t]{x}\neq \abstain} \left(\indic{y\in\fh[_t]{x}} - \indic{y\notin\fh[_t]{x}}\right) \right| \fh[_1]{},\ldots \fh[_T]{} \in \spaceH  \right\rbrace \text{.} \label{app:eq:unweightedset}
\end{align}

For $\ff{} \in \text{co}\left(\spaceH\right)$, $\eta > 0$, and $(x,y) \in \spaceX \times \spaceY$, let
\begin{align}
 \fnu[_{\ff{},\eta}]{x,y} = - \frac{1}{\eta} \log\left(\frac{1}{\abs{\spaceY}} \sum_{y^\prime \in \spaceY} \exp\Big(-\eta \left(\indic{y=y^\prime} - \indic{y\neq y^\prime}\right)\ff{x,y^\prime}\Big)\right)\text{.} \label{app:eq:nu}
\end{align}
We use this quantity to define the following confidence margin for $(x,y) \in \spaceX \times \spaceY$:
\begin{align}
 \ftheta[_{\fH{}}]{x,y} = \ftheta[_{\ff{},\eta}]{x,y} = \frac{1}{2}\left(\fnu[_{\ff{},\eta}]{x,y} - \max_{y^\prime \neq y} \fnu[_{\ff{},\eta}]{x,y^\prime} \right) \text{.}
\end{align}

 \paragraph{Bound on the true risk of $\fH{\cdot}$.} Note that $\ftheta[_{\ff{},\eta}]{x,y} \in [-1,1]$ since $\ff{x,y} \in [-1,1]$. Similarly, given an example $(x,y) \in \spaceX \times \spaceY$, if $\fH{\cdot}$ makes an error then it implies that there exists a label $\ell \neq y$ such that
\begin{align*}
&& \fF{x,y} \leq{}& \fF{x,\ell} \\
&\Leftrightarrow& \ff{x,y} \leq{}& \ff{x,\ell} \\
\tag{$\eta > 0$.} \\
&\Leftrightarrow& \eta\ff{x,y} \leq{}& \eta\ff{x,\ell} \\
&\Leftrightarrow& \exp\left(-\eta\ff{x,y}\right)-\exp\left(\eta\ff{x,y}\right) \geq{}& \exp\left(-\eta\ff{x,\ell}\right) - \exp\left(\eta\ff{x,\ell}\right) \\
&\Leftrightarrow& \exp\left(-\eta\ff{x,y}\right)-\exp\left(\eta\ff{x,y}\right) + \sum_{y^\prime \in \spaceY} \exp\left(\eta \ff{x,y^\prime}\right) \geq{}& \exp\left(-\eta\ff{x,\ell}\right) - \exp\left(\eta\ff{x,\ell}\right) + \sum_{y^\prime \in \spaceY} \exp\left(\eta \ff{x,y^\prime}\right) \\
&\Leftrightarrow& \sum_{y^\prime \in \spaceY} \exp\Big(-\eta \left(\indic{y=y^\prime} - \indic{y\neq y^\prime}\right)\ff{x,y^\prime}\Big) \geq{}& \sum_{y^\prime \in \spaceY} \exp\Big(-\eta \left(\indic{\ell=y^\prime} - \indic{\ell\neq y^\prime}\right)\ff{x,y^\prime}\Big) \\
&\Leftrightarrow& \log\left(\frac{1}{\abs{\spaceY}}\sum_{y^\prime \in \spaceY} \exp\Big(-\eta \left(\indic{y=y^\prime} - \indic{y\neq y^\prime}\right)\ff{x,y^\prime}\Big)\right) \geq{}& \log\left(\frac{1}{\abs{\spaceY}}\sum_{y^\prime \in \spaceY} \exp\Big(-\eta \left(\indic{\ell=y^\prime} - \indic{\ell\neq y^\prime}\right)\ff{x,y^\prime}\Big)\right) \\
&\Leftrightarrow& \fnu[_{\ff{},\eta}]{x,y} \leq{}& \fnu[_{\ff{},\eta}]{x,\ell} \\
&\Leftrightarrow& \fnu[_{\ff{},\eta}]{x,y} - \fnu[_{\ff{},\eta}]{x,\ell} \leq{}& 0 \\
&\Leftrightarrow& \ftheta[_{\ff{},\eta}]{x,y} \leq{}& 0\text{.}
\end{align*}
It implies that
\begin{align*}
\prob_{(x,y) \sim \distrS}\left[H(x) \neq y\right] = \prob_{(x,y) \sim \distrS}\left[\ftheta[_{\ff{},\eta}]{x,y} \leq 0\right]
\end{align*}
and hence, in the following, we prove an upper bound on $\prob_{(x,y) \sim \distrS}\left[\ftheta[_{\ff{},\eta}]{x,y} \leq 0\right]$.

To this end, let us first define $\ftildef{x,y} = \frac{1}{T}\sum_{t=1}^T \indic{\ftildeh[_t]{x}\neq \abstain} \left(\indic{y\in\ftildeh[_t]{x}} - \indic{y\notin\ftildeh[_t]{x}}\right) \in \text{un}_{T}$ with $\ftildeh[_1]{},\ldots,\ftildeh[_T]{}$ drawn from $\spaceH$ independently at random following the probability distribution defined by the weights $a_1,\ldots,a_C$. Notice that $\ftildef{x,y} \in \left[-1,1\right]$ and that given $x \in \spaceX$ and $y^\prime \in \spaceY$ fixed, $\expect_{\ftildef{}}\left[\ftildef{x,y^\prime}\right] = \ff{x,y^\prime}$. Hence, using Hoeffding's inequality, it holds that
\begin{align*}
    \prob_{\ftildef{}}\left[ \abs{\ftildef{x,y^\prime} - \ff{x,y^\prime}} \geq \frac{\theta}{4}\right] \leq 2\exp\left(\frac{-T\theta^2}{32}\right)\text{.}
\end{align*}
Then, by the union bound, it holds that
\begin{align}
    \prob_{\ftildef{}}\left[ \exists y^\prime \in \spaceY : \abs{\ftildef{x,y^\prime} - \ff{x,y^\prime}} \geq \frac{\theta}{4}\right] \leq 2\abs{\spaceY}\exp\left(\frac{-T\theta^2}{32}\right)\text{.} \label{app:ineq:betaT}
\end{align}
Consider the following technical fact from \citet{schapire2012boosting}. Define the grid
\begin{align*}
\varepsilon_\theta = \left\lbrace \frac{4\log\abs{\spaceY}}{i\theta} : i = 1, \ldots, \left\lceil \frac{8\log\abs{\spaceY}}{\theta^2} \right\rceil \right\rbrace\text{.}
\end{align*}
For any $\eta > 0$, let $\hat{\eta}$ be the closest value in $\varepsilon_\theta$ to $\eta$. Then for all $\ff{} \in \text{co}(\mathcal{H})$ and for all $(x,y) \in \mathcal{X} \times \mathcal{Y}$,
\begin{align}
\abs{\fnu[_{\ff{},\eta}]{x,y} - \fnu[_{\ff{},\hat{\eta}}]{x,y}} \leq \frac{\theta}{4}\text{.} \label{app:ineq:techfact}
\end{align}
We will now show that, with high probability, $\ftheta[_{\ff{},\eta}]{x,y}$ is close to $\ftheta[_{\ftildef{},\hat{\eta}}]{x,y}$ for any distribution $\distrT$ over $\spaceX \times \spaceY$.

Let $x \in \spaceX$ and $y \in \spaceY$ then we have that
\begin{align}
\abs{\fnu[_{\ff{},\eta}]{x,y} - \fnu[_{\ftildef{},\hat{\eta}}]{x,y}} \leq{}& \abs{\fnu[_{\ff{},\eta}]{x,y} - \fnu[_{\ftildef{},\eta}]{x,y}} + \abs{\fnu[_{\ftildef{},\eta}]{x,y} - \fnu[_{\ftildef{},\hat{\eta}}]{x,y}} \\
\tag{Using Inequality~\eqref{app:ineq:techfact} since $\ftildef{} \in \text{co}(\spaceH)$.} \\
\leq{}& \abs{\fnu[_{\ff{},\eta}]{x,y} - \fnu[_{\ftildef{},\eta}]{x,y}} + \frac{\theta}{4} \\
\leq{}& \left|- \frac{1}{\eta} \log\left(\frac{1}{\abs{\spaceY}} \sum_{y^\prime \in \spaceY} \exp\Big(-\eta \left(\indic{y=y^\prime} - \indic{y\neq y^\prime}\right)\ff{x,y^\prime}\Big)\right)\right. \\
&+\left. \frac{1}{\eta} \log\left(\frac{1}{\abs{\spaceY}} \sum_{y^\prime \in \spaceY} \exp\Big(-\eta \left(\indic{y=y^\prime} - \indic{y\neq y^\prime}\right)\ftildef{x,y^\prime}\Big)\right)\right| + \frac{\theta}{4} \\
\leq{}& \abs{\frac{1}{\eta} \log\left(\frac{\frac{1}{\abs{\spaceY}} \sum_{y^\prime \in \spaceY} \exp\Big(-\eta \left(\indic{y=y^\prime} - \indic{y\neq y^\prime}\right)\ftildef{x,y^\prime}\Big)}{\frac{1}{\abs{\spaceY}} \sum_{y^\prime \in \spaceY} \exp\Big(-\eta \left(\indic{y=y^\prime} - \indic{y\neq y^\prime}\right)\ff{x,y^\prime}\Big)}\right)} + \frac{\theta}{4} \\
\tag{By induction, one can show that $\displaystyle{\frac{\sum_i a_i}{\sum_i b_i} \leq \max_i \frac{a_i}{b_i}}$ for positive $a_i$ and $b_i$.} \\
\leq{}& \abs{\frac{1}{\eta} \log\left(\max_{y^\prime \in \spaceY}\frac{\exp\Big(-\eta \left(\indic{y=y^\prime} - \indic{y\neq y^\prime}\right)\ftildef{x,y^\prime}\Big)}{\exp\Big(-\eta \left(\indic{y=y^\prime} - \indic{y\neq y^\prime}\right)\ff{x,y^\prime}\Big)}\right)} + \frac{\theta}{4} \\
\leq{}& \abs{\frac{1}{\eta} \log\left(\max_{y^\prime \in \spaceY}\exp\Big(\eta \left(\indic{y=y^\prime} - \indic{y\neq y^\prime}\right)\ff{x,y^\prime}-\eta \left(\indic{y=y^\prime} - \indic{y\neq y^\prime}\right)\ftildef{x,y^\prime}\Big)\right)} + \frac{\theta}{4} \\
\leq{}& \abs{\max_{y^\prime \in \spaceY} \left(\indic{y=y^\prime} - \indic{y\neq y^\prime}\right)\ff{x,y^\prime}- \left(\indic{y=y^\prime} - \indic{y\neq y^\prime}\right)\ftildef{x,y^\prime}} + \frac{\theta}{4} \\
\leq{}& \max_{y^\prime \in \spaceY} \abs{\ff{x,y^\prime} - \ftildef{x,y^\prime}} + \frac{\theta}{4} \text{.} \label{app:ineq:boundnu}
\end{align}
And then it implies that:
\begin{align*}
\abs{\ftheta[_{\ff{},\eta}]{x,y} - \ftheta[_{\ftildef{},\hat{\eta}}]{x,y}} \leq{}& \abs{\frac{1}{2}\left(\fnu[_{\ff{},\eta}]{x,y} - \max_{y^\prime \neq y} \fnu[_{\ff{},\eta}]{x,y^\prime} \right) - \frac{1}{2}\left(\fnu[_{\ftildef{},\hat{\eta}}]{x,y} + \max_{y^\prime \neq y} \fnu[_{\ftildef{},\hat{\eta}}]{x,y^\prime}\right)} \\
\leq{}& \frac{1}{2}\abs{\fnu[_{\ff{},\eta}]{x,y} - \fnu[_{\ftildef{},\hat{\eta}}]{x,y}} + \frac{1}{2}\abs{\max_{y^\prime \neq y} \fnu[_{\ftildef{},\hat{\eta}}]{x,y^\prime} - \max_{y^\prime \neq y} \fnu[_{\ff{},\eta}]{x,y^\prime}} \\
\tag{Noticing that if $\abs{\fg{x}-\fg[^\prime]{x}} \leq B$ for all $x$ then $\abs{\max_x \fg{x}-\max_x\fg[^\prime]{x}} \leq B$ and applying Inequality~\eqref{app:ineq:boundnu} twice.} \\
\leq{}& \max_{y^\prime \in \spaceY} \abs{\ff{x,y^\prime} - \ftildef{x,y^\prime}} + \frac{\theta}{4}
\end{align*}
Then, using Inequality~\eqref{app:ineq:betaT}, for any $(x,y) \in \spaceX\times\spaceY$, we have
\begin{align*}
    \prob_{\ftildef{}}\left[ \abs{\ftheta[_{\ff{},\eta}]{x,y} - \ftheta[_{\ftildef{},\hat{\eta}}]{x,y}} \geq \frac{\theta}{2}\right] \leq 2\abs{\spaceY}\exp\left(\frac{-T\theta^2}{32}\right)\text{.}
\end{align*}
Then, given any distribution $\distrT$ over $\spaceX\times\spaceY$, we have that:
\begin{align*}
    \expect_{(x,y) \sim \distrT} \prob_{\ftildef{}}\left[ \abs{\ftheta[_{\ff{},\eta}]{x,y} - \ftheta[_{\ftildef{},\hat{\eta}}]{x,y}} \geq \frac{\theta}{2}\right] \leq \expect_{(x,y) \sim \distrT}  2\abs{\spaceY}\exp\left(\frac{-T\theta^2}{32}\right)
\end{align*}
which implies
\begin{align}
    \prob_{\distrT,\ftildef{}}\left[ \abs{\ftheta[_{\ff{},\eta}]{x,y} - \ftheta[_{\ftildef{},\hat{\eta}}]{x,y}} \geq \frac{\theta}{2}\right] \leq 2\abs{\spaceY}\exp\left(\frac{-T\theta^2}{32}\right). \label{app:ineq:fclosetildef}
\end{align}
This last inequality shows that $\ftheta[_{\ff{},\eta}]{x,y}$ is close to $\ftheta[_{\ftildef{},\hat{\eta}}]{x,y}$ for any distribution $\distrT$ over $\spaceX\times\spaceY$.

Let $\theta$ and $T$ be fixed. Given $\ftildef{}$ and $\hat{\eta}$, let $B_i$ be a Bernoulli random variable that is $1$ if $\ftheta[_{\ftildef{},\hat{\eta}}]{x_i,y_i} \leq \frac{\theta}{2}$ with $(x_i,y_i) \in \setS$. We have that
\begin{align*}
    \prob_{(x,y)\in\setS}\left[ \ftheta[_{\ftildef{},\hat{\eta}}]{x,y} \leq \frac{\theta}{2} \right] ={}& \frac{1}{n} \sum_{i=1}^n B_i \text{,}\\
    \prob_{(x,y)\sim \distrS}\left[ \ftheta[_{\ftildef{},\hat{\eta}}]{x,y} \leq \frac{\theta}{2} \right] ={}& \expect_{S\sim\distrS^n}\left[\prob_{(x,y)\in\setS}\left[ \ftheta[_{\ftildef{},\hat{\eta}}]{x,y} \leq \frac{\theta}{2} \right]\right]\text{.}
\end{align*}
Then, using Hoeffding's inequality, it holds that:
\begin{align*}
    \prob_{\setS \sim \distrS^n}\left[\prob_{(x,y)\sim \distrS}\left[ \ftheta[_{\ftildef{},\hat{\eta}}]{x,y} \leq \frac{\theta}{2} \right] \geq \prob_{(x,y)\in\setS}\left[ \ftheta[_{\ftildef{},\hat{\eta}}]{x,y} \leq \frac{\theta}{2} \right] + \epsilon \right] \leq \exp(-2\epsilon^2n)\text{.}
\end{align*}
Using the union bound for both $\ftildef{}$ and $\hat{\eta}$, it holds that $\forall \ftildef{} \in \text{un}_T$ and $\forall \hat{\eta} \in \varepsilon_\theta$, with probability at least $1-\delta$ over the choice of the random training set $\setS$:
\begin{align}
    \prob_{(x,y)\sim \distrS}\left[ \ftheta[_{\ftildef{},\hat{\eta}}]{x,y} \leq \frac{\theta}{2} \right] \leq \prob_{(x,y)\in\setS}\left[ \ftheta[_{\ftildef{},\hat{\eta}}]{x,y} \leq \frac{\theta}{2} \right] + \sqrt{\frac{\log\left(\frac{\abs{\varepsilon_\theta}\abs{\spaceH}^T}{\delta}\right)}{2n}}\text{.} \label{app:ineq:genftildef}
\end{align}

Now, notice that for two events $E$ and $E^\prime$ it holds that $\prob\left[E\right] \leq \prob\left[E^\prime\right] + \prob\left[E,\neg E^\prime\right]$ and thus:
\begin{align*}
\prob_{(x,y) \sim \distrS}\left[\ftheta[_{\ff{},\eta}]{x,y} \leq 0\right] \leq{}& \prob_{(x,y) \sim \distrS,\ftildef{}}\left[\ftheta[_{\ftildef{},\hat{\eta}}]{x,y} \leq \frac{\theta}{2}\right] + \prob_{(x,y) \sim \distrS,\ftildef{}}\left[\ftheta[_{\ff{},\eta}]{x,y} \leq 0,\ftheta[_{\ftildef{},\hat{\eta}}]{x,y} > \frac{\theta}{2}\right] \\ 
\tag{Applying Inequality~\eqref{app:ineq:fclosetildef} to the rightmost term.} \\
\leq{}& \prob_{(x,y) \sim \distrS,\ftildef{}}\left[\ftheta[_{\ftildef{},\hat{\eta}}]{x,y} \leq \frac{\theta}{2}\right] + 2\abs{\spaceY}\exp\left(\frac{-T\theta^2}{32}\right) \text{.}
\end{align*}
Similarly it holds that
\begin{align*}
\prob_{(x,y) \in \setS,\ftildef{}}\left[\ftheta[_{\ftildef{},\hat{\eta}}]{x,y} \leq \frac{\theta}{2}\right] \leq{}& \prob_{(x,y) \in \setS}\left[\ftheta[_{\ff{},\eta}]{x,y} \leq \theta\right] + 2\abs{\spaceY}\exp\left(\frac{-T\theta^2}{32}\right) \text{.}
\end{align*}
Combining these two results with Inequality~\eqref{app:ineq:genftildef} yields with probability at least $1-\delta$:
\begin{align}
\prob_{(x,y) \sim \distrS}\left[\ftheta[_{\ff{},\eta}]{x,y} \leq 0\right] \leq{}& \prob_{(x,y) \in \setS}\left[\ftheta[_{\ff{},\eta}]{x,y} \leq \theta\right] + \sqrt{\frac{\log\left(\frac{\abs{\varepsilon_\theta}\abs{\spaceH}^T}{\delta}\right)}{2n}} + 4\abs{\spaceY}\exp\left(\frac{-T\theta^2}{32}\right) \text{.}
\end{align}
Choosing $T = \ceil{\frac{16}{\theta^2}\log{\left(\frac{16\abs{\spaceY}^2n\theta^2}{\log{\abs{\spaceH}}}\right)}}$ with $\theta > \sqrt{\frac{\log\abs{\spaceH}}{16\abs{\spaceY}^2n}}$ and noticing that $\abs{\varepsilon_\theta} \leq \frac{8\log\abs{\spaceY}}{\theta^2}+1$ the error term becomes:

\begin{align*}
\sqrt{\frac{\log\left(\frac{\abs{\varepsilon_\theta}\abs{\spaceH}^T}{\delta}\right)}{2n}} +{}& 4\abs{\spaceY}\exp\left(\frac{-T\theta^2}{32}\right) \\
={}& \sqrt{\frac{\log\left(\frac{\abs{\varepsilon_\theta}}{\delta}\right) + T\log\abs{\spaceH}}{2n}} + 4\abs{\spaceY}\exp\left(\frac{-T\theta^2}{32}\right) \\
={}& \sqrt{\frac{\log\left(\frac{\abs{\varepsilon_\theta}}{\delta}\right) + \left(\frac{16}{\theta^2}\log{\left(\frac{16\abs{\spaceY}^2n\theta^2}{\log{\abs{\spaceH}}}\right)}+1\right)\log\abs{\spaceH}}{2n}} + 4\abs{\spaceY}\exp\left(\frac{-\frac{16}{\theta^2}\log{\left(\frac{16\abs{\spaceY}^2n\theta^2}{\log{\abs{\spaceH}}}\right)}\theta^2}{32}\right) \\
={}& \sqrt{\frac{\log\left(\frac{\abs{\varepsilon_\theta}}{\delta}\right) + \left(\frac{16}{\theta^2}\log{\left(\frac{16\abs{\spaceY}^2n\theta^2}{\log{\abs{\spaceH}}}\right)}+1\right)\log\abs{\spaceH}}{2n}} + 4\abs{\spaceY}\exp\left(-\frac{1}{2}\log{\left(\frac{16\abs{\spaceY}^2n\theta^2}{\log{\abs{\spaceH}}}\right)}\right) \\
={}& \sqrt{\frac{\log\left(\frac{\abs{\varepsilon_\theta}}{\delta}\right) + \left(\frac{16}{\theta^2}\log{\left(\frac{16\abs{\spaceY}^2n\theta^2}{\log{\abs{\spaceH}}}\right)}+1\right)\log\abs{\spaceH}}{2n}} + 4\abs{\spaceY}\sqrt{\frac{\log{\abs{\spaceH}}}{16\abs{\spaceY}^2n\theta^2}} \\
={}& \sqrt{\frac{\log\left(\frac{\abs{\varepsilon_\theta}}{\delta}\right) + \left(\frac{16}{\theta^2}\log{\left(\frac{16\abs{\spaceY}^2n\theta^2}{\log{\abs{\spaceH}}}\right)}+1\right)\log\abs{\spaceH}}{2n}} + \sqrt{\frac{\log{\abs{\spaceH}}}{n\theta^2}} \\
={}& \sqrt{\frac{\log\left(\frac{8\log\abs{\spaceY}}{\theta^2}+1\right)}{2n} + \frac{\log\left(\frac{1}{\delta}\right)}{2n} + 8\log{\left(\frac{16\abs{\spaceY}^2n\theta^2}{\log{\abs{\spaceH}}}\right)}\frac{\log\abs{\spaceH}}{n\theta^2} + \frac{\log\abs{\spaceH}}{2n}} + \sqrt{\frac{\log{\abs{\spaceH}}}{n\theta^2}} \text{.}
\end{align*}
This concludes the first part of the proof.

\paragraph{Bound on the empirical margin of $\fH{\cdot}$.}
This part of the proof is very similar to the proof of Theorem~\ref{th:trainerror} and the main difference is that $\theta$ has to be taken into account. To this end, let us assume that the strong classifier violates the margin, that is $\ftheta[_{\ff{},\eta}]{x,y} \leq \theta$. Given a label $\ell \neq y$, we define the following quantities:
\begin{align*}
    \fz{y^\prime} = \eta \left(\indic{y=y^\prime} - \indic{y\neq y^\prime}\right)\ff{x,y^\prime} - \eta\theta \text{,} \\
    \fz[_\ell]{y^\prime} = \eta \left(\indic{\ell=y^\prime} - \indic{\ell\neq y^\prime}\right)\ff{x,y^\prime} + \eta\theta \text{.}
\end{align*}
Notice that:
\begin{align*}
&& \ftheta[_{\ff{},\eta}]{x,y} \leq{}& \theta \\
&\Leftrightarrow& \frac{1}{2}\left(\fnu[_{\ff{},\eta}]{x,y} - \max_{\ell \neq y} \fnu[_{\ff{},\eta}]{x,\ell} \right) \leq{}& \theta \\
&\Leftrightarrow& \fnu[_{\ff{},\eta}]{x,y} - \theta \leq{}& \max_{\ell \neq y} \fnu[_{\ff{},\eta}]{x,\ell} + \theta \\
&\Leftrightarrow& - \log\left(\frac{1}{\abs{\spaceY}} \sum_{y^\prime \in \spaceY} \exp\Big(-\eta \left(\indic{y=y^\prime} - \indic{y\neq y^\prime}\right)\ff{x,y^\prime}\Big)\right) - \eta \theta \leq{}&\\
&& &\negspace{9em}\max_{\ell \neq y} - \log\left(\frac{1}{\abs{\spaceY}} \sum_{y^\prime \in \spaceY} \exp\Big(-\eta \left(\indic{\ell=y^\prime} - \indic{\ell\neq y^\prime}\right)\ff{x,y^\prime}\Big)\right) + \eta\theta \\
&\Leftrightarrow& \log\left(\frac{1}{\abs{\spaceY}} \sum_{y^\prime \in \spaceY} \exp\Big(-\eta \left(\indic{y=y^\prime} - \indic{y\neq y^\prime}\right)\ff{x,y^\prime}\Big)\right) + \eta \theta \geq{}&\\
&& &\negspace{9em}\min_{\ell \neq y} \log\left(\frac{1}{\abs{\spaceY}} \sum_{y^\prime \in \spaceY} \exp\Big(-\eta \left(\indic{\ell=y^\prime} - \indic{\ell\neq y^\prime}\right)\ff{x,y^\prime}\Big)\right) - \eta\theta \\
&\Leftrightarrow& \sum_{y^\prime \in \spaceY} \exp\left(\eta \theta\right)\exp\Big(-\eta \left(\indic{y=y^\prime} - \indic{y\neq y^\prime}\right)\ff{x,y^\prime}\Big) \geq{}&\\
&& &\negspace{9em}\min_{\ell \neq y} \sum_{y^\prime \in \spaceY} \exp\left(-\eta\theta\right) \exp\Big(-\eta \left(\indic{\ell=y^\prime} - \indic{\ell\neq y^\prime}\right)\ff{x,y^\prime}\Big) \\
&\Leftrightarrow& \sum_{y^\prime \in \spaceY} \exp\Big(-\eta \left(\indic{y=y^\prime} - \indic{y\neq y^\prime}\right)\ff{x,y^\prime} + \eta \theta\Big) \geq{}&\\
&& &\negspace{9em}\min_{\ell \neq y} \sum_{y^\prime \in \spaceY} \exp\Big(-\eta \left(\indic{\ell=y^\prime} - \indic{\ell\neq y^\prime}\right)\ff{x,y^\prime}-\eta\theta\Big) \\
&\Leftrightarrow& \sum_{y^\prime \in \spaceY} \exp\Big(-\fz{y^\prime}\Big) \geq{}& \min_{\ell \neq y} \sum_{y^\prime \in \spaceY} \exp\Big(-\fz[_\ell]{y^\prime}\Big)
\end{align*}
Then, when the strong classifier violates the margin, that is $\ftheta[_{\ff{},\eta}]{x,y} \leq \theta$, we have that:
\begin{align*}
\sum_{y^\prime \in \spaceY} \exp\Big(-\fz{y^\prime}\Big) \geq{}& \frac{1}{2}\sum_{y^\prime \in \spaceY} \exp\Big(-\fz{y^\prime}\Big) + \frac{1}{2} \min_{\ell \neq y} \sum_{y^\prime \in \spaceY} \exp\Big(-\fz[_\ell]{y^\prime}\Big) \\
\geq{}& \min_{\ell \neq y}\left[\frac{1}{2}\sum_{y^\prime \in \spaceY} \exp\Big(-\fz{y^\prime}\Big) + \frac{1}{2} \sum_{y^\prime \in \spaceY} \exp\Big(-\fz[_\ell]{y^\prime}\Big)\right] \\
\geq{}& \frac{1}{2}\min_{\ell \neq y}\sum_{y^\prime \in \spaceY}\left[\exp\Big(-\fz{y^\prime}\Big) + \exp\Big(-\fz[_\ell]{y^\prime}\Big)\right] \\
\geq{}& \frac{1}{2}\min_{\ell \neq y}\sum_{y^\prime \in \spaceY}\left[\exp\Big(-\eta \left(\indic{y=y^\prime} - \indic{y\neq y^\prime}\right)\ff{x,y^\prime} + \eta\theta\Big) + \exp\Big(-\eta \left(\indic{\ell=y^\prime} - \indic{\ell\neq y^\prime}\right)\ff{x,y^\prime} - \eta\theta\Big)\right] \\
\tag{Dropping all the terms that do not depend on $y$ and $\ell$ since $\forall x \in \nsetR,\exp{x} > 0$.} \\
\geq{}& \frac{1}{2}\min_{\ell \neq y}\sum_{y^\prime \in \left\lbrace y,\ell\right\rbrace}\left[\exp\Big(-\eta \left(\indic{y=y^\prime} - \indic{y\neq y^\prime}\right)\ff{x,y^\prime} + \eta\theta\Big) + \exp\Big(-\eta \left(\indic{\ell=y^\prime} - \indic{\ell\neq y^\prime}\right)\ff{x,y^\prime} - \eta\theta\Big)\right] \\
\tag{$\forall y^\prime \in \left\lbrace y,\ell \right\rbrace, (\indic{\ell=y^\prime} - \indic{\ell\neq y^\prime}) = -(\indic{y=y^\prime} - \indic{y\neq y^\prime})$.} \\
\geq{}& \frac{1}{2}\min_{\ell \neq y}\sum_{y^\prime \in \left\lbrace y,\ell\right\rbrace}\left[\exp\Big(-\eta \left(\indic{y=y^\prime} - \indic{y\neq y^\prime}\right)\ff{x,y^\prime} + \eta\theta\Big) + \exp\Big(\eta \left(\indic{y=y^\prime} - \indic{y\neq y^\prime}\right)\ff{x,y^\prime} - \eta\theta\Big)\right] \\
\geq{}& \frac{1}{2}\min_{\ell \neq y}\sum_{y^\prime \in \left\lbrace y,\ell\right\rbrace}\left[\exp\Big(-\fz{y^\prime}\Big) + \exp\Big(\fz{y^\prime}\Big)\right] \\
\tag{$\forall x>0, x + \frac{1}{x} \geq 2$.} \\
\geq{}& 2
\end{align*}
Now assume that the strong classifier violates the margin $M$ times, that is $\ftheta[_{\ff{},\eta}]{x,y} \leq \theta$ for $M$ examples, then
\begin{align*}
&& 2M \leq{}& \sum_{(x,y) \in \setS} \sum_{y^\prime \in \spaceY} \exp\Big(-\fz{y^\prime}\Big) \\
&\Leftrightarrow& 2M \leq{}& \sum_{(x,y) \in \setS} \sum_{y^\prime \in \spaceY} \exp\Big(-\eta \left(\indic{y=y^\prime} - \indic{y\neq y^\prime}\right)\ff{x,y^\prime} + \eta\theta\Big) \\
&\Leftrightarrow& 2M \leq{}& \sum_{(x,y) \in \setS} \sum_{y^\prime \in \spaceY} \exp\Big(-\eta \left(\indic{y=y^\prime} - \indic{y\neq y^\prime}\right)\ff{x,y^\prime}\Big)\exp\Big(\eta\theta\Big) \\
\tag{Let $\eta = \sum_{c} \alpha_c$ and recall that $\left(\sum_{c} \alpha_c\right)\ff{x,y} = \fF{x,y}$.} \\
&\Leftrightarrow& 2M \leq{}& \sum_{(x,y) \in \setS} \sum_{y^\prime \in \spaceY} \exp\Big(- \left(\indic{y=y^\prime} - \indic{y\neq y^\prime}\right)\fF{x,y^\prime}\Big)\exp\Big(\sum_{c} \alpha_c\theta\Big) \\
&\Leftrightarrow& 2M \leq{}& \exp\Big(\sum_{c} \alpha_c\theta\Big) \sum_{(x,y) \in \setS} \sum_{y^\prime \in \spaceY} \exp\Big(- \left(\indic{y=y^\prime} - \indic{y\neq y^\prime}\right)\fF{x,y^\prime}\Big)
\end{align*}
Recall that $\sum_{(x,y) \in \setS} \sum_{y^\prime \in \spaceY} \exp\Big(- \left(\indic{y=y^\prime} - \indic{y\neq y^\prime}\right)\fF{x,y^\prime}\Big) = n\abs{\spaceY}\prod_{c=1}^C Z_c$ as shown in Equation~\eqref{app:eq:prodzcvalue} in the proof of Theorem~\ref{th:trainerror}.
\begin{align*}
&& 2M \leq{}& \exp\Big(\sum_{c} \alpha_c\theta\Big) n\abs{\spaceY}\prod_{c=1}^C Z_c \\
&\Leftrightarrow& 2M \leq{}& \exp\Big(\sum_{c} \alpha_c\theta\Big) n\abs{\spaceY}\prod_{c=1}^C Z_c \\
&\Leftrightarrow& 2M \leq{}& n\abs{\spaceY}\prod_{c=1}^C Z_c\exp\left(\alpha_c\theta\right)  \\
&\Leftrightarrow& \frac{M}{n} \leq{}& \frac{\abs{\spaceY}}{2}\prod_{c=1}^C Z_c\exp\left(\alpha_c\theta\right) 
\end{align*}
Replacing $\alpha_c$ by its value, that is $\alpha_c = \frac{1}{2}\log{\left(\frac{W_{c,+} + \frac{1}{n}}{W_{c,-} + \frac{1}{n}}\right)}$ and noticing that $\displaystyle{\prob_{(x,y) \in \setS}\left[ \ftheta[_{f,\eta}]{x,y} \leq \theta \right] = \frac{M}{n}}$ concludes the proof.
\end{proof}

\section{Proof of Theorem~\ref{th:limit}}
\label{app:sec:limit}

\begin{reth}{\ref{th:limit}}[Lower bound on the probability that a strong classifier abstains]
Let $n \geq 2$ be the number of training examples, $p = \frac{2n^k}{n^3}$ with $k \in \left[0,3-\frac{\log(2)}{\log(n)}\right)$ be the probability that a triplet is available in the triplet set $\setT$ and $C = \frac{n^\beta}{2}$ with $\beta \in \left[0,1+\frac{\log{(n-1)}}{\log{(n)}}\right]$ be the number of classifiers combined in the learned classifier.
Let $\algA$ be any algorithm learning a classifier $\fH{\cdot} = \argmax_{y \in \spaceY}\left(\sum_{c=1}^C \alpha_c \indic{y \in \fh[_c]{\cdot}}\right)$ that combines several triplet classifiers using some weights $\alpha_c \in \nsetR$. Assume that triplet classifiers that abstain on all the training examples have a weight of $0$ (that is if $\fh[_c]{x_i} = \abstain$ for all the examples $(x_i,y_i) \in \setS$ then $\alpha_c = 0$).
Then the probability that $\fH{}$ abstains on a test example is bounded as follows:
\begin{align}
\prob_{(x,y) \sim \distrS} \left[ \fH{x} = \abstain \right] \geq{}& \Big(1-p + p\left(1-p\right)^n \Big)^C\text{.} \label{app:eq:boundexact}
\end{align}
\end{reth}

\begin{proof}
Recall that $\fH{\cdot} = \argmax_{y \in \spaceY}\left(\sum_{c=1}^C \alpha_c \indic{y \in \fh[_c]{\cdot}}\right)$.
Each component $\alpha_c \indic{y \in \fh[_c]{\cdot}}$ of $\fH{}$ abstains for a given example $x$ if $\fh[_c]{x} = \abstain$ or $\alpha_c = 0$.
Overall $\fH{}$ abstains if all its components also abstain and we bound this probability here.

Recall that $\setS = \left\lbrace (x_i,y_i) \right\rbrace_{i = 1}^n$ is a set of $n$ examples drawn i.i.d. from an unknown distribution $\distrS$ over the space $\spaceX \times \spaceY$ where $\abs{\spaceY} < \infty$ and recall that $C$ is the number of classifiers selected in $\fH{}$. Finally recall that each triplet is obtained with probability $p$ independently of the other triplets. It implies that each triplet classifier $\fh[_c]{\cdot}$ abstains with probability $q = 1 - p$ independently of the other triplet classifiers.

First we start by defining several random variables.
Let $Y_{1,1}, \ldots, Y_{n,C}$ be $nC$ independent random variables encoding the event that a classifier abstains on one example:
\begin{align}
Y_{i,c} = \left\lbrace \begin{matrix}
0 & \text{if $\fh[_c]{}$ abstains for example $x_i$,} \\
1 & \text{otherwise.} \\
\end{matrix} \right.
\end{align}
From the definition of the classifiers each of these random variables follow a Bernoulli distribution $\operatorname{B} \left(1, p\right)$.

Let $Y_{x,1}, \ldots, Y_{x,C}$ be $C$ independent random variables encoding the event that a classifier abstains on a new example $x$:
\begin{align}
Y_{x,c} = \left\lbrace \begin{matrix}
0 & \text{if $\fh[_c]{}$ abstains for example $x$,} \\
1 & \text{otherwise.} \\
\end{matrix} \right.
\end{align}
From the definition of the classifiers each of these random variables follow a Bernoulli distribution $\operatorname{B} \left(1, p\right)$.

Let $V_1, \ldots, V_C$ be $C$ independent random variables encoding the probability that a triplet classifier abstains on all the training examples:
\begin{align}
V_c = \sum_{i = 1}^n Y_{i,c}.
\end{align}
As a sum of $n$ independent Bernoulli trials, these random variables follow a Binomial distribution $\operatorname{B} \left(n, p\right)$.

Using the law of total probabilities we have that:
\begin{align*}
\prob_{(x,y) \sim \distrS} \left[ \fH{x} = \abstain \right] ={}& \prob_{(x,y) \sim \distrS} \left[ \fH{x} = \abstain \left| \sum_{c=1}^C V_cY_{x,c} = 0 \right. \right] \prob_{(x,y) \sim \distrS} \left[ \sum_{c=1}^C V_cY_{x,c} = 0 \right] \\
&+ \prob_{(x,y) \sim \distrS} \left[ \fH{x} = \abstain \left| \sum_{c=1}^C V_cY_{x,c} \neq 0 \right. \right] \prob_{(x,y) \sim \distrS} \left[ \sum_{c=1}^C V_cY_{x,c} \neq 0 \right]\text{.}
\end{align*}

Note that from the assumption that $\forall (x_i,y_i) \in \setS, \fh[_c]{x_i} = \abstain$ then $\alpha_c = 0$ and the definition of the random variables $Y_{\vx,\cdot}$ we have that:
\begin{align*}
\prob_{(x,y) \sim \distrS} \left[ \fH{x} = \abstain \left| \sum_{c=1}^C V_cY_{x,c} = 0 \right. \right] = 1 \text{.}
\end{align*}

Similarly, we have that:
\begin{align}
\prob_{(x,y) \sim \distrS} \left[ \fH{x} = \abstain \left| \sum_{c=1}^C V_cY_{x,c} \neq 0 \right. \right] \geq 0 \text{.} \label{app:eq:inequalitycond}
\end{align}

Hence we have that:
\begin{align}
\prob_{(x,y) \sim \distrS} \left[ \fH{x} = \abstain \right] \geq{}& \prob_{(x,y) \sim \distrS} \left[ \sum_{c=1}^C V_cY_{x,c} = 0 \right]\text{.} \label{app:eq:inequalitybound}
\end{align}

Let $U_1,\ldots,U_C$ be $C$ independent random variables such that:
\begin{align}
U_c = \left\lbrace \begin{array}{ll}
1 & \text{if $V_cY_{\vx,c} = 0$,} \\
0 & \text{otherwise.}
\end{array}\text{.} \right.
\end{align}
These random variables follow a Bernoulli distribution $\operatorname{B} \left(1, (1-p) + p(1-p)^n\right)$. To see that, note that $V_c$ and $Y_{x,c}$ are independent and thus we have that:
\begin{align*}
\prob_{(x,y) \sim \distrS} \left[ V_cY_{x,c} = 0  \right] ={}& 1 - \prob_{(x,y) \sim \distrS} \left[ V_cY_{x,c} \neq 0  \right] \\
={}& 1 - \prob_{(x,y) \sim \distrS} \left[ \left\lbrace V_c \neq 0 \right\rbrace \cap \left\lbrace Y_{x,c} \neq 0 \right\rbrace  \right] \\
={}& 1 - \prob_{(x,y) \sim \distrS} \left[ V_c \neq 0 \right] \prob_{(x,y) \sim \distrS} \left[ Y_{x,c} \neq 0 \right] \\
\tag{$Y_{x,c}$ follows a Bernoulli distribution $\operatorname{B} \left(1,p\right)$.} \\
={}& 1 - \left( 1 - \prob_{(x,y) \sim \distrS} \left[ V_c = 0 \right] \right) p \\
\tag{$V_c$ follows a Binomial distribution $\operatorname{B} \left(n,p\right)$.} \\
={}& 1 - \left( 1 - \begin{pmatrix}
n\\
0
\end{pmatrix}p^0\left(1-p\right)^{n-0} \right) p \\
={}& 1 - \left( 1 - \left(1-p\right)^n \right) p \\
={}& \left(1 - p\right) + p\left(1-p\right)^n \text{.}
\end{align*}

The r.h.s. of Inequality~\eqref{app:eq:inequalitybound} can then be written as:
\begin{align}
\prob_{(x,y) \sim \distrS} \left[ \sum_{c=1}^C V_cY_{x,c} = 0 \right] = \prob_{(x,y) \sim \distrS} \left[ \sum_{c=1}^C U_c = c \right]
\end{align}
which, by definition of the random variables $U_\cdot$, corresponds to the probability of obtaining $c$ successes among $c$ Bernoulli trials. In other words the random variable $U = \sum_{c=1}^C U_j$ follows a Binomial distribution $\operatorname{B} \left(c, (1-p) + p(1-p)^n\right)$. Following this we have that:
\begin{align}
\prob_{(x,y) \sim \distrS} \left[ \sum_{c=1}^C V_cY_{x,c} = 0 \right] ={}& \begin{pmatrix}
c \\
c
\end{pmatrix} \left( (1-p) + p(1-p)^n \right)^C \left( 1 - (1-p) + p(1-p)^n \right)^{C-C} \\
={}& \left( (1-p) + p(1-p)^n \right)^C \text{.}
\end{align}

Hence plugging this result in Inequality~\eqref{app:eq:inequalitybound} we have that:
\begin{align*}
\prob_{(x,y) \sim \distrS} \left[ \fH{x} = \abstain \right] \geq{}& \left( (1-p) + p(1-p)^n \right)^C
\end{align*}
which concludes the proof.
\end{proof}

\section{Proof of Example~\ref{ex:limit}}
\label{app:sec:limitex}

Note that here we prove a more general version of Example~\ref{ex:limit} than the one presented in the main paper. The latter follows directly by choosing $\beta = 1+\frac{\log{(n-1)}}{\log{(n)}}$. 

\begin{reex}{\ref{ex:limit}}
In Theorem~\ref{th:limit}, we choose different values for $p = 2n^{k-3}$ the probability of that a triplet is available and $C = \frac{n^\beta}{2}$ the number of combined classifiers. Then we take the limit as $n \to \infty$ to obtain the following results.

If $0 \leq \beta < 1$ then:
\begin{align*}
\lim_{n\rightarrow+\infty} \prob_{(x,y) \sim \distrS} \left[ \fH{x} = \abstain \right] \geq{}&
\left\lbrace
\begin{array}{ll}
1 & \text{if $0 \leq k < 3-\beta$,} \\
\exp(-1) & \text{if $k = 3-\beta$,} \\
0 & \text{if $3-\beta < k < 3-\frac{\log(2)}{\log(n)}$.}
\end{array}
\right.
\end{align*}

If $\beta = 1$ then:
\begin{align*}
\lim_{n\rightarrow+\infty} \prob_{(x,y) \sim \distrS} \left[ \fH{x} = \abstain \right] \geq{}&
\left\lbrace
\begin{array}{ll}
1 & \text{if $0 \leq k < 2$,} \\
\exp(\exp(-2)-1) & \text{if $k = 2$,} \\
0 & \text{if $2 < k < 3-\frac{\log(2)}{\log(n)}$.}
\end{array}
\right.
\end{align*}

If $1 < \beta \leq 1+\frac{\log{(n-1)}}{\log{(n)}}$ then:
\begin{align}
\lim_{n\rightarrow+\infty} \prob_{(x,y) \sim \distrS} \left[ \fH{x} = \abstain \right] \geq{}&
\left\lbrace
\begin{array}{ll}
1 & \text{if $0 \leq k<\frac{5-\beta}{2}$,} \\
\exp(-2) & \text{if $k=\frac{5-\beta}{2}$,} \\
0 & \text{if $\frac{5-\beta}{2}< k < 3-\frac{\log(2)}{\log(n)}$.}
\end{array}
\right. \label{app:eq:boundlimit}
\end{align}
\end{reex}

\begin{proof}
Replacing $p$ and $C$ by their values in Equation~\eqref{app:eq:boundexact} of Theorem~\ref{th:limit} and applying Lemma~\ref{app:lem:limit} gives the example.
\end{proof}

The next Lemma is a technical result used in the proof of Theorem~\ref{th:limit}.

\begin{mylem}[Limit of the right hand side of Theorem~\ref{th:limit}.\label{app:lem:limit}]
Given $n \geq 2$, $k \in \left[0,3-\frac{\log(2)}{\log(n)}\right)$ and $\beta \in \left[0,1+\frac{\log{(n-1)}}{\log{(n)}}\right]$, define:
\begin{align}
\ff{n,k,\beta} \doteq \left( 1-2n^{k-3} + 2n^{k-3}\left(1-2n^{k-3}\right)^n \right)^\frac{n^\beta}{2} \label{app:eq:seekedlimit}
\end{align}
then the following limits hold:

If $0 \leq \beta < 1$ then:
\begin{align*}
\lim_{n\rightarrow+\infty} \ff{n,k,\beta} ={}&
\left\lbrace
\begin{array}{ll}
1 & \text{if $0 \leq k < 3-\beta$,} \\
\exp(-1) & \text{if $k = 3-\beta$,} \\
0 & \text{if $3-\beta < k < 3-\frac{\log(2)}{\log(n)}$.}
\end{array}
\right.
\end{align*}

If $\beta = 1$ then:
\begin{align*}
\lim_{n\rightarrow+\infty} \ff{n,k,\beta} ={}&
\left\lbrace
\begin{array}{ll}
1 & \text{if $0 \leq k < 2$,} \\
\exp(\exp(-2)-1) & \text{if $k = 2$,} \\
0 & \text{if $2 < k < 3-\frac{\log(2)}{\log(n)}$.}
\end{array}
\right.
\end{align*}

If $1 < \beta \leq 1+\frac{\log{(n-1)}}{\log{(n)}}$ then:
\begin{align*}
\lim_{n\rightarrow+\infty} \ff{n,k,\beta} ={}&
\left\lbrace
\begin{array}{ll}
1 & \text{if $0 \leq k<\frac{5-\beta}{2}$,} \\
\exp(-2) & \text{if $k=\frac{5-\beta}{2}$,} \\
0 & \text{if $\frac{5-\beta}{2}< k < 3-\frac{\log(2)}{\log(n)}$.}
\end{array}
\right.
\end{align*}
\end{mylem}

\begin{proof}
We are looking for the limit when $n\rightarrow\infty$ of the function defined in Equation~\eqref{app:eq:seekedlimit}.
\begin{align}
\lim_{n\rightarrow+\infty}& \left(1-2n^{k-3} + 2n^{k-3}\left(1-2n^{k-3}\right)^n \right)^\frac{n^\beta}{2} \nonumber \\
={}& \lim_{n\rightarrow+\infty} \exp\left(\frac{n^\beta}{2}\log\left(1-2n^{k-3} + 2n^{k-3}\left(1-2n^{k-3}\right)^n \right)\right) \nonumber \\
\tag{Maclaurin series of $\log$ since $\abs{2n^{k-3} - 2n^{k-3}\left(1-2n^{k-3}\right)^n} < 1$.} \nonumber \\
={}& \lim_{n\rightarrow+\infty} \exp\left(\frac{n^\beta}{2}\left(-\sum_{l=1}^\infty\frac{\left(2n^{k-3} - 2n^{k-3}\left(1-2n^{k-3}\right)^n\right)^l}{l} \right)\right) \nonumber \\
={}& \lim_{n\rightarrow+\infty} \exp\left(-\sum_{l=1}^\infty\frac{n^\beta}{2l}\left(2n^{k-3} - 2n^{k-3}\left(1-2n^{k-3}\right)^n\right)^l \right) \nonumber \\
={}& \lim_{n\rightarrow+\infty} \exp\left(\underbrace{-\frac{n^\beta}{2}\left(2n^{k-3} - 2n^{k-3}\left(1-2n^{k-3}\right)^n\right)}_{\text{Main term, with $l=1$}} \right. \nonumber \\
&\negspace{-15em} \left. \underbrace{-\sum_{l=2}^\infty\frac{n^\beta}{2l}\left(2n^{k-3} - 2n^{k-3}\left(1-2n^{k-3}\right)^n\right)^l}_{\text{Remaining term, with $l>1$}} \right) \text{.} \label{app:ineq:mainlimit}
\end{align}

In the following we study the limits of the two underlined terms in Equation~\eqref{app:ineq:mainlimit}.
\paragraph{Main term, with $l=1$.}
To obtain the limit of this term we proceed in two steps.
First we decompose it in a product of two sub terms and we study their individual limits.
Then we consider the indeterminate forms that arise after considering the products of the limits and we obtain the correct limits by upper and lower bounding the term and showing that the limits of the two bounds are identical.

We are interested in:
\begin{align*}
\lim_{n\rightarrow+\infty} -\frac{n^\beta}{2}\left(2n^{k-3} - 2n^{k-3}\left(1-2n^{k-3}\right)^n\right) = \lim_{n\rightarrow+\infty} -n^{k-3+\beta}\left(1 - \left(1-2n^{k-3}\right)^n\right)\text{.}
\end{align*}
\begin{itemize}
\item On the one hand, since $\beta \in \left[0,1+\frac{\log{(n-1)}}{\log{(n)}}\right]$ we have:
\begin{align}
\lim_{n\rightarrow+\infty}& -n^{k-3+\beta} = 
\left\lbrace
\begin{array}{ll}
0 & \text{if $0 \leq k < 3-\beta$,}\\
-1 & \text{if $k = 3-\beta$,}\\
-\infty & \text{if $3-\beta < k < 3-\frac{\log(2)}{\log(n)}$.}
\end{array}
\right. \label{app:eq:leftmult}
\end{align}

\item On the other hand, we have:
\begin{align}
\lim_{n\rightarrow+\infty}& \left(1 - \left(1-2n^{k-3}\right)^n\right) = 
\left\lbrace
\begin{array}{ll}
0 & \text{if $0 \leq k < 2$,} \\
1-\exp\left(-2\right) & \text{if $k = 2$,} \\
1 & \text{if $2 < k < 3-\frac{\log(2)}{\log(n)}$.} 
\end{array}
\right.\label{app:eq:rightmult}
\end{align}
To see that note that:
\begin{align*}
\lim_{n\rightarrow+\infty} \left(1-2n^{k-3}\right)^n = \exp\left(\lim_{n\rightarrow+\infty} n\log\left(1-2n^{k-3}\right)\right)
\end{align*}
which is an indeterminate form that can be solved using l'H\^opital's rule. Given the derivatives:
\begin{align*}
\frac{\partial \log\left( 1-2n^{k-3}\right)}{\partial n} ={}& \frac{-2(k-3)n^{k-4}}{1-2n^{k-3}}\text{,}\\
\frac{\partial \frac{1}{n}}{\partial n} ={}& \frac{-1}{n^2}\text{,}
\end{align*}
it follows that:
\begin{align*}
\lim_{n\rightarrow+\infty} n\log\left(1-2n^{k-3}\right) ={}& \lim_{n\rightarrow+\infty} -\frac{2(3-k)n^{k-2}}{1-2n^{k-3}} \\
={}& \left\lbrace \begin{matrix}
0 & \text{if $0 \leq k < 2$,} \\
-2 & \text{if $k = 2$,} \\
-\infty & \text{if $2 < k < 3-\frac{\log(2)}{\log(n)}$.} 
\end{matrix} \right.
\end{align*}
It implies:
\begin{align*}
\lim_{n\rightarrow+\infty} \left(1-2n^{k-3}\right)^n = \left\lbrace \begin{matrix}
1 & \text{if $0 \leq k < 2$,} \\
\exp\left(-2\right) & \text{if $k = 2$,} \\
0 & \text{if $2 < k < 3-\frac{\log(2)}{\log(n)}$.} 
\end{matrix} \right.
\end{align*}
\end{itemize}

Combining Limits~\eqref{app:eq:leftmult}~and~\eqref{app:eq:rightmult} we obtain the following limits:

If $0 \leq \beta < 1$ then:
\begin{align*}
\lim_{n\rightarrow+\infty} -n^{k-3+\beta}\left(1 - \left(1-2n^{k-3}\right)^n\right) = 
\left\lbrace
\begin{array}{ll}
0 & \text{if $0 \leq k < 2$,} \\
0 & \text{if $k = 2$,} \\
0 & \text{if $2 < k < 3-\beta$,} \\
-1 & \text{if $k = 3-\beta$,} \\
-\infty & \text{if $3-\beta < k < 3-\frac{\log(2)}{\log(n)}$.}
\end{array}
\right.
\end{align*}

If $\beta = 1$ then:
\begin{align*}
\lim_{n\rightarrow+\infty} -n^{k-3+\beta}\left(1 - \left(1-2n^{k-3}\right)^n\right) = 
\left\lbrace
\begin{array}{ll}
0 & \text{if $0 \leq k < 2$,} \\
\exp(-2)-1 & \text{if $k = 2$,} \\
-\infty & \text{if $2 < k < 3-\frac{\log(2)}{\log(n)}$.}
\end{array}
\right.
\end{align*}

If $1 < \beta \leq 2$ then:
\begin{align}
\lim_{n\rightarrow+\infty} -n^{k-3+\beta}\left(1 - \left(1-2n^{k-3}\right)^n\right) = 
\left\lbrace
\begin{array}{ll}
0 & \text{if $0 \leq k < 3-\beta$,} \\
0 & \text{if $k = 3-\beta$,} \\
\text{Indeterminate} & \text{if $3-\beta < k < 2$,} \\
-\infty & \text{if $k = 2$,} \\
-\infty & \text{if $2 < k < 3-\frac{\log(2)}{\log(n)}$.}
\end{array}
\right. \label{app:eq:indeterminate}
\end{align}

The only indeterminate form arises when $1 < \beta \leq 2$ and $3-\beta < k < 2$.
To solve this indeterminate form we propose to upper and lower bound $-n^{k-3+\beta}\left(1 - \left(1-2n^{k-3}\right)^n\right)$ and to show that the limits coincides.
Notice that, given our assumptions on $k$ we have that $-2n^{k-3} \in (-1,0]$ and hence, since $n \geq 2$ we can apply Bernoulli's inequalities\footnote{$1+\frac{nx}{1-nx+x} \geq (1+x)^n \geq 1+nx$ for $x \in (-1,0]$ and $n \in \nsetN$. Note that these inequalities might hold for more general values of $x$ and $n$ but we restricted ourselves to the case of interest for the proof.} to obtain:
\begin{align*}
& & \left(1-2n^{k-2}\right) \leq{}& \left(1-2n^{k-3}\right)^n \leq \left(1-\frac{2n^{k-2}}{1+2n^{k-2}-2n^{k-3}}\right) \\
\Leftrightarrow & & 1-\left(1-2n^{k-2}\right) \geq{}& 1-\left(1-2n^{k-3}\right)^n \geq 1-\left(1-\frac{2n^{k-2}}{1+2n^{k-2}-2n^{k-3}}\right) \\
\Leftrightarrow & & 2n^{k-2} \geq{}& 1-\left(1-2n^{k-3}\right)^n \geq \frac{2n^{k-2}}{1+2n^{k-2}-2n^{k-3}} \\
\Leftrightarrow & & -2n^{2k-5+\beta} \leq{}& -n^{k-3+\beta}\left(1 - \left(1-2n^{k-3}\right)^n\right) \leq -\frac{2n^{2k-5+\beta}}{1+2n^{k-2}-2n^{k-3}} 
\end{align*}
Furthermore, recalling that we are only interested in the case where $1 < \beta \leq 2$ and $3-\beta < k < 2$, we have the following limits:
\begin{align*}
\lim_{n\rightarrow+\infty} -2n^{2k-5+\beta} ={}&
\left\lbrace
\begin{array}{ll}
0 & \text{if $3-\beta<k<\frac{5-\beta}{2}$,} \\
-2 & \text{if $k=\frac{5-\beta}{2}$,} \\
-\infty & \text{if $\frac{5-\beta}{2}<k<2$.}
\end{array}
\right. \\
\lim_{n\rightarrow+\infty} -\frac{2n^{2k-5+\beta}}{1+2n^{k-2}-2n^{k-3}} ={}& 
\left\lbrace
\begin{array}{ll}
0 & \text{if $3-\beta<k<\frac{5-\beta}{2}$,} \\
-2 & \text{if $k=\frac{5-\beta}{2}$,} \\
-\infty & \text{if $\frac{5-\beta}{2}<k<2$.}
\end{array}
\right.
\end{align*}
Hence we obtain:
\begin{align*}
\lim_{n\rightarrow+\infty} -n^{k-3+\beta}\left(1 - \left(1-2n^{k-3}\right)^n\right) ={}& 
\left\lbrace
\begin{array}{ll}
0 & \text{if $3-\beta<k<\frac{5-\beta}{2}$,} \\
-2 & \text{if $k=\frac{5-\beta}{2}$,} \\
-\infty & \text{if $\frac{5-\beta}{2}<k<2$.}
\end{array}
\right.
\end{align*}
Combining this result with Equation~\eqref{app:eq:indeterminate} gives the limit of the main term:

If $0 \leq \beta < 1$ then:
\begin{align*}
\lim_{n\rightarrow+\infty} -n^{k-3+\beta}\left(1 - \left(1-2n^{k-3}\right)^n\right) = 
\left\lbrace
\begin{array}{ll}
0 & \text{if $0 \leq k < 3-\beta$,} \\
-1 & \text{if $k = 3-\beta$,} \\
-\infty & \text{if $3-\beta < k < 3-\frac{\log(2)}{\log(n)}$.}
\end{array}
\right.
\end{align*}

If $\beta = 1$ then:
\begin{align*}
\lim_{n\rightarrow+\infty} -n^{k-3+\beta}\left(1 - \left(1-2n^{k-3}\right)^n\right) = 
\left\lbrace
\begin{array}{ll}
0 & \text{if $0 \leq k < 2$,} \\
\exp(-2)-1 & \text{if $k = 2$,} \\
-\infty & \text{if $2 < k < 3-\frac{\log(2)}{\log(n)}$.}
\end{array}
\right.
\end{align*}

If $1 < \beta \leq 2$ then:
\begin{align}
\lim_{n\rightarrow+\infty} -n^{k-3+\beta}\left(1 - \left(1-2n^{k-3}\right)^n\right) = 
\left\lbrace
\begin{array}{ll}
0 & \text{if $0 \leq k < \frac{5-\beta}{2}$,} \\
-2 & \text{if $k = \frac{5-\beta}{2}$,} \\
-\infty & \text{if $\frac{5-\beta}{2} < k < 3-\frac{\log(2)}{\log(n)}$.}
\end{array}
\right. \label{app:limit:firstterm}
\end{align}

\paragraph{Remaining term, with $l>1$.}
We can rewrite the remaining term as follows:
\begin{align*}
-\sum_{l=2}^\infty\frac{n^\beta}{2l}\left(2n^{k-3} - 2n^{k-3}\left(1-2n^{k-3}\right)^n\right)^l ={}& -\sum_{l=2}^\infty\frac{n^\beta}{2l}2n^{kl-3l}\left(1 - \left(1-2n^{k-3}\right)^n\right)^l \\
={}& -\sum_{l=2}^\infty \frac{\left(1 - \left(1-2n^{k-3}\right)^n\right)^l}{l}n^{(k-3)l+\beta}
\end{align*}
Notice that $\forall k \in \left[0,3-\frac{\log(2)}{\log(n)}\right), \forall l \geq 2, \forall n \geq 2$ we have that $\left(1 - \left(1-2n^{k-3}\right)^n\right)^l \in [0,1]$. Following this we have:
\begin{align*}
l > \frac{\beta}{3-k} \Rightarrow \lim_{n\rightarrow+\infty} -n^{(k-3)l+\beta} = 0 \Rightarrow{}& \lim_{n\rightarrow+\infty} -\frac{\left(1 - \left(1-2n^{k-3}\right)^n\right)^l}{l}n^{(k-3)l+\beta} = 0 \\
l \leq \frac{\beta}{3-k} \Rightarrow  \lim_{n\rightarrow+\infty} -n^{(k-3)l+\beta} \leq 0 \Rightarrow{}&\lim_{n\rightarrow+\infty} -\frac{\left(1 - \left(1-2n^{k-3}\right)^n\right)^l}{l}n^{(k-3)l+\beta} \leq 0\text{.}
\end{align*}
From this, noticing that $k < \frac{6-\beta}{2}$ implies $\frac{\beta}{3-k}<2$, we obtain:
\begin{align}
0 \leq k < \frac{6-\beta}{2} \Rightarrow \lim_{n\rightarrow+\infty} -\sum_{l=2}^\infty\frac{n^2}{2l}\left(2n^{k-3} - 2n^{k-3}\left(1-2n^{k-3}\right)^n\right)^l = 0 \label{app:limit:secondterma}\\
\frac{6-\beta}{2} \leq k < 3-\frac{\log(2)}{\log(n)} \Rightarrow \lim_{n\rightarrow+\infty} -\sum_{l=2}^\infty\frac{n^2}{2l}\left(2n^{k-3} - 2n^{k-3}\left(1-2n^{k-3}\right)^n\right)^l \leq 0 \label{app:limit:secondtermb}
\end{align}

\paragraph{Limit of Equation~\eqref{app:ineq:mainlimit}.}
We now combine the limits obtained in Equations~\eqref{app:limit:firstterm},~\eqref{app:limit:secondterma}~and~\eqref{app:limit:secondtermb} to obtain the following limits:

If $0 \leq \beta < 1$ then $\frac{6-\beta}{2} > 3-\beta$ and thus the remaining term does not change the limit of the main term which implies:
\begin{align*}
\lim_{n\rightarrow+\infty} -\frac{n^\beta}{2}\left(2n^{k-3} - 2n^{k-3}\left(1-2n^{k-3}\right)^n\right) &-\sum_{l=2}^\infty\frac{n^\beta}{2l}\left(2n^{k-3} - 2n^{k-3}\left(1-2n^{k-3}\right)^n\right)^l = \\
&\left\lbrace
\begin{array}{ll}
0 & \text{if $0 \leq k < 3-\beta$,} \\
-1 & \text{if $k = 3-\beta$,} \\
-\infty & \text{if $3-\beta < k < 3-\frac{\log(2)}{\log(n)}$.}
\end{array}
\right.
\end{align*}

If $\beta = 1$ then $\frac{6-\beta}{2} > 2$ and thus the remaining term does not change the limit of the main term which implies:
\begin{align*}
\lim_{n\rightarrow+\infty} -\frac{n^\beta}{2}\left(2n^{k-3} - 2n^{k-3}\left(1-2n^{k-3}\right)^n\right) &-\sum_{l=2}^\infty\frac{n^\beta}{2l}\left(2n^{k-3} - 2n^{k-3}\left(1-2n^{k-3}\right)^n\right)^l = \\
&\left\lbrace
\begin{array}{ll}
0 & \text{if $0 \leq k < 2$,} \\
\exp(-2)-1 & \text{if $k = 2$,} \\
-\infty & \text{if $2 < k < 3-\frac{\log(2)}{\log(n)}$.}
\end{array}
\right.
\end{align*}

If $1 < \beta \leq 1+\frac{\log(n-1)}{\log(n)}$ then $\frac{6-\beta}{2} > \frac{5-\beta}{2}$ and thus the remaining term does not change the limit of the main term which implies:
\begin{align*}
\lim_{n\rightarrow+\infty} -\frac{n^\beta}{2}\left(2n^{k-3} - 2n^{k-3}\left(1-2n^{k-3}\right)^n\right) &-\sum_{l=2}^\infty\frac{n^\beta}{2l}\left(2n^{k-3} - 2n^{k-3}\left(1-2n^{k-3}\right)^n\right)^l = \\
&\left\lbrace
\begin{array}{ll}
0 & \text{if $0 \leq k < \frac{5-\beta}{2}$,} \\
-2 & \text{if $k = \frac{5-\beta}{2}$,} \\
-\infty & \text{if $\frac{5-\beta}{2} < k < 3-\frac{\log(2)}{\log(n)}$.}
\end{array}
\right.
\end{align*}

Plugging this result back into Equation~\eqref{app:ineq:mainlimit} gives the lemma.
\end{proof}

\section{Proof of Theorem~\ref{th:equality}}
\label{app:sec:equality}

\begin{reth}{\ref{th:equality}}[Exact bound on the probability that a strong classifier abstains]
In Theorem~\ref{th:limit}, further assume that each triplet classifier that does not abstain on at least one training example has a weight different from $0$ (if for at least one example $(x_i,y_i) \in \setS$ we have that $\fh[_c]{x_i} \neq \abstain$ then $\alpha_c \neq 0$). Then equality holds in Equation~\eqref{app:eq:boundexact}.
\end{reth}
\begin{proof}
The theorem follows directly by noticing that, in the proof of Theorem~\ref{th:limit}, if we further assume that each triplet classifier that does not abstain on at least one training example has a weight different from $0$ (if for at least one example $(x_i,y_i) \in \setS$ we have that $\fh[_c]{x_i} \neq \abstain$ then $\alpha_c \neq 0$) then we have equality in Equation~\eqref{app:eq:inequalitycond}:
\begin{align*}
\prob_{(x,y) \sim \distrS} \left[ \fH{x} = \abstain \left| \sum_{c=1}^C V_cY_{x,c} \neq 0 \right. \right] = 0 \text{.}
\end{align*}

And thus we also have that:
\begin{align*}
\prob_{(x,y) \sim \distrS} \left[ \fH{x} = \abstain \right] ={}& \prob_{(x,y) \sim \distrS} \left[ \sum_{c=1}^C V_cY_{x,c} = 0 \right]\text{.}
\end{align*}
\end{proof}

Note that, in the very unlikely event that, in one of the iterations of TripletBoost, the selected triplet classifier has an error of exactly $\nicefrac{1}{2}$, Theorem~\ref{th:equality} might not hold for our method.

\end{document}